%% file: ms.tex
\pdfoutput=1

\documentclass[smallextended,natbib,runningheads]{svjour3}
\journalname{}
\smartqed  % flush right qed marks, e.g. at end of proof
\usepackage{graphicx}
\usepackage{times}
\usepackage{amsfonts}
\usepackage{amsmath}
\usepackage{amssymb}
\usepackage{bm, color}
\usepackage{subcaption}

\usepackage{algorithm}
\usepackage{algorithmic}

\usepackage{Sty/mcr}

\usepackage{color}
\usepackage{bm}
\usepackage{setspace}
\usepackage{hyperref}

\begin{document}

\title{Statistical Inference for Sequential Feature Selection after Domain Adaptation
}


\author{
Duong Tan Loc$^{1, 2}$         
\and
Nguyen Thang Loi$^{1, 2}$
\and \\
Vo Nguyen Le Duy$^{1, 2, 3, \ast}$
}

\authorrunning{Duong Tan Loc, Nguyen Thang Loi, Vo Nguyen Le Duy} % if too long for running head

\institute{
           $^1$University of Information Technology, Ho Chi Minh City, Vietnam \\
	   $^2$Vietnam National University, Ho Chi Minh City, Vietnam \\
	   $^3$RIKEN Center for Advanced Intelligence Project (AIP), Tokyo,
103-0027, Japan \\
           $^\ast$Corresponding author. \email{duyvnl@uit.edu.vn}
}

%\date{Received: date / Revised: date}

\maketitle

\begin{abstract}
\input{abstract}
\end{abstract}

\input{sec1}

\input{sec2}
\input{sec3}
\input{sec4}
\input{sec5}
\input{sec6}

%\begin{acknowledgements}
%\end{acknowledgements}

\bibliographystyle{apalike}      % basic style, author-year citations
\bibliography{ref.bib}   % name your BibTeX data base

\newpage
\input{appendix}
\end{document}

%% file: abstract.tex
%In the context of high-dimensional regression, feature selection methods, such as sequential feature selection (SeqFS), are commonly used to identify relevant features. 
%%
%When data is limited, domain adaptation (DA) becomes crucial. DA is a paradigm that transfers knowledge from a related source domain to a target domain, with the goal of improving generalization performance.
%%
%Although SeqFS under DA is an important task in machine learning, none of the existing methods can guarantee the reliability of its results.
%%
%In this paper, we propose a novel statistical inference method for testing the features selected by SeqFS after DA, leveraging the selective inference (SI) framework.
%%
%The main advantage of the proposed method is its capability to control the false positive rate (FPR) below a pre-specified level of guarantee, $\alpha$ (e.g., 0.05).
%%
%Additionally, a strategic approach is introduced to enhance the statistical power of the test while maintaining a reasonable computational cost.
%%
%Furthermore, we provide extensions of the proposed method to SeqFS with model selection criteria such as AIC, BIC, and adjusted R-squared.
%%
%Extensive experiments are conducted on both synthetic and real-world datasets to validate the theoretical results and demonstrate the superior performance of the proposed method.
%

In high-dimensional regression, feature selection methods, such as sequential feature selection (SeqFS), are commonly used to identify relevant features. When data is limited, domain adaptation (DA) becomes crucial for transferring knowledge from a related source domain to a target domain, improving generalization performance. Although SeqFS-DA is an important task in machine learning, none of the existing methods can guarantee the reliability of its results. In this paper, we propose a novel method for testing the features selected by SeqFS-DA. The main advantage of the proposed method is its capability to control the false positive rate (FPR) below a significance level $\alpha$ (e.g., 0.05). Additionally, a strategic approach is introduced to enhance the statistical power of the test. Furthermore, we provide extensions of the proposed method to SeqFS with model selection criteria including AIC, BIC, and adjusted R-squared. Extensive experiments are conducted on both synthetic and real-world datasets to validate the theoretical results and demonstrate the proposed method's superior performance.

%% file: sec1.tex
\section{Introduction}
\label{sec:introduction}

Feature selection (FS) is a crucial task in machine learning (ML) and statistics, aimed at identifying the most relevant features in a dataset to enhance model performance, improve interpretability, and reduce computational complexity. 
One of the most popular techniques for FS is sequential feature selection (SeqFS), a greedy algorithmic approach that iteratively constructs a feature set to optimize model performance.
There are two main types of SeqFS: \emph{forward selection} and \emph{backward elimination}. 
In forward selection, starting with an empty feature set, features are sequentially selected one at a time. At  each step the feature is chosen based on its ability to improve the model performance.
In backward elimination, the process starts with all features, and features are removed one by one. Each removal is based on identifying the feature that has the least impact on model performance.
SeqFS has played a critical role in various applications and has been widely used across numerous fields, including economics and finance \citep{he2013feature}, bioinformatics \citep{inza2002gene, saeys2007review, xu2015sequence}, and fault detection \citep{yan2018cost}.

In many real-world applications, the performance of SeqFS can be significantly impaired by limited data availability. 
This challenge arises because insufficient labeled data often leads to overfitting and poor generalization. 
Domain adaptation (DA) provides an effective solution by enabling models to leverage knowledge from a source domain with abundant labeled data and transfer it to a target domain with limited labeled data. By exploiting the underlying similarities between the two domains, DA techniques—such as those based on optimal transport (OT)—align their data distributions, reducing discrepancies and facilitating more effective knowledge transfer. 
This alignment enables SeqFS techniques to identify relevant features in the target domain, even in the presence of data scarcity. 
As a result, DA mitigates the adverse effects of limited data while enhancing the performance of SeqFS in practical applications, where acquiring labeled data is often costly and time-consuming.

When performing SeqFS-DA, there is a significant risk of incorrectly identifying irrelevant features as relevant. 
These erroneous FS results are commonly referred to as false positives (FPs), which can have severe consequences, particularly in high-stakes domains like healthcare. 
For instance, if a model selects irrelevant features, a patient might be falsely assigned a high risk for a serious condition such as Alzheimer’s disease based on spurious genetic or neuroimaging data. 
This can lead to unnecessary treatments, including cognitive therapies or invasive diagnostic procedures, which carry physical, emotional, and financial costs—while failing to address the patient’s actual medical needs. 
Similarly, in infectious disease diagnostics, a false positive could falsely indicate that a healthy individual is infected, resulting in unwarranted quarantines, diagnostic tests, and treatments. 
These not only cause stress but also waste valuable healthcare resources. Such examples underscore the critical need for rigorous statistical methods to control the false positive rate (FPR), ensuring that feature selection processes yield reliable results and prevent harmful or erroneous decision-making.

\begin{figure*}[hbt!]
\centering
\includegraphics[width=1\textwidth]{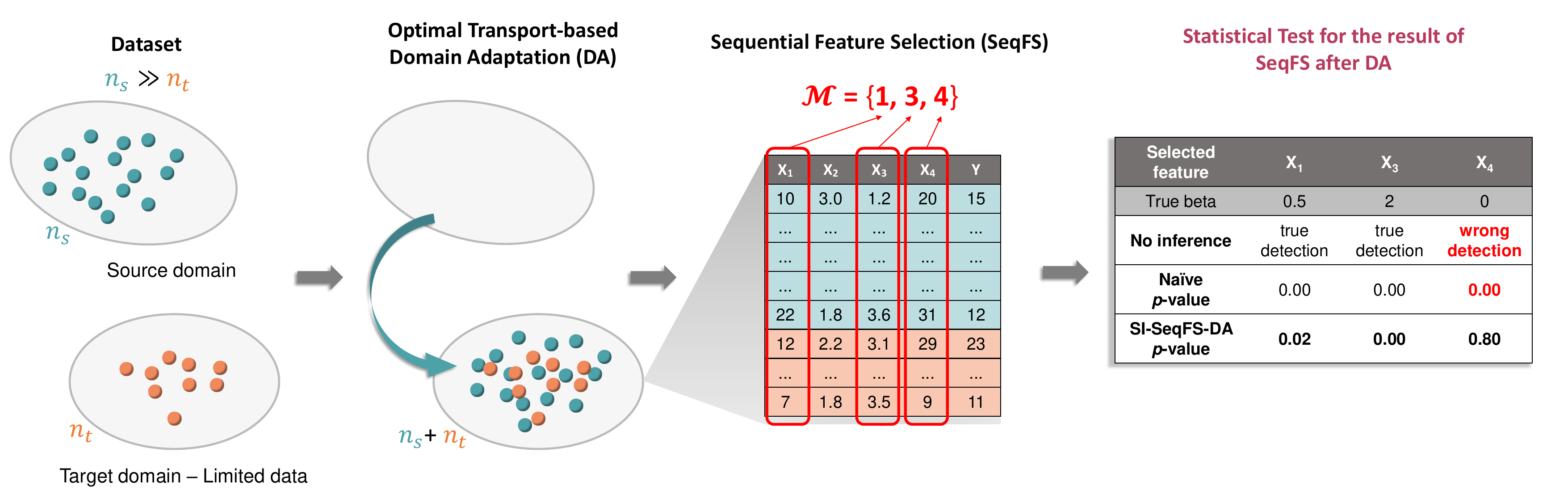}
\caption{Illustration of the proposed SI-SeqFS-DA method. When SeqFS-DA is performed without statistical inference, it often results in the selection of irrelevant features (e.g., $X_4$), as the naive $p$-value for these features may appear small, even though they are falsely detected.
In contrast, the SI-SeqFS-DA method improves feature selection by effectively distinguishing between false positives (FPs) and true positives (TPs). It assigns large $p$-values to FPs and small $p$-values to TPs, ensuring accurate identification of relevant features.}
    \label{fig:SI-SeqFS-DA}
    \vspace{-10pt}
\end{figure*}

We also emphasize the critical importance of managing the false negative rate (FNR) in statistical and ML applications. 
In statistical literature, a common strategy involves initially controlling the FPR at a pre-specified threshold, such as $\alpha = 0.05$, to control the probability of incorrectly identifying an irrelevant feature as relevant. 
However, an exclusive focus on FPR may overlook the consequences of false negatives.
To address this, practitioners aim to minimize the FNR, which is equivalent to maximizing the true positive rate (TPR), defined as TPR $=$ $1$ $-$ FNR. This dual objective---controlling the FPR while maximizing the TPR---balances reducing spurious discoveries and ensuring relevant features are not missed. 
Following this established methodology, this paper adopts a similar approach by proposing a method that theoretically controls the probability of misclassifying an irrelevant feature as relevant while simultaneously minimizing the probability of misclassifying a relevant feature as irrelevant.

To the best of our knowledge, none of the existing methods can control the FPR of SeqFS under DA. 
The primary challenge lies in the fact that, without a thorough examination of the SeqFS selection strategy under the influence of DA, the FPR cannot be properly controlled.
Several methods have been proposed in the literature for controlling the FPR in SeqFS techniques \citep{loftus2015selective, tibshirani2016exact, sugiyama2021more}. 
However, these methods assume that the data comes from the same distribution, which becomes \emph{invalid} in scenarios where a distribution shift occurs, necessitating the use of DA.
Consequently, these existing methods are not applicable to our setting.

To address the challenge of controlling the FPR, our idea is to leverage Selective Inference (Lee et al., 2016) to perform statistical tests on the features selected by SeqFS under DA.
The work of \citet{le2024cad} is the first to utilize SI to enable statistical inference within the context of DA. 
However, their method primarily addresses the unsupervised anomaly detection problem, which is fundamentally different from the setting of SeqFS under DA  we consider in this paper.
Subsequently, \citet{loi2024statistical} proposed a statistical method to test the results of FS under DA. 
Their approach, however, mainly focuses on the Lasso \citep{tibshirani1996regression}, which is distinct from SeqFS methods. 
Furthermore, in their work, controlling the FPR is achieved by relying on the KKT optimality conditions inherent in the convex optimization problem, which is not present in SeqFS methods.
Consequently, their approach cannot be applied to our setting.
Moreover, SI is problem- and model-specific, and directly applying SI in our setting is non-trivial.
This challenge arises from the need to develop a new method for the specific setting and structure of the ML model.
As a result, we must thoroughly examine the selection strategy of SeqFS under the context of DA.

In this paper, we focus on Optimal Transport (OT)-based DA \citep{flamary2016optimal}, a method that has gained popularity in the OT community.
Regarding SeqFS methods, we initially consider Forward SeqFS, where the number of selected features, denoted as $K$, is pre-defined, and subsequently extend the approach to Backward SeqFS.
To enhance the adaptability of our approach, we further extend it to cases where $K$ is not fixed but instead determined dynamically using established model selection criteria, such as the Akaike information criterion (AIC), Bayesian information criterion (BIC), and adjusted $R^2$.
The discussions on future extensions to other types of DA are provided in \S \ref{sec:conclusion}.

\vspace{8pt}
{\noindent \textbf{Contributions.} Our contributions are as follows:}

\begin{itemize}

	\item[$\bullet$] We propose a novel statistical method, named \emph{SI-SeqFS-DA} (statistical inference for SeqFS-DA), to test the features selected by SeqFS in the context of DA.
The proposed SI-SeqFS addresses the challenge of accounting for the effects of SeqFS under DA, ensuring valid inference and providing valid $p$-values for the selected features.
To the best of our knowledge, this is the first method capable of properly controlling the FPR in SeqFS-DA.
	
	\item[$\bullet$] To minimize the FNR, i.e., maximize the TPR, while properly controlling the FPR, we introduce a strategic approach utilizing the concept of \emph{divide-and-conquer}, inspired by \citep{duy2020quantifying, le2024cad}.
	
	\item[$\bullet$] We present several extensions of the proposed SI-SeqFS for both Forward SeqFS and Backward SeqFS, incorporating model selection criteria such as AIC, BIC, and adjusted $R^2$.
	
	\item[$\bullet$] We conduct thorough experiments on synthetic and real-world datasets to validate our theoretical results, highlighting the superior performance of the proposed SI-SeqFS method. Our implementation is available at 
	\begin{center}
	\href{https://github.com/locluclak/SI-SeqFS-DA}{https://github.com/locluclak/SI-SeqFS-DA}.
	\end{center}
\end{itemize}

\vspace{8pt}
{\noindent \textbf{Related works.}}
Traditional statistical inference in SeqFS often encounters challenges related to the validity of \textit{p}-values. A prominent issue stems from the use of \textit{naive} \textit{p}-values, computed under the assumption that the set of selected features is determined a priori. 
This assumption, however, does not hold in practice when employing data-driven feature selection methods such as SeqFS-DA, where the features are chosen adaptively based on the observed data. 
As a result, the naive $p$-values fail to account for the inherent selection bias introduced during the selection process of SeqFS-DA, leading to inflated FPR, undermining the statistical reliability of the inference. 
Data splitting (DS) provides a solution by partitioning the dataset into two distinct subsets: one used for feature selection and the other reserved for inference. This separation ensures that the feature selection process remains independent of the hypothesis testing phase, thereby preserving the validity of the resulting $p$-values. Despite its advantages, DS has notable limitations. Dividing the dataset reduces the amount of data available for each phase, which may weaken the statistical power of the statistical test. Moreover, in some scenarios, such as when the data exhibits strong correlations or limited sample size, splitting the data may not be feasible, further restricting the applicability of this approach.

SI has gained attention as a method for conducting valid inference on the features selected by FS techniques in linear models. 
SI was first introduced in the context of the Lasso by \citet{lee2016exact}. 
The core idea behind SI is to perform inference conditional on the FS process, thereby accounting for the selection bias. 
By doing so, SI enables the computation of valid \textit{p}-values, effectively mitigating the biases introduced during the selection phase and ensuring more reliable statistical conclusions.
The seminal paper not only established the foundation for SI research in the context of FS \citep{loftus2014significance, fithian2015selective, tibshirani2016exact, yang2016selective, hyun2018exact, sugiyama2021more, fithian2014optimal, duy2021more} but also catalyzed the development of SI methods for more complex supervised learning algorithms. 
These include boosting \citep{rugamer2020inference}, decision trees \citep{neufeld2022tree}, kernel methods \citep{yamada2018post}, higher-order interaction models \citep{suzumura2017selective}, and deep neural networks \citep{duy2022quantifying, miwa2023valid}.
However, these methods assume that the data is drawn from the same distribution. As a result, they lose their validity in the context of DA, where distribution shifts occur, making them unsuitable for such scenarios.

Moreover, SI has been introduced and developed for unsupervised learning problems, such as change point detection  \citep{umezu2017selective, hyun2018post, duy2020computing, sugiyama2021valid, jewell2022testing}, clustering \citep{lee2015evaluating, inoue2017post, gao2022selective, chen2022selective}, and segmentation \citep{tanizaki2020computing, duy2022quantifying}
Additionally, SI can be applied to statistical inference on the Dynamic Time Warping (DTW) distance \citep{duy2022exact} and the Wasserstein distance \citep{duy2021exact}, further extending its versatility to various distance metrics.

The studies most closely related to this paper are \citet{le2024cad} and \citet{loi2024statistical}. 
In \citet{le2024cad}, the authors propose a method for computing valid $p$-values for anomalies detected by an anomaly detection method within an OT-based domain adaptation (DA) setting. 
However, their focus is on unsupervised learning and anomaly detection, which differs fundamentally from the supervised setting of SeqFS under DA that we address in this paper .
Subsequently, the authors of \citet{loi2024statistical} introduced a statistical method for testing the results of features selected by Lasso in the context of DA. 
In their approach, controlling the FPR is achieved by leveraging the KKT optimality conditions inherent in the convex optimization problem of Lasso after DA, a property that is not present in SeqFS-DA methods. Consequently, their method is not directly applicable to our setting.

%% file: sec2.tex
\section{Problem Statement}
To formulate the problem, we consider a regression setup involving two random response vectors defined as follows:
\begin{align*}
    \bm Y^s &= (Y_1^s, \dots, Y_{n_s}^s)^\top \sim \mathbb{N}(\bm \mu^s, \Sigma^s), \\
    \bm Y^t &= (Y_1^t, \dots, Y_{n_t}^t)^\top \sim \mathbb{N}(\bm \mu^t, \Sigma^t),
\end{align*}
where $n_s$ and $n_t$ denote the number of instances in the source and target domains, respectively. The $\bm \mu^s$ and $\bm \mu^t$ are unknown signals, $\bm \veps^s$ and $\bm \veps^t$ are the Gaussian noise vectors with the covariance matrices $\Sigma^s$ and $\Sigma^t$, which are assumed to be known or estimable from independent data.
We denote the feature matrices in the source and target domains, which are non-random, by $X^s \in \RR^{n_s \times p}$ and $X^t \in \RR^{n_t \times p}$, 
respectively, where $p$ is the number of features.
We assume that the number of instances in the target domain is limited, i.e., $n_t$ is much smaller than $n_s$.
The goal is to conduct the statistical test on the results of SeqFS after DA.

\subsection{Optimal Transport (OT)-based DA \citep{flamary2016optimal}}

Let us define the source and target data as:
\begin{align} \label{eq:Ds_Dt}
	D^s &= 
	\begin{pmatrix}
		X^s ~ \bm Y^s
	\end{pmatrix} 
	\in \RR^{n_s \times (p + 1)}
	\quad \text{ and } \quad
	D^t = 
	\begin{pmatrix}
		X^t ~ \bm Y^t
	\end{pmatrix} \in \RR^{n_t \times (p + 1)}.
\end{align}

Next, we define the the cost matrix as:
\begin{align*}
	C(D^s, D^t) 
	& = \Big[
	\big \| D_i^s - D_j^t \big \|^2_2 
	\Big]_{ij} \in \RR^{n_s \times n_t},
\end{align*}
for any $i \in [n_s] = \{1, 2, ..., n_s\}$ and $j \in [n_t]$.
Here, $D_i^s \in \RR^{p + 1}$ and $D_j^t \in \RR^{p + 1}$ are the $i^{\rm th}$ and $j^{\rm th}$ rows of $D^s$ and $D^t$, respectively.
The OT problem for DA between the source and target domains is then defined as:
\begin{align} \label{eq:ot_problem}
		&\hat{T} = \argmin 
		\limits_{T \in \RR^{n_s \times n_t}, ~T \geq 0} 
		~ \big \langle T, C(D^s, D^t) \big \rangle \\ 
		& \quad \quad \quad \quad ~ \text{s.t.} ~~~ T \bm{1}_{n_t} = \frac{\bm 1_{n_s}}{{n_s}}, ~ T^\top \bm{1}_{n_s} = \frac{\bm 1_{n_t}}{{n_t}}  \nonumber,
\end{align}
where $\langle\cdot,\cdot\rangle$ is the Frobenius inner product, $\bm{1}_n \in \RR^n$ is the vector whose elements are set to $1$.
Once the optimal transportation matrix $\hat{T}$ is obtained, the source instances are transported into the target domain.
The transformation $\tilde{D}^s$ of $D^s$ is defined as:
\begin{align} \label{eq:tilde_D_s}
	\tilde{D}^s 
		= n_s \hat{T} D^t \in \RR^{n_s \times (p + 1)}.
\end{align} 
Further details can be found in Section 3.3 of \citet{flamary2016optimal}.
Let us decompose $\tilde{D}^s$ into 
$
	\tilde{D}^s = 
	\begin{pmatrix}
		\tilde{X}^s ~ \tilde{\bm Y}^s
	\end{pmatrix},
$
the matrix $\tilde{X}^s$ and vector $\tilde{\bm Y}^s$ can be defined as:
\begin{align} \label{eq:data_after_da}
	\tilde{X}^s = n_s \hat{T} X^t
	\quad 
	\text{and}
	\quad
	\tilde{\bm Y}^s = n_s \hat{T} \bm Y^t,
\end{align}
according to \eq{eq:tilde_D_s} and the definition of $D^t$ in \eq{eq:Ds_Dt}.
Here, $\tilde{X}^s$ and $\tilde{\bm Y}^s$ represent the transformations of $X^s$ and $\bm Y^s$ to the target domain, respectively.

\subsection{Sequential Feature Selection (SeqFS) after OT-based DA}

For simplicity, we primarily focus on Forward SeqFS.
The extension to Backward SeqFS will be discussed later.
Henceforth, any mention of SeqFS implicitly refers to Forward SeqFS.
Recall that in SeqFS, the feature that provides the greatest improvement to the fit is repeatedly added to the current model.
After each addition, the coefficients are recomputed using least squares regression on the selected features.
This process continues until the model includes $K$ features, where $K$ is pre-specified by the analysts.

After transforming the data from the source domain to the target domain, we apply SeqFS to the combined dataset consisting of the transformed source data and the target data.
Specifically, let us denote
\begin{align} \label{eq:tilde_X_Y}
	\tilde{X} = 
	\begin{pmatrix}
		\tilde{X}^s \\ 
		X^t
	\end{pmatrix}
	\in \RR^{(n_s + n_t) \times p}, 
	\quad
	\tilde{\bm Y} = 
	\begin{pmatrix}
		\tilde{\bm Y}^s \\ 
		\bm Y^ t
	\end{pmatrix}
	\in \RR^{n_s + n_t},
\end{align}
and $\cM \in [p]$ be a set of features, the residual sum of squares from  regressing $\tilde{\bm Y}$ onto $\tilde{X}_\cM$ is defined as:
\begin{align*}
    {\rm RSS}(\tilde{\bm Y}, \tilde{X}_\cM) 
    =  
    \Big \| \Big(I_{n_s+n_t} - P_{\tilde{X}_{\cM}} \Big) \tilde{\bm Y}
    \Big \|_2^2
\end{align*}
where $I_{n_s + n_t} \in \RR^{n_s + n_t}$ is the identity matrix and $P_{\tilde{X}_{\cM}} = \tilde{X}_{\mathcal{M}} \left( \tilde{X}_{\mathcal{M}}^\top \tilde{X}_{\mathcal{M}}\right)^{-1} \tilde{X}_{\mathcal{M}}^\top$.
Then, the procedure of applying SeqFS on $\big ( \tilde{X}, \tilde{\bm Y} \big )$ to obtain a set $\cM$ of indices of $K$ selected features in the target domain is described as follows:
\begin{enumerate}
	\item Let $\cM_0$ denote the \emph{null} model that does not contain any features
	\item For $k = 1$ to $K$:
	\begin{itemize}
		\item[$\bullet$] At step $k$, the feature $j_k$ is selected by solving the optimization problem:
		\begin{align*}
		 j_k = \argmin \limits_{j \in [p] \setminus \mathcal{M}_{k-1}}  
		 {\rm RSS}\Big(\tilde{\bm Y}, \tilde{X}_{\mathcal{M}_{k-1} \cup \{j\}} \Big),
		\end{align*}
		where $\cM_{k - 1} = \{ j_1, ..., j_{k - 1}\}$ is the set of selected features up to step $k-1$.
	\end{itemize}
	\item For notational simplicity, we  denote the final set of selected features as: 
\begin{align} \label{eq:final_selected_model}
	\cM = \cM_K.
\end{align}
\end{enumerate}

\subsection{Statistical Inference on the Selected Features}

We aim to assess whether the features selected in $\mathcal{M}$ are truly relevant or merely the result of chance.
In order to quantify the statistical significance of the $j^{th}$ selected features in $\mathcal{M}$, we consider statistically testing the following  hypotheses:
\begin{align} \label{eq:hypotheses}
 {\rm H}_{0, j}:  \beta_j  = 0 \quad \text{vs.} \quad {\rm H}_{1, j}:  \beta_j \neq 0,
\end{align}
where $\beta_j = \left [ \big ({X^t_{\cM}}^\top X^t_{\cM} \big)^{-1} {X^t_{\cM}}^\top \bm \mu^t \right ]_j$ and $X^t_{\cM}$ denotes the sub-matrix of $X^t$ consisting of the columns corresponding to the set $\cM$.

To test these hypotheses, a natural choice of the test statistic is the least square estimate, defined as:
\begin{align*} 
	\hat{\beta}_j = \left [ \big ({X^t_{\cM}}^\top X^t_{\cM} \big)^{-1} {X^t_{\cM}}^\top \bm Y^t \right ]_j
%	= \bm \eta_j^\top {\bm Y^s \choose \bm Y^t },
\end{align*}
The test statistic can be re-written as a linear contrast w.r.t. the data, expressed as:
\begin{align}\label{eq:test_statistic}
	\hat{\beta}_j = \bm \eta_j^\top {\bm Y^s \choose \bm Y^t }
	\quad
	\text{where}
	\quad
	\bm \eta_j = 
\begin{pmatrix}
	\bm 0^{s} \\ 
	X^t_{\cM} \big ({X^t_{\cM}}^\top X^t_{\cM} \big)^{-1} \bm{e}_j^t
\end{pmatrix},
\end{align}
$\bm 0^{s} \in \RR^{n_s}$ represents a vector where all entries are set to 0, 
$\bm e^{t}_j \in \RR^{|\cM|}$ is a vector in which the $j^{\rm th}$ entry is set to $1$, and $0$ otherwise.

After computing the test statistic from \eq{eq:test_statistic}, we calculate the $p$-value to guide the decision-making.
At a significance level $\alpha \in [0,1]$, typically 0.05, we reject $H_{0,j}$ and conclude that the $j^{th}$ feature is relevant if the $p$-value is less than or equal to $\alpha$. Conversely, if the $p$-value exceeds $\alpha$, there is insufficient evidence to determine the relevance of the $j^{th}$ feature.

\subsection{Computation of a valid $p$-value}

Suppose the hypotheses in \eq{eq:hypotheses} are fixed, meaning they are non-random. In this case, the traditional (naive) $p$-value is calculated by

\begin{align*}
	p^{\rm naive}_j = 
	\mathbb{P}_{\rm H_{0, j}} 
	\Bigg ( 
		\left | \bm \eta_j^\top {\bm Y^s \choose \bm Y^t } \right |
		\geq 
		\left | \bm \eta_j^\top {\bm Y^s_{\rm obs} \choose \bm Y^t_{\rm obs} } \right |
	\Bigg ), 
\end{align*}
where $\bm Y^s_{\rm obs}$ and $\bm Y^t_{\rm obs}$ are the observations (realizations) of the random vectors $\bm Y^s$ and $\bm Y^t$, respectively.
The naive $p$-value is valid in the sense that
\begin{align} \label{eq:valid_p_value}
	\mathbb{P} \Big (
	\underbrace{p_j^{\rm naive} \leq \alpha \mid {\rm H}_{0, j} \text{ is true }}_{\text{a false positive}}
	\Big) = \alpha, ~~ \forall \alpha \in [0, 1],
\end{align} 
i.e., the probability of obtaining a false positive is controlled at the significance level $\alpha$.
However, in our setting, the hypotheses in \eq{eq:hypotheses} are actually \emph{not} fixed in advance, as they are defined based on the (random) results obtained by applying SeqFS after DA on the data.
As a result, the property of a valid $p$-value, as described in \eq{eq:valid_p_value}, is no longer satisfied.
Consequently, in the setting considered in this paper, the naive $p$-value is \emph{invalid}, as it fails to account for the effects of SeqFS under DA.
In the next section, we leverage the SI framework to account for the influence of SeqFS under DA and introduce a valid $p$-value for conducting the statistical test to control the FPR.

%% file: sec3.tex
\section{Proposed Method}
\label{sec:proposed_method}

To conduct valid statistical inference, it is essential to examine how the data influences the feature selection results through the selection strategy of SeqFS after DA. We address this challenge by leveraging the concept of SI. In the SI framework, statistical inference is based on the sampling distribution of the test statistic, conditioned on the SeqFS results after DA, thereby accounting for how data is processed in SeqFS after DA to obtain the set of selected features.

\subsection{The valid $p$-value in SI-SeqFS-DA}
\begin{figure*}[!t]
    \centering
    \includegraphics[width=1\textwidth]{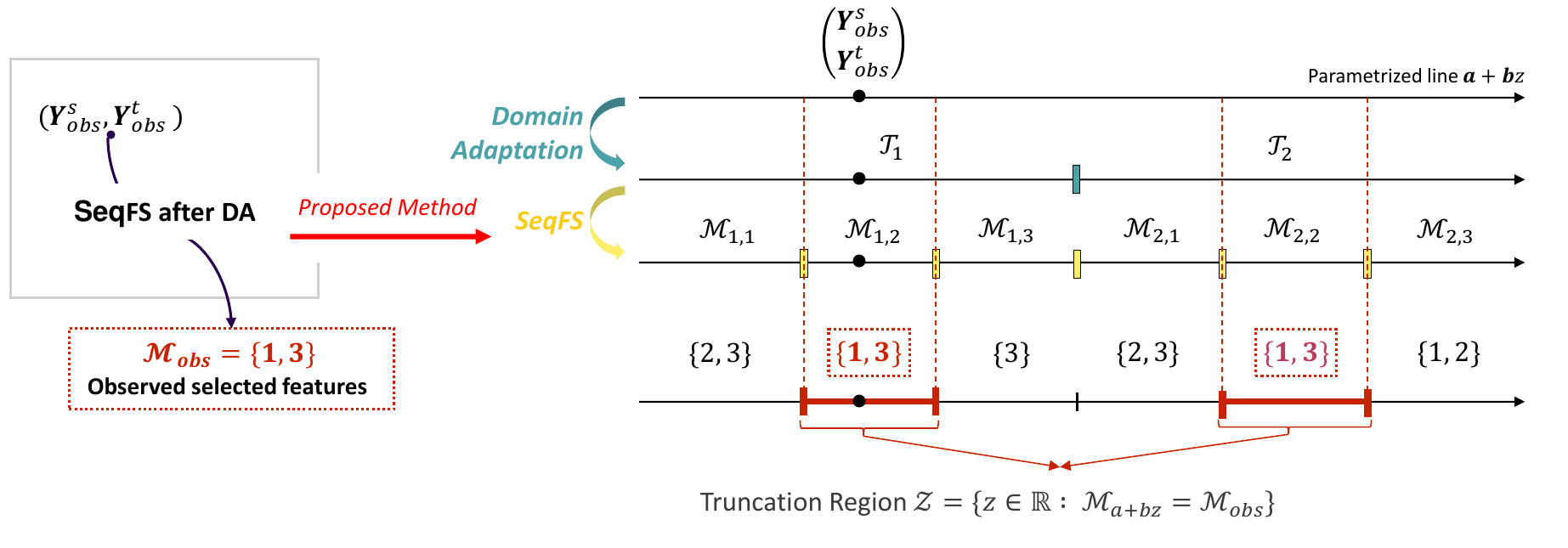}
    \caption{
    % After performing DA, we apply FS to identify the relevant features. 
    % %
    % Next, we parametrize the data using a scalar parameter $z$ in the dimension of the test statistic to define the truncation region $\cZ$, whose the data have the \emph{same} FS results as the observed data. Finally, we conduct the inference by conditioning on $\cZ$. To enhance the efficiency, we utilize a divide-and-conquer strategy to effectively identify the region $\cZ$.
    Illustration of the SI-SeqFS-DA method. First, we transform the data using domain adaptation (DA). Subsequently, sequential feature selection (SeqFS) is applied to identify the relevant features. The data is then parameterized using a scalar parameter $z$, defined in the dimension of the test statistic, to determine the truncation region $\cZ$. To improve computational efficiency, a divide-and-conquer strategy is employed to effectively identify $\cZ$. Finally, valid statistical inference is performed within the identified region $\cZ$.
    }
    \label{fig:line_search_approach}
    \vspace{-10pt}
\end{figure*}
To compute the valid $p$-value, we must first determine the distribution of the test statistic described in \eq{eq:test_statistic}, which is defined as follows based on the concept of SI:
\begin{align} \label{eq:conditional_distribution}
	\mathbb{P} \Bigg ( 
	\bm \eta_j^\top {\bm Y^s \choose \bm Y^t }
	~
	\Big |
	~ 
	\cM_{\bm Y^s, \bm Y^t}
	=
	\cM_{\rm obs}
	\Bigg ),
\end{align}
where $\cM_{\bm Y^s, \bm Y^t}$ is the set of selected features of SeqFS after DA for any \emph{random} vectors $\bm Y^s$ and $\bm Y^t$, and $\cM_{\rm obs} = \cM_{\bm Y^s_{\rm obs}, \bm Y^t_{\rm obs}}$ is the observed selected features when applying SeqFS-DA on the observed data $\bm Y^s_{\rm obs}$ and $\bm Y^t_{\rm obs}$. 
Here, the conditioning part
$
\cM_{\bm Y^s, \bm Y^t}
=
\cM_{\rm obs}
$
in \eq{eq:conditional_distribution} indicates that we only consider vectors $\bm Y^s$ and $\bm Y^t$ on which a specific set of selected features $\cM_{\rm obs}$ is obtained.

Based on the distribution of the test statistic in \eq{eq:conditional_distribution}, the selective $p$-value is then defined as follows:
\begin{align} \label{eq:selective_p}
	p^{\rm selective}_j = 
	\mathbb{P}_{{\rm H}_{0, j}} 
	\Bigg ( ~
		\left | \bm \eta_j^\top {\bm Y^s \choose \bm Y^t } \right |
		\geq 
		\left | \bm \eta_j^\top {\bm Y^s_{\rm obs} \choose \bm Y^t_{\rm obs} } \right |
		~~
		\Bigg | 
		~~
		\begin{array}{l}
		\cM_{\bm Y^s, \bm Y^t}
		=
		\cM_{\rm obs}, \\
		\cQ_{\bm Y^s, \bm Y^t}
		=
		\cQ_{\rm obs}
		\end{array} ~
	\Bigg ), 
\end{align}
%
%where $\cE$ is the conditioning event defined as
%%
%\begin{align} \label{eq:conditioning_event}
%\cE = \Big \{ 
%	\cM_{\bm Y^s, \bm Y^t}
%	=
%	\cM_{\rm obs}, ~
%	\cQ_{\bm Y^s, \bm Y^t}
%	=
%	\cQ_{\rm obs}
%\Big \}. 
%\end{align}
%
where  $\cQ_{\bm Y^s, \bm Y^t}$ is the \emph{nuisance component} defined as 
\begin{align} \label{eq:q_and_b}
	\cQ_{\bm Y^s, \bm Y^t} = 
	\Big ( 
	I_{n_s + n_t} - 
	\bm b
	\bm \eta_j^\top \Big ) 
	{\bm Y^s \choose \bm Y^t}
\end{align}
with 
$
	\bm b = \frac{\Sigma \bm \eta_j}
	{\bm \eta_j^\top \Sigma \bm \eta_j}
$
and 
$
\Sigma = 
\begin{pmatrix}
	\Sigma^s & 0 \\ 
	0 & \Sigma^t
\end{pmatrix}.
$
We note that the nuisance component $\cQ_{\bm Y^s, \bm Y^t}$ corresponds to
the component $\bm z$ in the seminal paper of \cite{lee2016exact} (see Sec. 5, Eq. (5.2)).
The additional conditioning on $\cQ_{\bm Y^s, \bm Y^t}$ is necessary for technical reasons, specifically to enable tractable inference. 
This approach is standard in the SI literature and is employed in nearly all SI-related works we cite.

\begin{lemma} \label{lemma:valid_selective_p}
The selective $p$-value proposed in \eq{eq:selective_p} is a valid $p$-value, i.e.,
\begin{align*}
	\mathbb{P}_{{\rm H}_{0, j}}  \Big (
	p_j^{\rm selective} \leq \alpha
	\Big) = \alpha, ~~ \forall \alpha \in [0, 1].
\end{align*} 
\end{lemma}

\begin{proof}
The proof is deferred to Appendix \ref{app:proof_valid_p}.
\end{proof}

Lemma \ref{lemma:valid_selective_p} indicates that, by utilizing the proposed selective $p$-value, the FPR is theoretically controlled for any level $\alpha \in [0, 1]$.
Once the conditioning event 
$\big \{ \cM_{\bm Y^s, \bm Y^t}
=
\cM_{\rm obs}, ~
\cQ_{\bm Y^s, \bm Y^t}
=
\cQ_{\rm obs} \big \}$
is identified, the selective $p$-value can be computed.
The proper characterization of the conditioning event is crucial for the computation, as it directly influences the validity of the inference. 
In the next section, we will present the detailed identification of $\big \{ \cM_{\bm Y^s, \bm Y^t}
=
\cM_{\rm obs}, ~
\cQ_{\bm Y^s, \bm Y^t}
=
\cQ_{\rm obs} \big \}$, which will then allow for the effective calculation of the selective $p$-value.

\subsection{Characterization of the Conditioning Event}

Let us define the set of ${\bm Y^s \choose \bm Y^t } \in \RR^{n_s + n_t}$ that satisfies the conditions in \eq{eq:selective_p} as:
\begin{align} \label{eq:conditional_data_space}
	{\hspace{-2mm}}\cD = \left \{ 
	{\bm Y^s \choose \bm Y^t }
	\in \RR^{n_s + n_t}
	 ~\Big | ~
	\begin{array}{l}
	\cM_{\bm Y^s, \bm Y^t}
	=
	\cM_{\rm obs}, \\
	\cQ_{\bm Y^s, \bm Y^t}
	=
	\cQ_{\rm obs}
	\end{array}
	\right \}. 
\end{align}
According to the second condition on the nuisance component, the data in $\cD$ is restricted to a \emph{line} in $\RR^n$  as stated in the following lemma.

\begin{lemma} \label{lemma:data_line}
Let us define $\bm a = \cQ_{\rm obs}$, the set $\cD$ in \eq{eq:conditional_data_space} can be rewritten using a scalar parameter $z \in \RR$ as:
\begin{align} \label{eq:conditional_data_space_line}
	\cD = \left \{ {\bm Y^s \choose \bm Y^t } = \bm Y(z) = \bm a + \bm b z \mid z \in \cZ \right \},
\end{align}
where $\bm b$ is defined in \eq{eq:q_and_b} and $\cZ$ is defined as:
\begin{align} \label{eq:cZ}
	\cZ = \Big \{ 
	z \in \RR 
	\mid 
	\cM_{\bm a + \bm b z} = \cM_{\rm obs}
	\Big \}.
\end{align}
Here, with a slight abuse of notation,
$
\cM_{\bm a + \bm b z} = \cM_{{\bm Y^s \choose \bm Y^t }}
$
is equivalent to $\cM_{\bm Y^s, \bm Y^t}$.
\end{lemma}
\begin{proof}
The proof is deferred to Appendix \ref{app:proof_line}.
\end{proof}

%The proof is provided in Appendix \ref{app:proof_line}. 

Lemma \ref{lemma:data_line} shows that we need to focus only on the \emph{one-dimensional projected} data space 
$\cZ$ in \eq{eq:cZ} rather than the $(n_s + n_t)$-dimensional data space in \eq{eq:conditional_data_space}. 
The idea of restricting the conditional space to a line was implicitly utilized in \citet{lee2016exact} and explicitly discussed in Sec. 6 of \citet{liu2018more}.
Let us denote a random variable $Z \in \RR$ and its observation $Z_{\rm obs} \in \RR$ as follows:
\begin{align*}
	Z = \bm \eta_j^\top {\bm Y^s \choose \bm Y^t } \in \RR 
	~~ \text{and} ~~ 
	Z_{\rm obs} = \bm \eta_j^\top {\bm Y^s_{\rm obs} \choose \bm Y^t_{\rm obs} } \in \RR.
\end{align*}
Then, the selective $p$-value in (\ref{eq:selective_p}) can be rewritten as 
\begin{align} \label{eq:selective_p_reformulated}
	p^{\rm selective}_j = \mathbb{P}_{\rm H_{0, j}} \Big ( |Z| \geq |Z_{\rm obs}| \mid  Z \in \cZ \Big ).
\end{align}

Once the truncation region $\cZ$ is identified, computation of the selective $p$-value in (\ref{eq:selective_p_reformulated}) is straightforward.
Therefore, the remaining task is to identify $\cZ$.

\subsection{Identification of Truncation Region $\cZ$}
\label{subsec:identification_cZ}

Due to the inherent complexity of $\cZ$, whose structure is often too intricate to be captured by simple procedures, directly identifying it poses significant challenges.
To overcome this obstacle, we employ a \emph{``divide-and-conquer''} strategy, which breaks down the problem into more manageable subproblems. 
Inspired by the methodologies of \citet{le2024cad} and \citet{loi2024statistical}, we introduce an efficient method (demonstrated in Fig. \ref{fig:line_search_approach}) that allows for the identification of $\cZ$ in a systematic and computationally feasible manner.
Our strategy is outlined as follows:

\begin{itemize}
	\item[$\bullet$] We decompose the problem into multiple sub-problems by additionally conditioning on the transportation for DA and the \emph{order sequence} of the sets of the selected features across the $K$ steps of the SeqFS method (Sec. \ref{subsubsec:divide_conquer}).
	\item[$\bullet$] We show that each sub-problem can be solved efficiently (Sec \ref{subsubsec:solving_each_subproblem}).
	\item[$\bullet$] We integrate the solutions of multiple sub-problems to obtain $\cZ$ (Sec. \ref{subsubsec:compute_cZ}).
\end{itemize}

\subsubsection{Divide-and-conquer strategy} \label{subsubsec:divide_conquer}

Let $U$ represent the total number of possible transportations for DA along the parameterized line. 
For each transportation $\cT_u$, where $u \in [U]$, we define $V_u$ as the total number of possible order sequences of selected feature sets across the $K$ steps of SeqFS after the $\cT_u$ transportation.
The entire one-dimensional space $\RR$ can be decomposed as:
\begin{align} \label{eq:divide_and_conquer}
	\RR
	& = 
	\bigcup \limits_{u \in [U]}
	\bigcup \limits_{v \in [V_u]}
	\Big \{
	\underbrace{
	z  \in \RR
	\mid 
	\cT_{\bm a + \bm b z} = \cT_u, ~
	\cO_{\bm a + \bm b z} = \cO_{v} 
	}_{
	\text{a sub-problem of additional conditioning}
	}
	\Big \},
\end{align}
where $\cT_{\bm a + \bm b z}$ represents the OT-based DA on $\bm a + \bm b z$ and
$\cO_{\bm a + \bm b z}$ denotes the order sequence of the sets of the selected features across the $K$ steps of the SeqFS after DA. Specifically, the $\cO_{\bm a + \bm b z}$ is defined as follows:
\begin{align*}
	\cO_{\bm a + \bm b z}
	= \Big (
	\cM_1, \cM_2, ..., \cM_K
	\Big ),
\end{align*}
where $\cM_1, \cM_2, \dots, \cM_K$ implicitly represent the selected feature sets across the $K$ steps applied to the data $\bm a + \bm b z$.

Let us define a function that maps $( \cT_{\bm a + \bm b z}, \cO_{\bm a + \bm b z} )$ to a final selected set of features $\cM_{\bm a + \bm b z}$ after $K$ steps of SeqFS-DA as:
\begin{align*}
	\cA : \big ( \cT_{\bm a + \bm b z}, \cO_{\bm a + \bm b z} ) \mapsto \cM_{\bm a + \bm b z}.
\end{align*}
We aim to search a set
\begin{align} \label{eq:cW}
	\cW = 
	\Big \{ 
		(u, v) \mid \cA \big(\cT_u, \cO_v \big ) = \cM_{\rm obs}
	\Big \}.
\end{align}
After obtaining $\cW$, the region $\cZ$ in \eq{eq:cZ} can be computed as follows:
\begin{align}
	\cZ 
	& = \Big \{ 
	z \in \RR 
	\mid 
	\cM_{\bm a + \bm b z} = \cM_{\rm obs}
	\Big \} \nonumber \\ 
	& = 
	\bigcup \limits_{(u, v) \in \cW}
	\Big \{ 
	z \in \RR 
	\mid 
	\cT_{\bm a + \bm b z} = \cT_u,
	\cO_{\bm a + \bm b z} = \cO_v
	\Big \}. \label{eq:cZ_new}
\end{align}

\subsubsection{Solving of each sub-problem}
\label{subsubsec:solving_each_subproblem}

For any $u \in [U]$ and $v \in [V_u]$, we define the subset of the one-dimensional projected dataset along a line that corresponds to a sub-problem in \eq{eq:divide_and_conquer} as follows:
\begin{align} \label{eq:cZ_extra_condition}
	\cZ_{u, v} = 
	\Big \{z \in \RR 
	\mid 
	\cT_{\bm a + \bm b z} = \cT_u,
	\cO_{\bm a + \bm b z} = \cO_v 
	\Big \}.
\end{align}
The sub-problem region can be re-written as:
\begin{align} \label{eq:Z_u_and_Z_v}
	\cZ_{u, v}  = \cZ_u \cap \cZ_v,
\end{align}
where 
$
	\cZ_u  = 
	\Big \{ 
	z \in \RR
	\mid 
	\cT_{\bm a + \bm b z} = \cT_u
	\Big \}
$
and
$
	\cZ_v = 
	\Big \{ 
	z  \in \RR
	\mid 
	\cO_{\bm a + \bm b z} = \cO_v
	\Big \}.
$

\begin{lemma} \label{lemma:cZ_u}
The set $\cZ_u$ can be characterized by a set of quadratic inequalities w.r.t. $z$ described as follows:
\begin{align*}
	\cZ_u
	= \Big \{ 
	z \in \RR 
	\mid 
	\bm p + \bm q z + \bm f z^2 \geq \bm 0
	\Big \},
\end{align*}
where vectors $\bm p$, $\bm q$, and $\bm f$ are defined in Appendix \ref{app:proof_cZ_u}.
\end{lemma}

\begin{proof}
The proof is deferred to Appendix \ref{app:proof_cZ_u}. 
\end{proof}
%
%The purpose of Lemma \ref{lemma:cZ_u} is to ensure that the transportation $\cT_u$ remains the same for all $z \in \cZ_u$.
%
%The composition of the set $\cZ_u$ is determined by the selected algorithm, either $\cA_1$ forward SFS or $\cA_2$ backward SFS. we will detail the process for constructing $\cZ_u$ in each approach

\begin{lemma}\label{lemma:cZFS_v}
Let us define the forward SeqFS problem after $\cT_u$ transportation as, for any $k \in [K]$:
\begin{equation*}
    j_k (z) = \argmin \limits_{j \in [p] \setminus \mathcal{M}_{k-1}} 
    {\rm RSS}\Big (\tilde{\bm Y}_u(z), \tilde{X}_{u_{\mathcal{M}_{k-1} \cup \{j\}}}
    \Big),
\end{equation*}
where $\tilde{\bm Y}_u(z) = \Omega_u \bm Y (z)$ and $\tilde{X}_{u_{\mathcal{M}_{k-1} \cup \{j\}}} = (\Omega_u X)_{\mathcal{M}_{k-1} \cup \{j\}}$.
Here, the matrix $\Omega_u$ is defined as follows: 
\begin{align*}
	\Omega_u = 
	\begin{pmatrix}
		0_{n_s \times n_s} & n_s \cT_u \\
		0_{n_t \times n_s} & I_{n_t}
	\end{pmatrix}
	\in \RR^{(n_s + n_t) \times (n_s + n_t)},
\end{align*}
where $0_{n \times m} \in \RR^{n \times m}$ is the zero matrix, $I_n \in \RR^{n \times n}$ is the identity matrix, and $X = (X^s ~ X^t)^\top$.
The set $\cZ_v$ in \eq{eq:Z_u_and_Z_v} can be characterized by a set of quadratic inequalities w.r.t. $z$ described as follows:
\begin{align*}
    \cZ_v = \Big \{ 
	z \in \RR 
	\mid 
	\bm w + \bm r z + \bm o z^2 \leq \bm 0
	\Big \},
\end{align*}
where the vectors $\bm w$, $\bm r$, and $\bm o$ are defined in Appendix \ref{app:proof_Zv_fw}.
\end{lemma}

\begin{proof}
The proof is deferred to Appendix \ref{app:proof_Zv_fw}.
\end{proof}
In Lemmas \ref{lemma:cZ_u} and \ref{lemma:cZFS_v}, we show that $\cZ_u$ and $\cZ_v$ can be \emph{analytically obtained} by solving the respective systems of quadratic inequalities.
Once $\cZ_u$ and $\cZ_v$ are computed, the sub-problem region $\cZ_{u, v}$ in \eq{eq:Z_u_and_Z_v} is determined by the intersection $\cZ_{u, v} = \cZ_u \cap \cZ_v$.

\begin{remark}

The selective $p$-value computed using the sub-problem region $\cZ_{u,v}$ remains valid and this fact is well-known in the literature of conditional SI. 
However, the primary drawback in this case is the low TPR, i.e.,  high FNR.
To minimize the FNR, it is necessary to compute the selective $p$-value using the region $\cZ$ defined in \eq{eq:cZ}. The approach for efficiently computing $\cZ$ will be discussed in the next section.
\end{remark}

\begin{algorithm}[!t]
\caption{\texttt{SI-SeqFS-DA}}
\label{alg:si_seqfs_da}
\begin{footnotesize}
\textbf{Input:} $X^s, \bm Y^s_{\rm obs}, X^t, \bm Y^t_{\rm obs}, z_{\min}, z_{\max}$

\begin{algorithmic}[1]
\vspace{4pt}
    \STATE $\cM_{\rm obs} \gets$ SeqFS after DA on $\Big (X^s, \bm Y^s_{\rm obs} \Big )$ and  $\Big (X^t, \bm Y^t_{\rm obs} \Big)$
    \vspace{4pt}
    \FOR{$j \in \cM_{\rm obs}$}
    \vspace{4pt}
        \STATE Compute $\bm \eta_{j} \gets$ Eq. (\ref{eq:test_statistic}), $\boldsymbol{a} \text{ and } \boldsymbol{b} \gets$ Eq. (\ref{eq:conditional_data_space_line}), $X = \big (X^s ~ X^t \big )^\top$
%        \vspace{4pt}
%        \STATE $X = \big (X^s ~ X^t \big )^\top$
        \vspace{4pt}
        \STATE $\cW \gets$ {\tt divide\_and\_conquer} \big (X, $\boldsymbol{a}, \boldsymbol{b}, z_{\min}, z_{\max}$ \big ) $\quad \quad  //$ Algorithm \ref{alg:divide_and_conquer}
        \vspace{4pt}
        \STATE Identify $\cZ \gets $ Eq. \eq{eq:cZ_new} with $\cW$
        \vspace{4pt}
        \STATE Compute $p^{\rm selective}_{j} \gets$ Eq. (\ref{eq:selective_p_reformulated}) with $\cZ$
        \vspace{4pt}
    \ENDFOR
\end{algorithmic}
\textbf{Output:} $\{p^{\rm selective}_{i}\}_{i \in \mathcal{M}_{\rm obs}}$
\end{footnotesize}
\end{algorithm}

\begin{algorithm}[!t]
\renewcommand{\algorithmicrequire}{\textbf{Input:}}
\renewcommand{\algorithmicensure}{\textbf{Output:}}
\begin{footnotesize}
 \begin{algorithmic}[1]
  \REQUIRE $X, \bm a, \bm b, z_{\rm min}, z_{\rm max}$
	\vspace{4pt}
	\STATE Initialization: $u = 1$, $v = 1$, $z_{u, v}  = z_{\rm min}$, $\cW = \emptyset$
	\vspace{4pt}
	\WHILE {$z_{u, v} < z_{\rm max}$}
		\vspace{4pt}
		\STATE Obtain $\cT_u \leftarrow$ OT-based DA on $\bm a + \bm b z_{u, v}$
		\vspace{4pt}
		\STATE Compute $[\ell_u, r_u] = \cZ_u \leftarrow$ Lemma \ref{lemma:cZ_u} 
		\vspace{4pt}
		\STATE $r_{u, v} = \ell_u$
		\vspace{4pt}
		\WHILE {$r_{u, v} < r_u$}
		\vspace{4pt}
		\STATE Compute $\tilde{X}_u$ and $\tilde{\bm Y}_u(z_{u, v}) \gets $ Lemma \ref{lemma:cZFS_v}
		\vspace{4pt}
		\STATE Obtain $\cO_v \leftarrow$ FS after DA on $\big (\tilde{X}_u, \tilde{\bm Y}_u(z_{u, v}) \big )$
		\vspace{4pt}
		\STATE $\cZ_v \leftarrow$ Lemma \ref{lemma:cZFS_v} 
		\vspace{4pt}
		\STATE $[\ell_{u, v}, r_{u, v}] = \cZ_{u, v} \leftarrow \cZ_u \cap \cZ_v$ 
		\vspace{6pt}
		\STATE $\cW \leftarrow \cW \cup \{ (u, v)\} $ \textbf{if} $\cA \big(\cT_u, \cO_v \big ) = \cM_{\rm obs}$
		\vspace{4pt}
		\STATE $v \leftarrow v + 1$, $z_{u, v} = r_{u, v}$
		%
%		\IF {$r_{u, v} = r_u$}
%		\vspace{4pt}
%		\STATE $v \leftarrow 1$, $u \leftarrow u + 1$, $z_{u, v} = r_{u, v}$, \textbf{break}
%		\vspace{4pt}
%		\ENDIF
		%
%		\STATE \textbf{if} {$R_{u, v} = R_u$} \textbf{then} ${\rm break}$ \textbf{end if}
		%
%		\vspace{4pt}
%		\STATE $v \leftarrow v + 1$, $z_{u, v} = r_{u, v}$
		\vspace{4pt}
		\ENDWHILE	
		\vspace{4pt}
		\STATE $v \leftarrow 1$, $u \leftarrow u + 1$, $z_{u, v} = r_{u, v}$
		\vspace{4pt}
	\ENDWHILE
	\vspace{2pt}
  \ENSURE $\cW$ 
 \end{algorithmic}
\end{footnotesize}
\caption{{\tt divide\_and\_conquer}}
\label{alg:divide_and_conquer}
\end{algorithm}

\subsubsection{Computation of $\cZ$ in \eq{eq:cZ} by combining multiple sub-problems} \label{subsubsec:compute_cZ}

To identify $\cW$ in \eq{eq:cW}, the OT-based DA and SeqFS after DA are iteratively performed on a sequence of datasets
$\bm a + \bm b z$, over an adequately broad range of $z \in [z_{\rm min}, z_{\rm max}]$\footnote{We set $z_{\rm min} = -20\sigma$ and $z_{\rm max} = 20 \sigma$, $\sigma$ is the standard deviation of the distribution of the test statistic, because the probability mass outside this range is negligibly small.}.
For simplicity, we focus on the case where each of $\cZ_u$ and $\cZ_v$ consists of a single interval\footnote{If $\cZ_u$ or $\cZ_v$ is a union of intervals, we can select the interval containing the data point that we are currently considering.}. Consequently, $\cZ_{u,v}$ is also a single interval. 
We denote $\cZ_u = [\ell_u, r_u]$ and $\cZ_{u, v} = [\ell_{u, v}, r_{u, v}]$.
The divide-and-conquer procedure is summarized in Algorithm \ref{alg:divide_and_conquer}.
Once $\cW$ in \eq{eq:cW} is obtained through Algorithm \ref{alg:divide_and_conquer}, we can compute $\cZ$ in \eq{eq:cZ_new}, which is then used to calculate the proposed selective $p$-value in \eq{eq:selective_p_reformulated}.
The complete steps of the proposed method are outlined in Algorithm \ref{alg:si_seqfs_da}.

%% file: sec4.tex
\section{Extensions to Backward SeqFS and Criteria for Optimal Model Selection} \label{sec:extensions}

In this section, we further enhance the proposed SI-SeqFS-DA method by extending it to Backward SeqFS and incorporating optimal model selection criteria, such as AIC, BIC, and adjusted $R^2$, to automatically determine the number of selected features $K$.

\subsection{Backward SeqFS}

The Backward SeqFS algorithm begins with a model that incorporates all available features and progressively assesses their relevance. 
In each iteration, it removes the least relevant feature based on a predefined criterion. 
This process continues until the feature set is reduced to a specified number $K$ of selected features.
At step $k \in \{p, p - 1, \dots, K + 1\}$, the feature $j_k$ is removed as 
\begin{align}
    j_k = \argmin \limits_{j \in \mathcal{M}_{k}} ~ {\rm RSS}\Big (\tilde{\bm Y}, \tilde{X}_{\mathcal{M}_{k} \setminus \{j\}} \Big ),
\end{align}
where $\mathcal{M}_p = [p].$
Similar to the truncation region $\cZ$ described in \eq{eq:cZ} of Sec. \ref{sec:proposed_method}, the truncation region $\cZ^{\rm bw}$ for the case of SI in Backward SeqFS-DA is defined as:
\begin{align} \label{eq:cZ_backward}
	\cZ^{\rm bw} 
	& = \Big \{ 
	z \in \RR 
	\mid 
	\cM^{\rm bw}_{\bm a + \bm b z} = \cM_{\rm obs}^{\rm bw} 
	\Big \} \nonumber \\ 
	& = 
	\bigcup \limits_{(u, v) \in \cW^{\rm bw} }
	\Big \{ 
	z \in \RR 
	\mid 
	\cT_{\bm a + \bm b z} = \cT_u,
	\cO^{\rm bw} _{\bm a + \bm b z} = \cO^{\rm bw} _v
	\Big \}, 
\end{align}
where 
$
	\cW^{\rm bw} = 
	\Big \{ 
		(u, v) \mid \cA \big(\cT_u, \cO_v ^{\rm bw}\big ) = \cM_{\rm obs}^{\rm bw}
	\Big \}
$.
Here, $\cM^{\rm bw}_{\bm a + \bm b z}$ denotes the set of $K$ features selected by Backward SeqFS-DA on the data $\bm a + \bm b z$, the order sequence $\cO^{\rm bw}_{\bm a + \bm b z}$ in the backward feature elimination process is defined as:
\begin{align*}
	\cO^{\rm bw}_{\bm a + \bm b z}
	= \Big (
	\cM_p, \cM_{p - 1}, ..., \cM_{K}
	\Big ).
\end{align*}
According to \eq{eq:cZ_backward}, the sub-problem is defined as follows:
\begin{align} \label{eq:cZbw_extra_condition}
	\cZ^{\rm bw}_{u, v} = 
	\Big \{z \in \RR 
	\mid 
	\cT_{\bm a + \bm b z} = \cT_u,
	\cO^{\rm bw}_{\bm a + \bm b z} = \cO^{\rm bw}_v 
	\Big \},
\end{align}
which can be decomposed into
\begin{align} \label{eq:Z_u_and_Z_v_bw}
	\cZ^{\rm bw}_{u, v}  = \cZ_u \cap \cZ^{\rm bw}_v,
\end{align}
where 
$
	\cZ_u  = 
	\Big \{ 
	z \in \RR
	\mid 
	\cT_{\bm a + \bm b z} = \cT_u
	\Big \}
$ is derived from Lemma \ref{lemma:cZ_u}
and
$
	\cZ_v^{\rm bw} = 
	\Big \{ 
	z  \in \RR
	\mid 
	\cO_{\bm a + \bm b z}^{\rm bw} = \cO_v^{\rm bw}
	\Big \}.
$ is obtained from the following Lemma \ref{lemma:cZBS_v}.

\begin{lemma}\label{lemma:cZBS_v}

Let us define the backward SeqFS problem after $\cT_u$ transportation as, for any $k \in \{p, p - 1, \dots, K + 1\}$:
\begin{align*}
    j_k = \argmin \limits_{j \in \mathcal{M}_{k}} 
{\rm RSS} \Big (\tilde{\bm Y}_u(z), \tilde{X}_{u_{\mathcal{M}_{k} \setminus \{j\}}} \Big ),
\end{align*}
where $\tilde{\bm Y}_u(z) = \Omega_u \bm Y (z)$ and $\tilde{X}_{u_{\mathcal{M}_{k} \setminus \{j\}}} = (\Omega_u X)_{\mathcal{M}_{k} \setminus \{j\}}$.
The set $\cZ_v$ in \eq{eq:Z_u_and_Z_v_bw} can be characterized by a set of quadratic inequalities w.r.t. $z$ described as follows:
\begin{align*}
    \cZ_v^{\rm bw} = \Big \{ 
	z \in \RR 
	\mid 
	\bm w^{\rm bw} + \bm r^{\rm bw} z + \bm o^{\rm bw} z^2 \leq \bm 0
	\Big \},
\end{align*}
where the vectors $\bm w^{\rm bw}$, $\bm r^{\rm bw}$, and $\bm o^{\rm bw}$ are defined in Appendix \ref{app:proof_Zv_bw}.

\end{lemma}

\subsection{Criteria for Automatically Determining the Number of Selected Features}

\paragraph{\textbf{AIC}.}
The goal of the AIC criterion for SeqFS-DA is to identify the model with the lowest AIC score among the models $\mathcal{M}_K$ selected by SeqFS after DA, where $K \in \{1,\dots, p\}$. 
In the case of the model with Gaussian errors, the AIC score for a $\cM_K$ is defined as
\begin{align*}
   {\rm AIC} \big (\cM_K, \tilde{\bm Y} \big ) = 
    \big (\tilde{\bm Y} - \tilde{X}_{\cM_K} \hat{\bm \beta}_{\cM_K} \big)^\top
    \Sigma^{-1}
    \big (\tilde{\bm Y} - \tilde{X}_{\cM_K} \hat{\bm \beta}_{\cM_K} \big) + 2 |\cM_K| 
\end{align*}
where 
$|\cM_K|$ is the number of selected features and 
$
\hat{\bm \beta}_{\cM_K}
= 
\big (\tilde{X}_{\cM_K}^\top \tilde{X}_{\cM_K} \big )^{-1} \tilde{X}_{\cM_K}^\top \tilde{\bm Y}.
$
The AIC can be re-written as:
\begin{align*}
	{\rm AIC}\big (\cM_K, \tilde{\bm Y} \big) = 
	\tilde{\bm Y}^T P_{\tilde{X}_{\cM_K}}^{\perp^\top} \Sigma^{-1} P_{\tilde{X}_{\cM_K}}^{\perp} \tilde{\bm Y}
	+ 2 |\cM_K|, 
\end{align*}
where 
$
P_{\tilde{X}_{\cM_K}}^{\perp} = 
I_n - \tilde{X}_{\cM_K} (\tilde{X}_{\cM_K}^\top \tilde{X}_{\cM_K} \big )^{-1} \tilde{X}_{\cM_K}^\top
$.
Let us denote by $\cM_{\hat{K}^{\rm AIC}}$ is the optimal model selected by AIC, the additional event of AIC-based model selection after $\cT_u$ transportation is defined as
\begin{align*}
	\cZ_u^{\rm AIC} = 
	\left \{ 
		z \in \RR \mid 
		{\rm AIC}\Big (\cM_{\hat{K}^{\rm AIC}}, \tilde{\bm Y}_u(z) \Big) \leq {\rm AIC}\Big (\cM_{K}, \tilde{\bm Y}_u(z) \Big) \text{ for any } K \in [p]
	\right \}.
\end{align*}
Since the ${\rm AIC}\Big (\cM_{ {K}}, \tilde{\bm Y}_u(z) \Big)$ is a quadratic function w.r.t. $z$, the $\cZ_u^{\rm AIC}$ can be identified by solving the system of quadratic inequalities.

\paragraph{\textbf{BIC}.} The BIC, derived from a Bayesian perspective, is similar to the AIC but includes an additional penalty term that accounts for the sample size.
The BIC is defined as:
\begin{align*}
    {\rm BIC} \big (\cM_K, \tilde{\bm Y} \big ) &= 
    \big (\tilde{\bm Y} - \tilde{X}_{\cM_K} \hat{\bm \beta}_{\cM_K} \big)^\top
    \Sigma^{-1}
    \big (\tilde{\bm Y} - \tilde{X}_{\cM_K} \hat{\bm \beta}_{\cM_K} \big) + \log(n_s + n_t) |\cM_K| \notag\\
    &= \tilde{\bm Y}^T P_{\tilde{X}_{\cM_K}}^{\perp^\top} \Sigma^{-1} P_{\tilde{X}_{\cM_K}}^{\perp} \tilde{\bm Y} + \log(n_s + n_t)|\cM_K|. %\label{eq:bic}
\end{align*}
Let us denote by $\cM_{\hat{K}^{\rm BIC}}$ is the optimal model selected by BIC, the additional event of BIC-based model selection after $\cT_u$ transportation is defined as
\begin{align*}
	\cZ_u^{\rm BIC} = 
	\left \{ 
		z \in \RR \mid 
		{\rm BIC}\Big (\cM_{\hat{K}^{\rm BIC}}, \tilde{\bm Y}_u(z) \Big) \leq {\rm BIC}\Big (\cM_{K}, \tilde{\bm Y}_u(z) \Big) \text{ for any } K \in [p]
	\right \}.
\end{align*}
Similar to the case of AIC, the $\cZ_u^{\rm BIC}$ can be identified by solving the system of quadratic inequalities.

\paragraph{\textbf{Adjusted $R^2$.}} The adjusted $R^2$ is another commonly used criterion for model selection, which is defined as follows:
\begin{align*}
    \text{Adjusted $R^2$} &= 1 - \frac{{\rm RSS}\big (\tilde{\bm Y}, \tilde{X}_{\cM_K} \big) / (n - |\cM_K| -1)}{{\rm TSS}\big(\tilde{\bm Y}\big)/ (n -1) },
\end{align*}
where ${\rm TSS}\big(\tilde{\bm Y}\big)$ is the \emph{total sum of squares} for the $\tilde{\bm Y}$.
Maximizing the adjusted $R^2$ is equivalent to minimizing 
\begin{align*}
%\label{eq:adjr2}
    \frac{{\rm RSS}\big(\tilde{\bm Y}, \tilde{X}_{\cM_{K}}\big)}{n - |\cM_K| - 1} 
    = 
    \frac{\tilde{\bm Y}^\top P_{\tilde{X}_{\cM_K}}^{\perp^\top} P_{\tilde{X}_{\cM_K}}^{\perp} \tilde{\bm Y}}{n - |\cM_K| - 1}.
\end{align*}
Let us denote by $\cM_{\hat{K}^{\rm adj}}$ is the optimal model selected by the adjusted $R^2$, the additional event of adjusted $R^2$-based model selection after $\cT_u$ transportation is:
\begin{align*}
	\cZ_u^{\rm adj} = 
	\left \{ 
		z \in \RR ~ \Big | ~
		\frac{{\rm RSS}\big(\tilde{\bm Y}_u(z), \tilde{X}_{\cM_{\hat{K}^{\rm adj}}}\big)}{n - |\cM_{\hat{K}^{\rm adj}}| - 1} 
		\leq
		\frac{{\rm RSS}\big(\tilde{\bm Y}_u(z), \tilde{X}_{\cM_{K}}\big)}{n - |\cM_K| - 1} 
		\text{ for any } K \in [p]
	\right \},
\end{align*}
which can also be obtained by solving the system of quadratic inequalities w.r.t. $z$.

%The sub-problem of SFS after DA with stopping criterion is obtained by $\cZ_{u,v} \cap \cZ_v^{\text{crit}}$, 
%\begin{align*}
%\label{eq:ZC}
%    \cZ_v^{\text{crit}} &= \left\{ z \in \mathbb{R} \mid C_{\bm a +\bm b z} = C_v\right\} \notag 
%\\
%    &= \left\{ z \in \mathbb{R} \mid \bm d + \bm e z + \bm g z^2\right\}
%\end{align}
%Depending on the chosen criterion, $C$ could be $C^{AIC}$, $C^{BIC}$ or $C^{Adj R^2}$. The proof of \ref{eq:ZC} is deferred to Appendix 
% \ref{}
% The sub-problem of SFS after DA with stopping criterion is obtained by $\cZ_{u,v} \cap \cZ^{\text{crit}}$ which $\cZ^{\text{crit}}$ represents $\cZ^{AIC}$, $\cZ^{BIC}$ or $\cZ^{Adj R^2}$ depending the chosen criterion. 

%% file: sec5.tex
\section{Experiments}
\label{sec:experiments}

In this section, we demonstrate the performance of the proposed SI-SeqFS-DA approach.
We considered the following methods:

\begin{itemize}
	\item[$\bullet$] {\tt SI-SeqFS-DA:} the proposed method
	
	\item[$\bullet$] {\tt SI-SeqFS-DA-oc:} the proposed method focuses exclusively on a single sub-problem, i.e., over-conditioning, described in \S\ref{subsubsec:solving_each_subproblem}. This can be considered as an extension of the framework presented in \cite{lee2016exact} to our setting
	
	\item[$\bullet$] {\tt DS:} data splitting 
	
	\item[$\bullet$] {\tt Bonferroni:} the most popular multiple testing
	
	\item[$\bullet$] {\tt Naive:} traditional statistical inference
	
	\item[$\bullet$] {\tt No inference:} SeqFS after DA without inference
\end{itemize}

It is important to note that if a method fails to control the FPR at the desired level $\alpha$, it is considered \emph{invalid}, and its TPR becomes irrelevant. Additionally, a method with a high TPR inherently implies a low FNR.
In all experiments, we set $\alpha$ to 0.05.

\subsection{Numerical experiments}

\paragraph{\textbf{Synthetic data generation.}}

We generated the response vector $\bm{Y}^s$ in the source domain as
$\bm Y^s_i = {X^s_i}^\top \bm \beta^s + \veps$,
$X^s_i \sim \mathbb{N}(\bm 0, I_p), \forall i \in [n_s]$, and $\veps \sim \mathbb{N}(0, 1)$.
Similarly, the $\bm Y^t$ in the target domain was generated as 
$\bm Y^t_i = {X^t_i}^\top \bm \beta^t + \veps$ where
$X^t_i \sim \mathbb{N}(\bm 0, I_p)$.
We set $p = 5$ and $K = 3$. 
Regarding the FPR experiments, all elements of $\boldsymbol{\beta}^t$ were set to 0, and $n_s \in \{50, 100, 150, 200\}$. In regards to the TPR experiments, all elements of $\boldsymbol{\beta}^t$ in each test case were set to one of $\{1, 2, 3, 4\}$ and $n_s = 100$.
We set $n_t = 10$, indicating that the target data is limited.
In all experiments, the elements of $\bm{\beta}^s$ were set to 2. Since the statistical inference was conducted only on the target data, the values of $\bm{\beta}^s$ did not affect the results.
Each experiment was repeated 120 times.

\paragraph{\textbf{The results of FPRs and TPRs.}}
The results of FPR and TPR for the two cases of  Forward and Backward SeqFS after DA are shown in Figs. \ref{fig:forward_fpr_tpr} and \ref{fig:backward_fpr_tpr}. In the plots on the left, the methods {\tt SI-SeqFS-DA}, {\tt SI-SeqFS-DA-oc}, {\tt Bonferroni}, and {\tt DS} successfully controlled the FPR, while the {\tt Naive} and {\tt No Inference} methods \textit{failed} to do so. Since the {\tt Naive} and {\tt No Inference} methods could not control the FPR, their TPRs are not considered further.
In the plots on the right, the {\tt SI-SeqFS-DA} method achieves the highest TPR compared to the other methods in all cases, indicating that it has the lowest FNR.
We additionally present the FPRs and TPRs for the SeqFS-DA using the AIC, BIC, and adjusted $R^2$ in Figs. \ref{fig:FS_AIC_fpr_tpr}, \ref{fig:BS_AIC_fpr_tpr}, \ref{fig:FS_BIC_fpr_tpr}, \ref{fig:BS_BIC_fpr_tpr}, \ref{fig:FS_AdjR2_fpr_tpr}, and \ref{fig:BS_AdjR2_fpr_tpr}.
In all cases, the proposed {\tt SI-SeqFS-DA} successfully controlled the FPR while achieving the highest TPR compared to the competitors.

% FORWARD K-Step FPR TPR

\begin{figure}[!t]
    \centering
    \begin{subfigure}[b]{0.47\linewidth}  % Reduced width for better spacing
        \includegraphics[width=\linewidth]{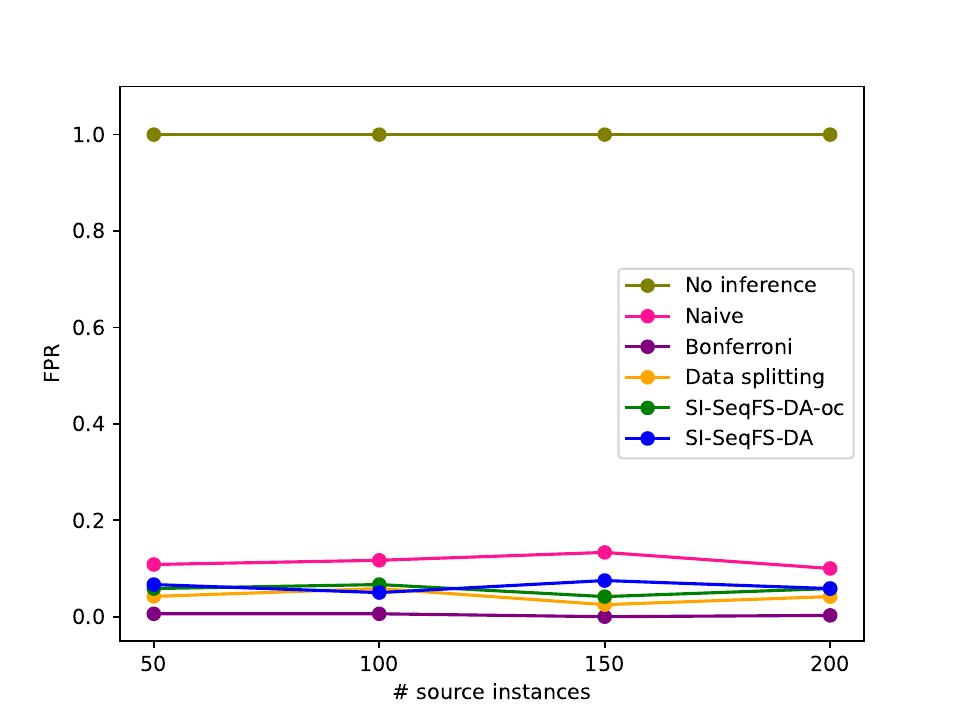}  % Make sure the path is correct
        \caption{FPR}
        \label{fig:fpr_univariate_forward}
    \end{subfigure}
    \hspace{0.02\linewidth}  % Add some space between subfigures
    \begin{subfigure}[b]{0.47\linewidth}
        \includegraphics[width=\linewidth]{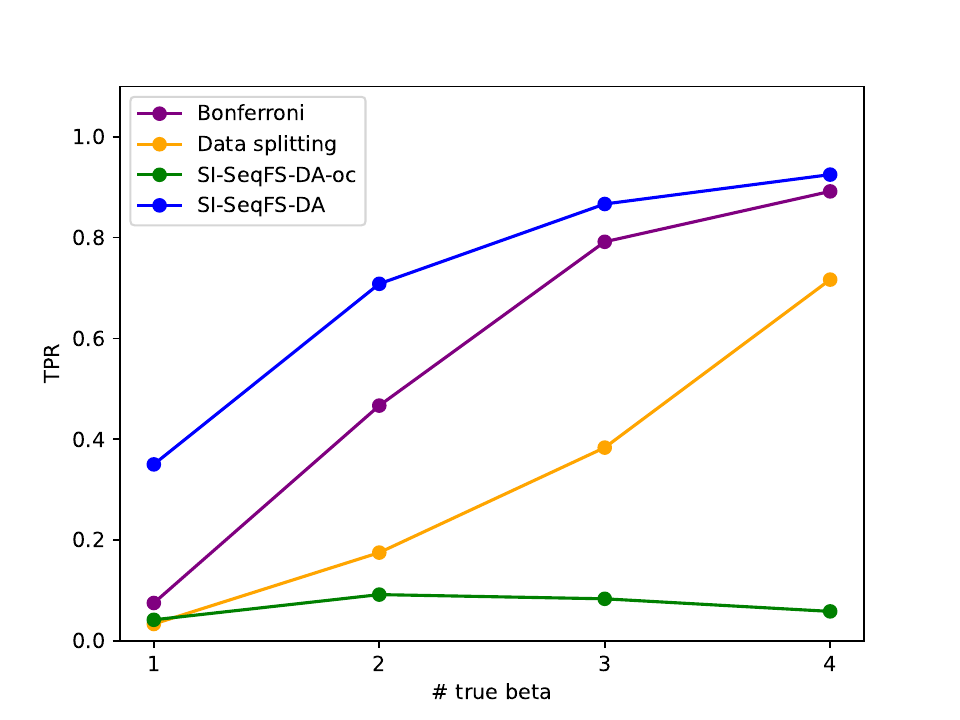}
        \caption{TPR}
        \label{fig:tpr_univariate_forward}
    \end{subfigure}
%    \vspace{-4pt}
    \caption{FPR and TPR in the case of Forward SeqFS}
    \label{fig:forward_fpr_tpr}
    \vspace{-10pt}
\end{figure}

% BACKWARD K-Step FPR TPR

\begin{figure}[!t]
    \centering
    \begin{subfigure}[b]{0.47\linewidth}  % Reduced width for better spacing
        \includegraphics[width=\linewidth]{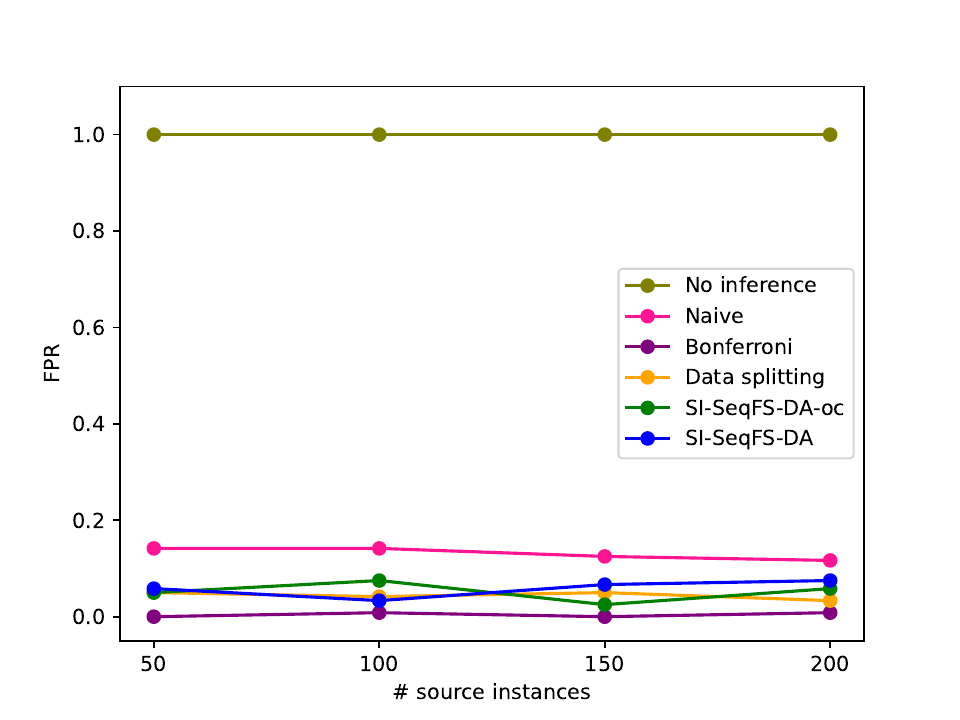}  % Make sure the path is correct
        \caption{FPR}
        \label{fig:fpr_univariate_backward}
    \end{subfigure}
    \hspace{0.02\linewidth}  % Add some space between subfigures
    \begin{subfigure}[b]{0.47\linewidth}
        \includegraphics[width=\linewidth]{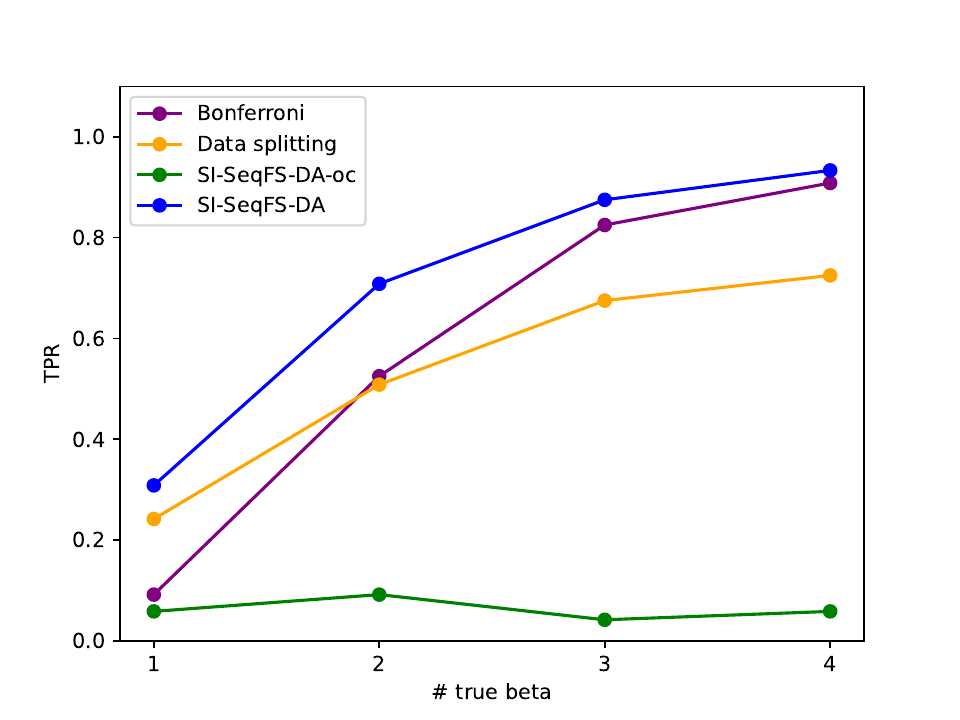}
        \caption{TPR}
        \label{fig:tpr_univariate_backward}
    \end{subfigure}
%    \vspace{-4pt}
    \caption{FPR and TPR in the case of Backward SeqFS}
    \label{fig:backward_fpr_tpr}
    \vspace{-10pt}
\end{figure}
% 
% 
% 
% 
% 
% 
% 

% 
% 
% 
% 
% 
% 

% \red{Results on AIC}
% FORWARD SeqFS with Criteria
% AIC
\begin{figure}[!t]
    \centering
    \begin{subfigure}[b]{0.47\linewidth}  % Reduced width for better spacing
        \includegraphics[width=\linewidth]{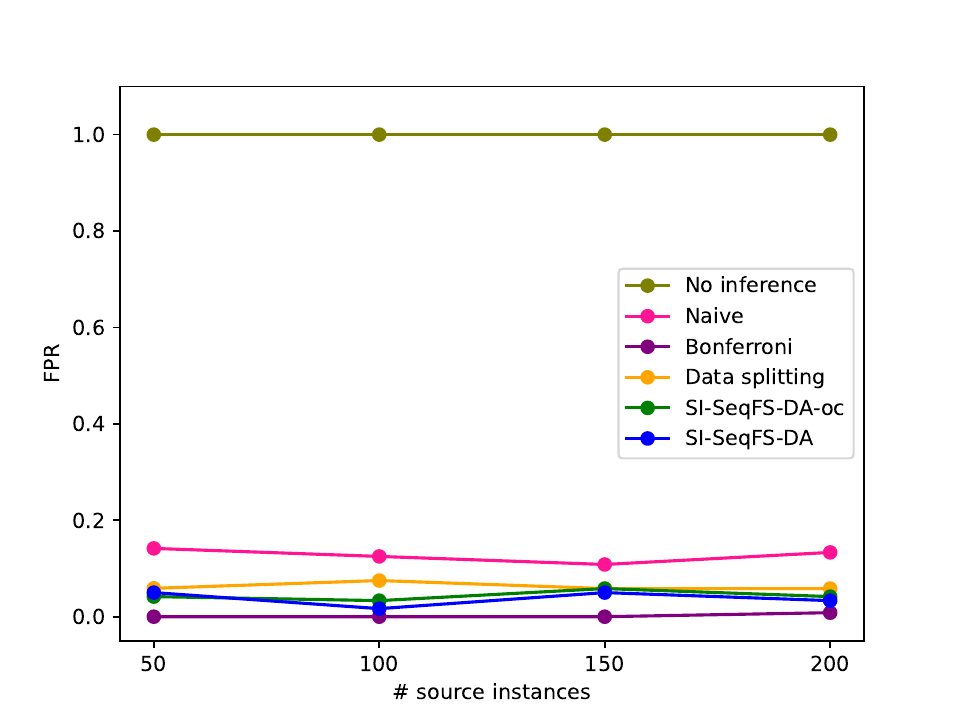}  % Make sure the path is correct
        \caption{FPR}
        \label{fig:FPR_FS_AIC}
    \end{subfigure}
    \hspace{0.02\linewidth}  % Add some space between subfigures
    \begin{subfigure}[b]{0.47\linewidth}
        \includegraphics[width=\linewidth]{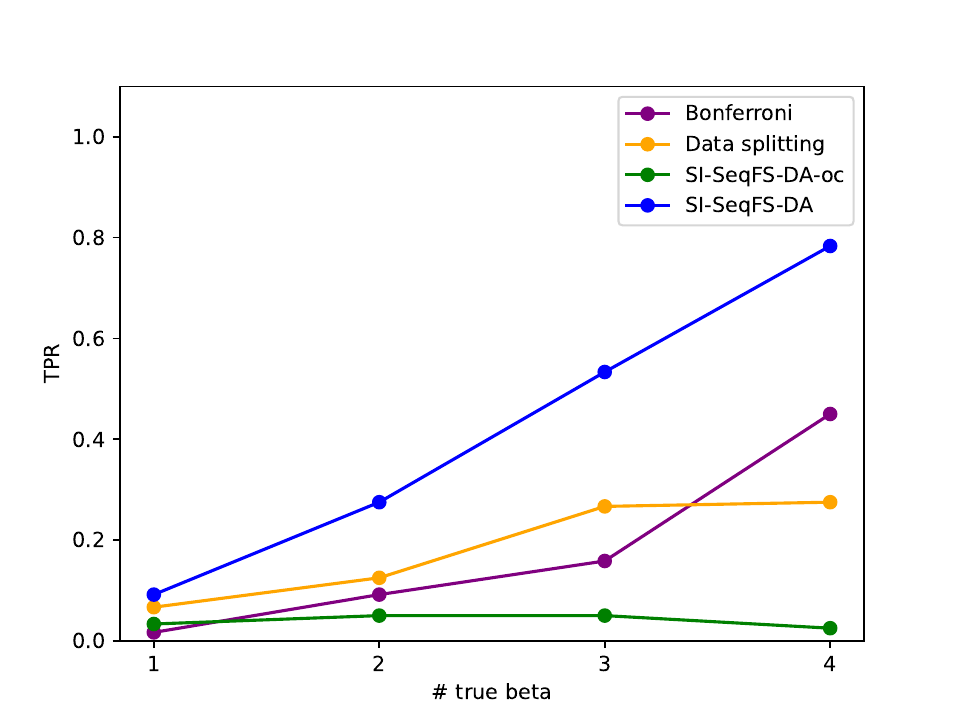}
        \caption{TPR}
        \label{fig:TPR_FS_AIC}
    \end{subfigure}
%    \vspace{-4pt}
    \caption{FPR and TPR in the case of Forward SeqFS with AIC}
    \label{fig:FS_AIC_fpr_tpr}
    \vspace{-10pt}
\end{figure}

% BACKWARD SeqFS with Criteria
% AIC
\begin{figure}[!t]
    \centering
    \begin{subfigure}[b]{0.47\linewidth}  % Reduced width for better spacing
        \includegraphics[width=\linewidth]{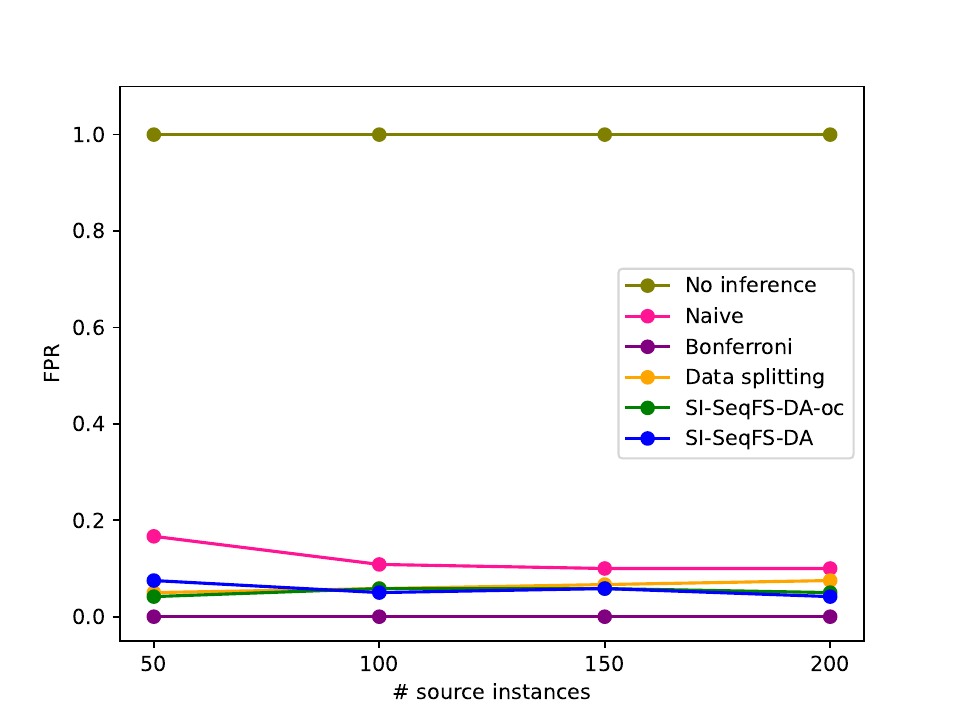}  % Make sure the path is correct
        \caption{FPR}
        \label{fig:FPR_BS_AIC}
    \end{subfigure}
    \hspace{0.02\linewidth}  % Add some space between subfigures
    \begin{subfigure}[b]{0.47\linewidth}
        \includegraphics[width=\linewidth]{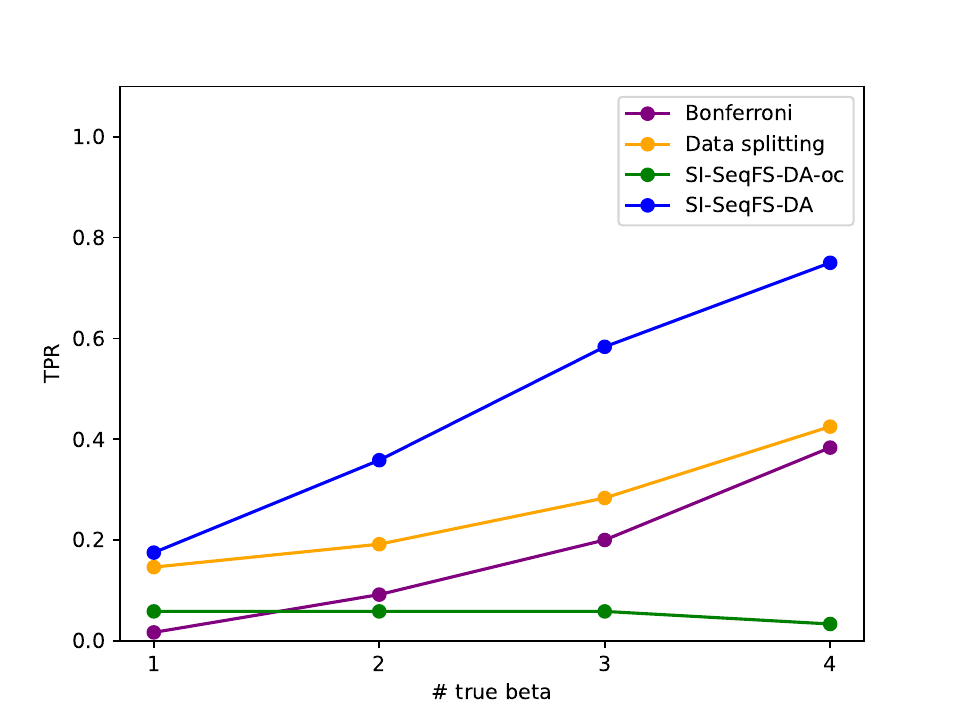}
        \caption{TPR}
        \label{fig:TPR_BS_AIC}
    \end{subfigure}
%    \vspace{-4pt}
    \caption{FPR and TPR in the case of Backward SeqFS with AIC}
    \label{fig:BS_AIC_fpr_tpr}
    \vspace{-10pt}
\end{figure}

% \red{Results on BIC}
% FORWARD BIC
% BIC
\begin{figure}[!t]
    \centering
    \begin{subfigure}[b]{0.47\linewidth}  % Reduced width for better spacing
        \includegraphics[width=\linewidth]{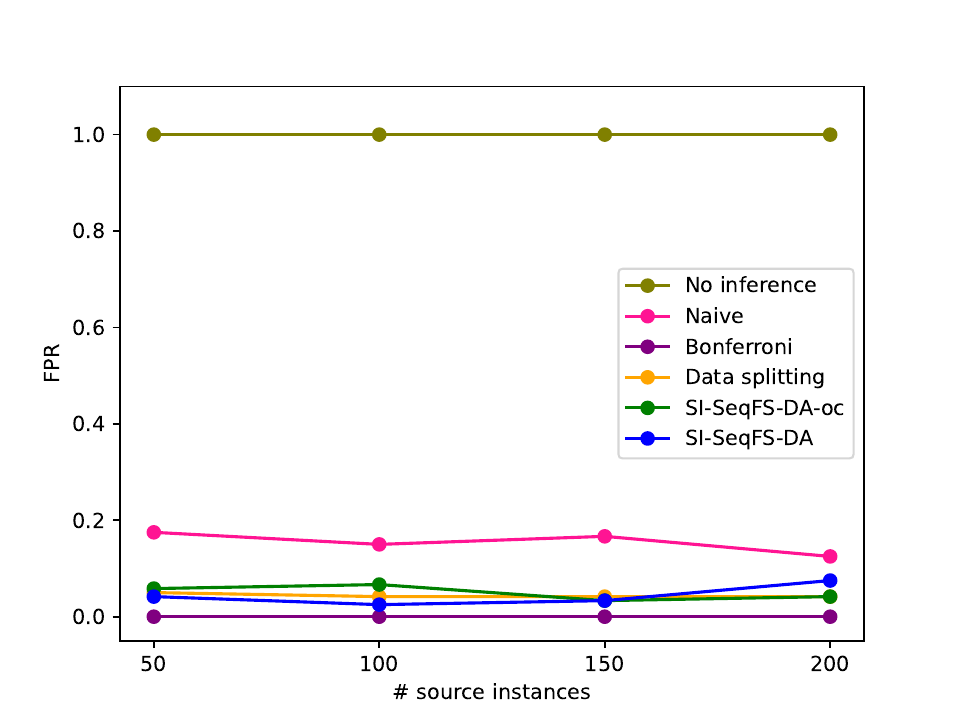}  % Make sure the path is correct
        \caption{FPR}
        \label{fig:FPR_FS_BIC}
    \end{subfigure}
    \hspace{0.02\linewidth}  % Add some space between subfigures
    \begin{subfigure}[b]{0.47\linewidth}
        \includegraphics[width=\linewidth]{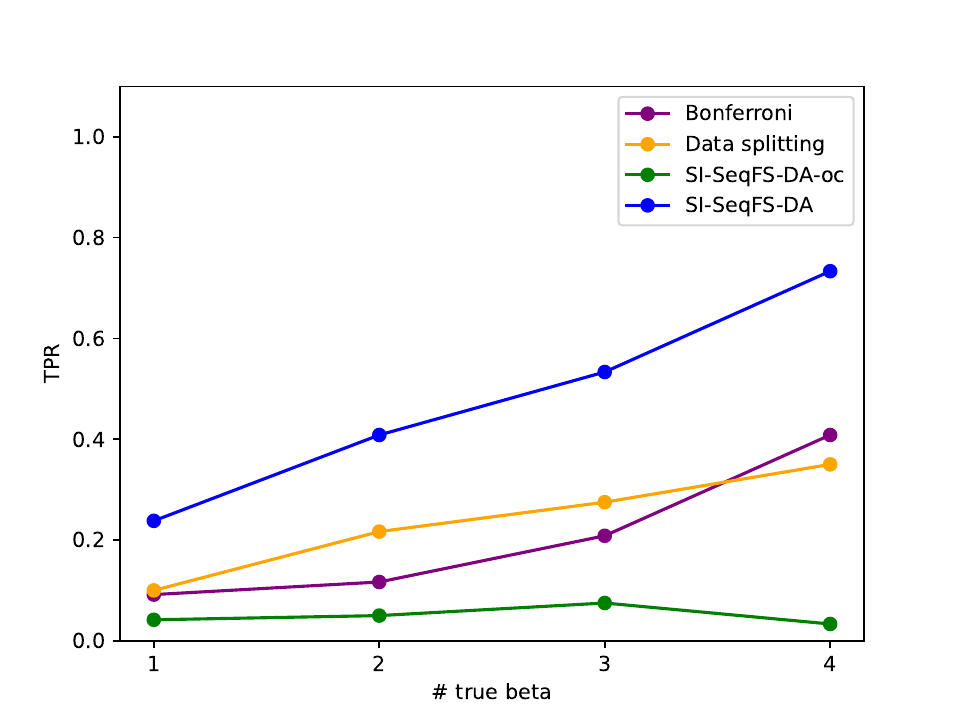}
        \caption{TPR}
        \label{fig:TPR_FS_BIC}
    \end{subfigure}
%    \vspace{-4pt}
    \caption{FPR and TPR in the case of Forward SeqFS with BIC}
    \label{fig:FS_BIC_fpr_tpr}
    \vspace{-10pt}
\end{figure}
% BACKWARD
% BIC
\begin{figure}[!t]
    \centering
    \begin{subfigure}[b]{0.47\linewidth}  % Reduced width for better spacing
        \includegraphics[width=\linewidth]{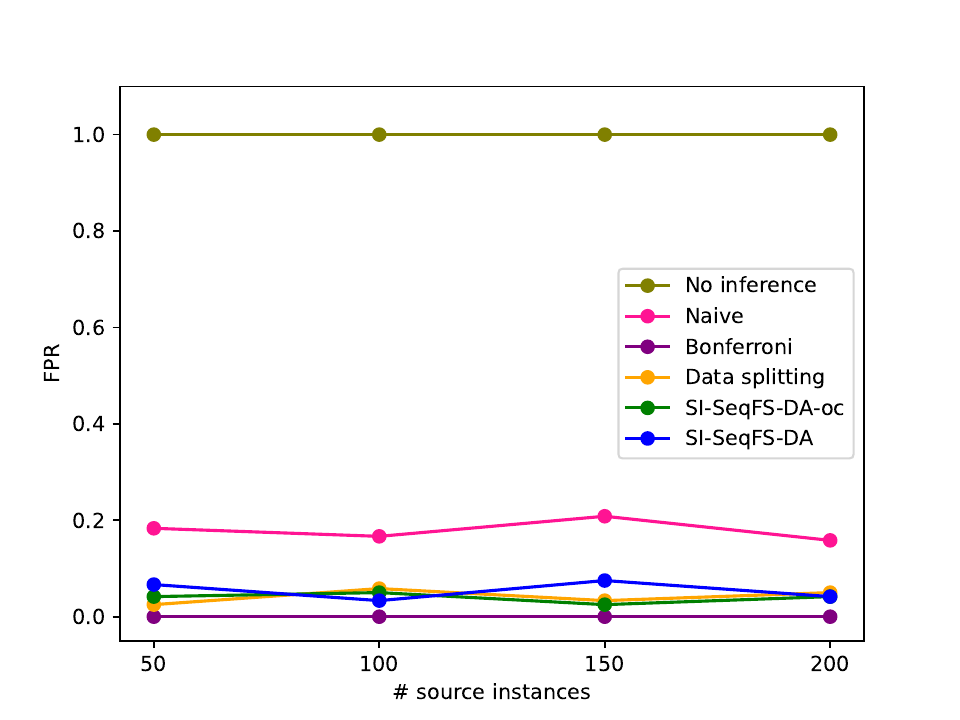}  % Make sure the path is correct
        \caption{FPR}
        \label{fig:FPR_BS_BIC}
    \end{subfigure}
    \hspace{0.02\linewidth}  % Add some space between subfigures
    \begin{subfigure}[b]{0.47\linewidth}
        \includegraphics[width=\linewidth]{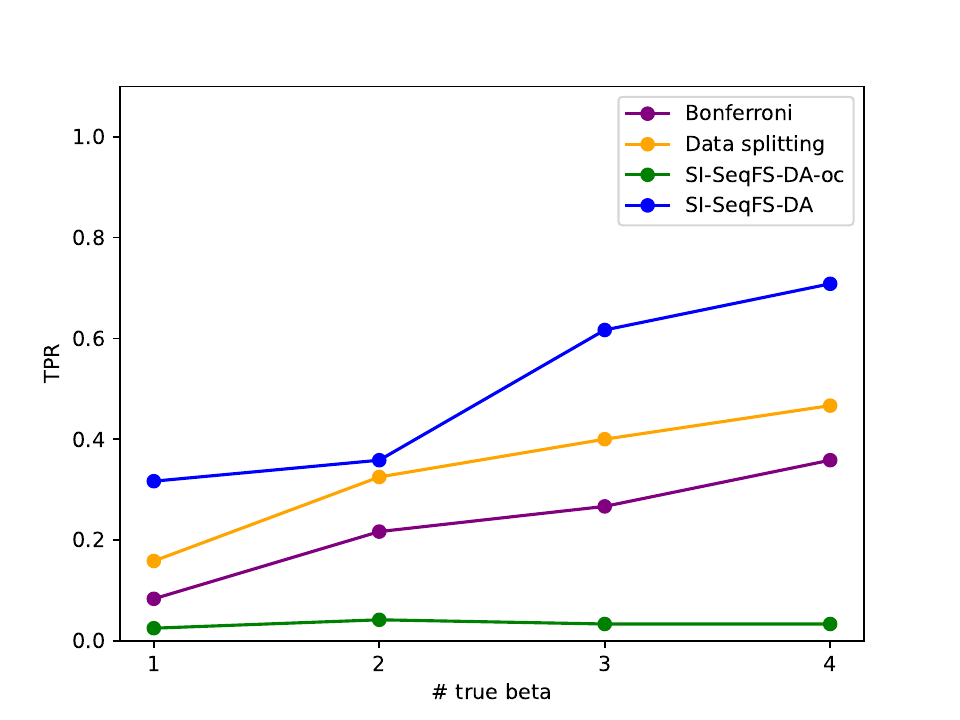}
        \caption{TPR}
        \label{fig:TPR_BS_BIC}
    \end{subfigure}
%    \vspace{-4pt}
    \caption{FPR and TPR in the case of Backward SeqFS with BIC}
    \label{fig:BS_BIC_fpr_tpr}
    \vspace{-10pt}
\end{figure}

% \red{Results on Adjusted R2}
% FORWARD
% AdjR2
\begin{figure}[!t]
    \centering
    \begin{subfigure}[b]{0.47\linewidth}  % Reduced width for better spacing
        \includegraphics[width=\linewidth]{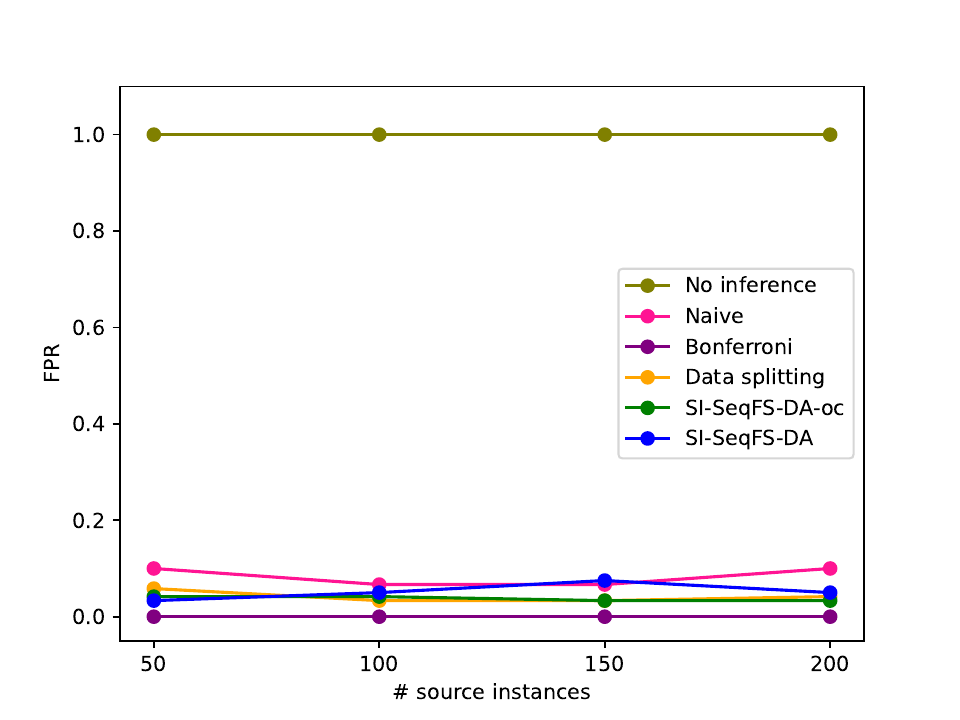}  % Make sure the path is correct
        \caption{FPR}
        \label{fig:FPR_FS_AdjR2}
    \end{subfigure}
    \hspace{0.02\linewidth}  % Add some space between subfigures
    \begin{subfigure}[b]{0.47\linewidth}
        \includegraphics[width=\linewidth]{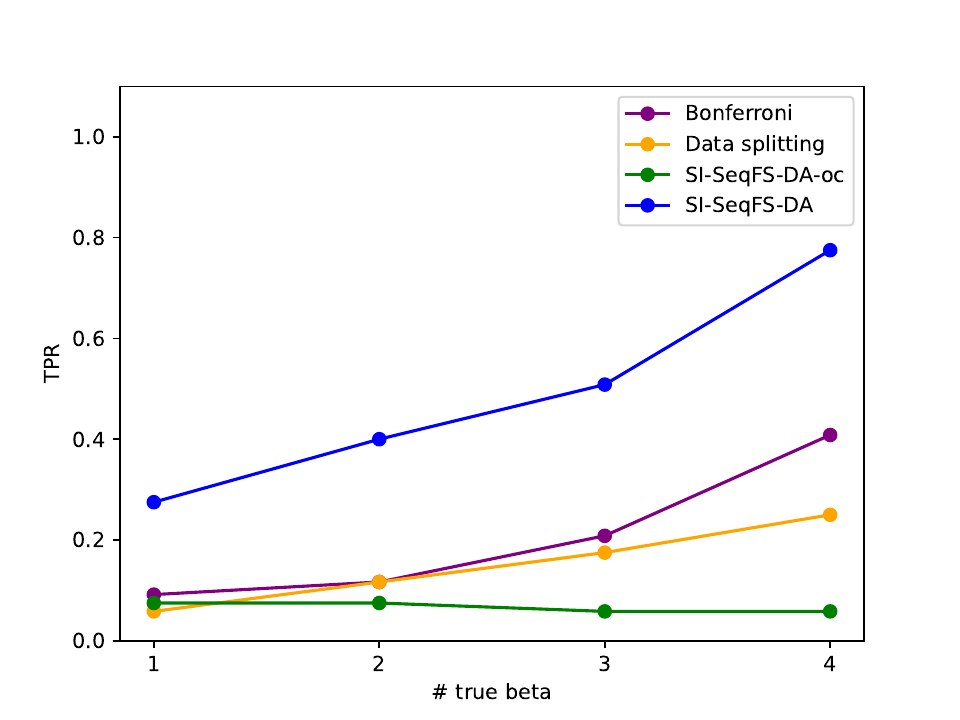}
        \caption{TPR}
        \label{fig:TPR_FS_AdjR2}
    \end{subfigure}
%    \vspace{-4pt}
    \caption{FPR and TPR in the case of Forward SeqFS with adjusted $R^2$}
    \label{fig:FS_AdjR2_fpr_tpr}
    \vspace{-10pt}
\end{figure}

% BACKWARD
% AdjR2
\begin{figure}[!t]
    \centering
    \begin{subfigure}[b]{0.47\linewidth}  % Reduced width for better spacing
        \includegraphics[width=\linewidth]{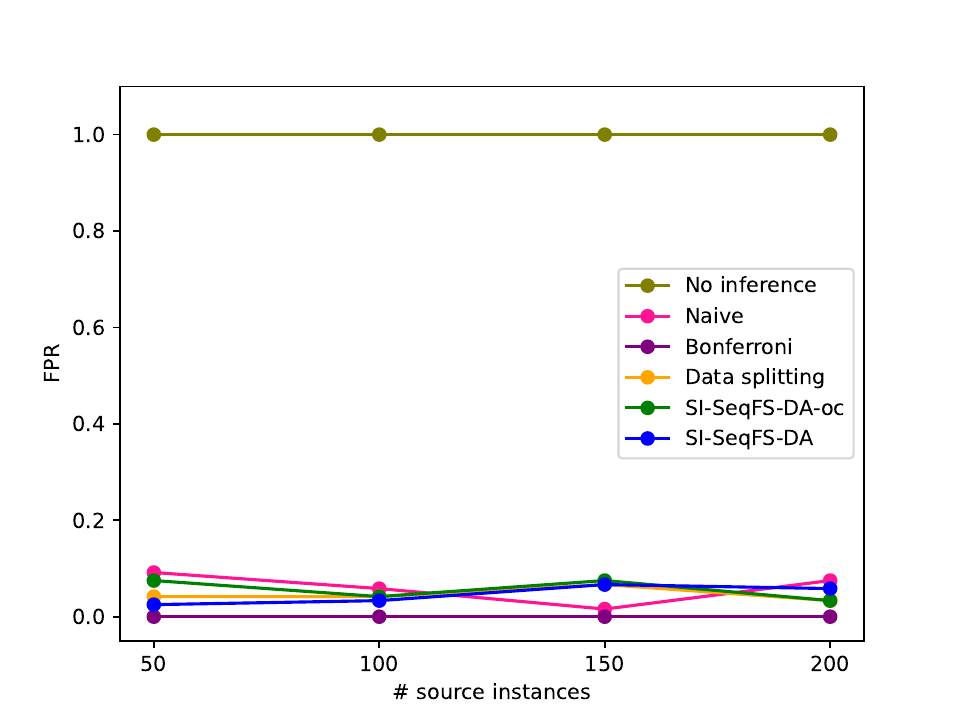}  % Make sure the path is correct
        \caption{FPR}
        \label{fig:FPR_BS_AdjR2}
    \end{subfigure}
    \hspace{0.02\linewidth}  % Add some space between subfigures
    \begin{subfigure}[b]{0.47\linewidth}
        \includegraphics[width=\linewidth]{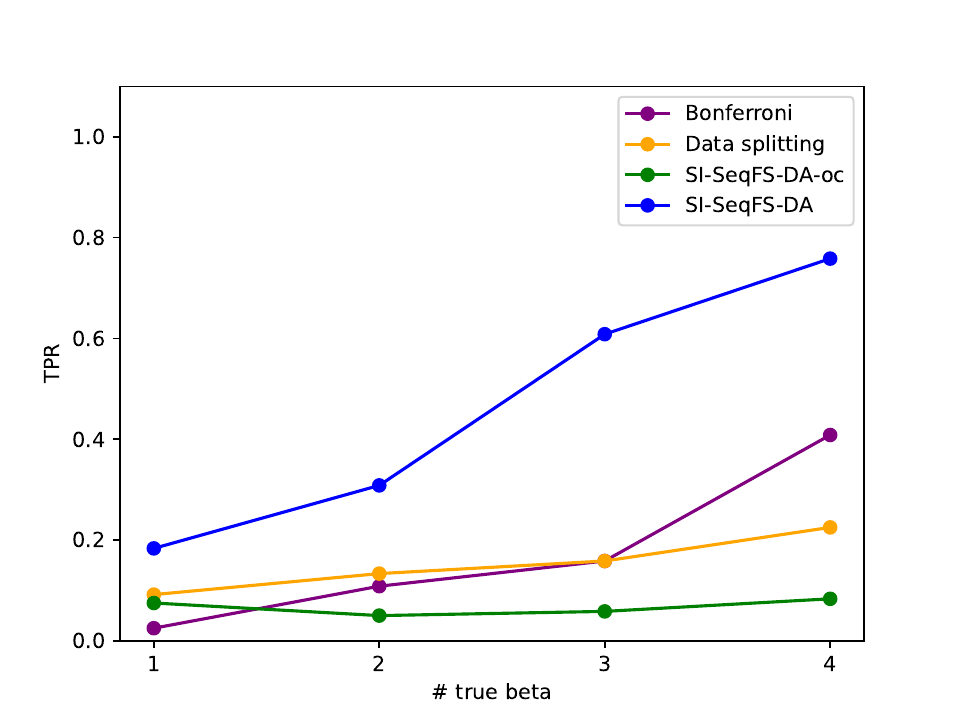}
        \caption{TPR}
        \label{fig:TPR_BS_AdjR2}
    \end{subfigure}
%    \vspace{-4pt}
    \caption{FPR and TPR in the case of Backward SeqFS with adjusted $R^2$}
    \label{fig:BS_AdjR2_fpr_tpr}
    \vspace{-10pt}
\end{figure}

\paragraph{\textbf{Computational time.}} In Figure \ref{fig:computation_cost}, we present boxplots showing the computation time for each \textit{p}-value, as well as the actual number of intervals of $z$ encountered on the line when constructing the truncation region $\mathcal{Z}$ with respect to $n_s$. The plots illustrate that the complexity of {\tt SI-SeqFS-DA} increases linearly w.r.t. $n_s$.

%\red{In the case of SI-SeqFS-DA with the AIC criterion, Figure \ref{fig:computation_cost_AIC} further illustrates the computation cost, emphasizing the linear increase in complexity with respect to $n_s$}

\begin{figure}[!t]
    \centering
    \begin{subfigure}[b]{0.47\linewidth}
        \includegraphics[width=\linewidth]{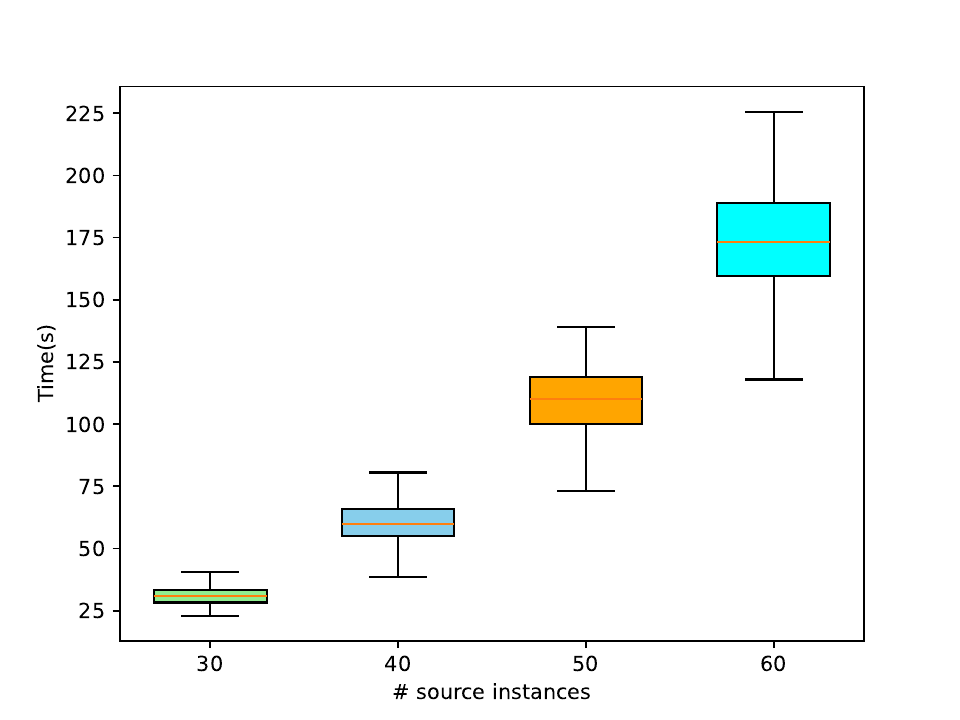}
        \caption{Computational time}
        \label{fig:time}
    \end{subfigure}
    \hspace{0.02\linewidth}
    \begin{subfigure}[b]{0.47\linewidth}
        \includegraphics[width=\linewidth]{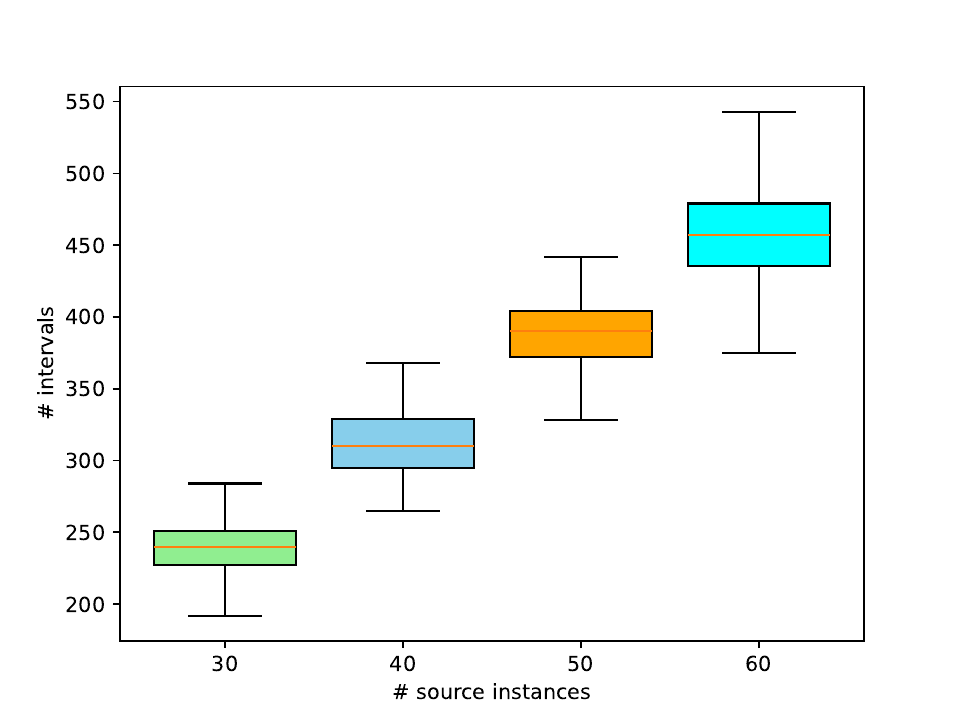}
        \caption{Encountered intervals}
        \label{fig:polytope}
    \end{subfigure}
%    \vspace{-4pt}
    \caption{Computational cost of the proposed SI-SeqFS-DA}
    \label{fig:computation_cost}
%    \vspace{-8pt}
\end{figure}

\paragraph{\textbf{Results on high-dimensional data.}}
We conducted experiments to evaluate the performance of SI-SeqFS-DA in high-dimensional settings. The setup included $p = \{1000, 1500, 2000, 2500 \}, n_s = 100$,  $n_t = 10$, and $K=4$. The results are shown in Fig. \ref{fig:FPR_highdim}.
We also present the computational cost in terms of time using a high-dimensional real-world dataset, as shown in Fig. \ref{fig:time_riboflavin}.
Specifically, we used the riboflavin production dataset, which contains 4,088 features and is available in \emph{High-Dimensional Statistics with a View Toward Applications in Biology}.
In this experiment, the source domain comprised data from \emph{riboflavinGrouped}, while the target domain comprised data from \emph{riboflavin}.
We set $p=4088$ and randomly selected instances from the source and target domains, with $n_s \in \{40, 60, 80, 100\}$. 

\begin{figure}[!t]
    \centering
    % \captionsetup[subfigure]{labelformat=empty}
    \begin{subfigure}[b]{0.47\linewidth}
        \includegraphics[width=\linewidth]{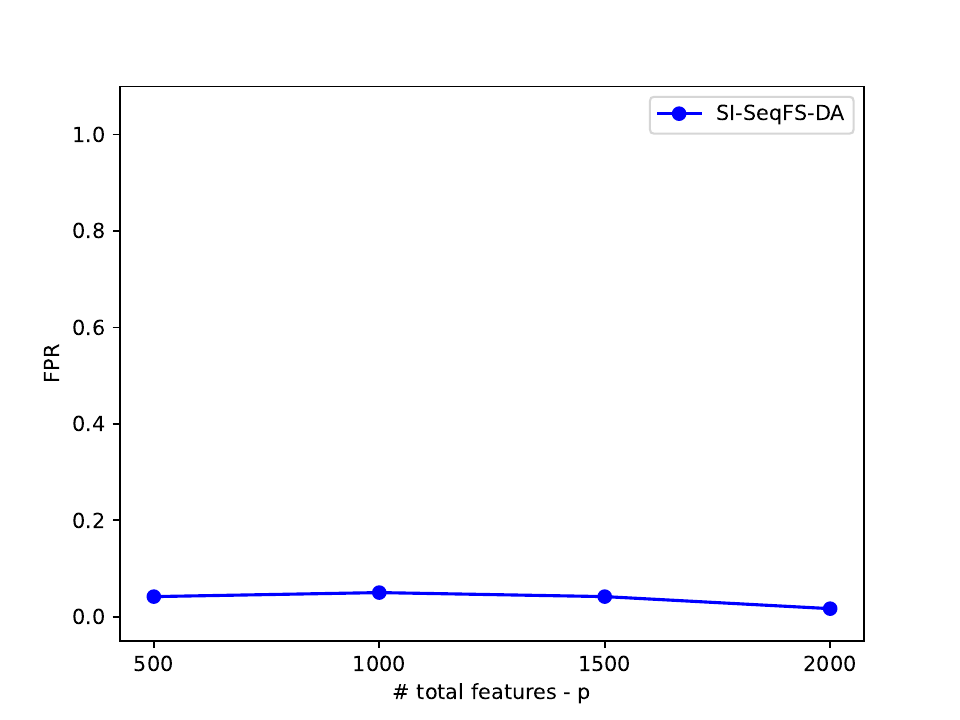}
        \caption{FPR in High-dimensional setting.}
        \label{fig:FPR_highdim}
    \end{subfigure}
    \hspace{0.02\linewidth}
    \begin{subfigure}[b]{0.47\linewidth}
        \includegraphics[width=\linewidth]{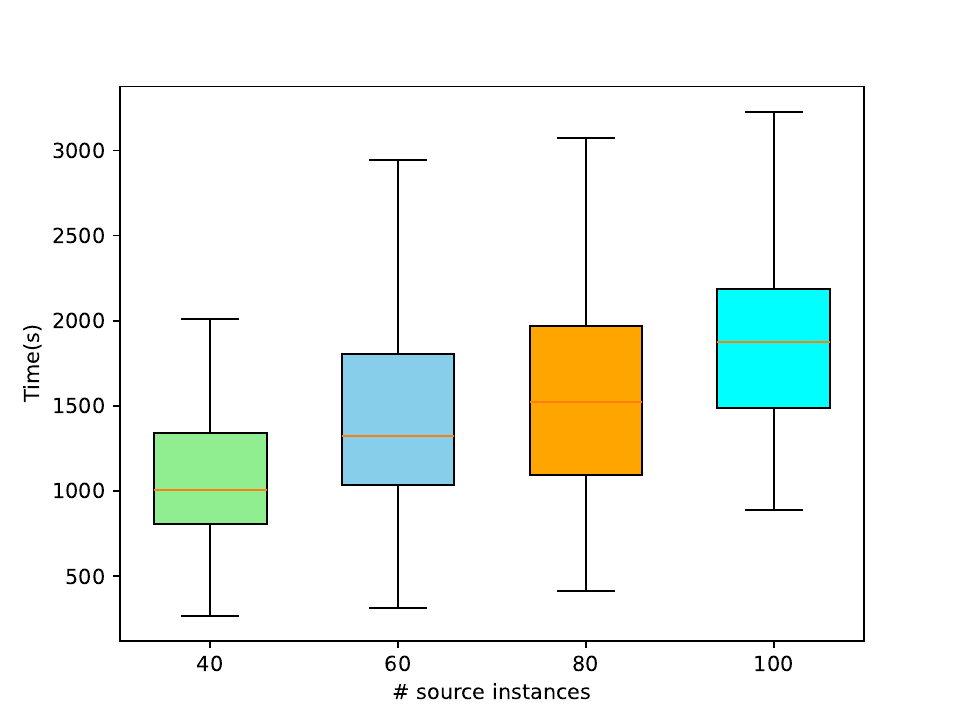}
        \caption{Computational time on riboflavin dataset}
        \label{fig:time_riboflavin}
    \end{subfigure}
%    \vspace{-4pt}
    \caption{SI-SeqFS-DA in high-dimensional setting}
    \label{fig:highdim}
%    \vspace{-8pt}
\end{figure}

\subsection{Results on Real-World Datasets}
\label{subsec:realdata}

We conducted comparisons on three real-world datasets: the Diabetes dataset \citep{efron2004least}, the Heart Failure dataset, and the Seoul Bike dataset, all of which are available in the UCI Machine Learning Repository.
In this section, we present the experimental results for Forward SeqFS after DA.
The additional results for the extended methods can be found in Appendix \ref{app:additional_experiments}.
For each dataset, we present the distribution of $p$-values for each feature.
We randomly selected instances from both the source and target domains, with $n_s = 100$ and $n_t = 15$.
Forward SeqFS was used for the analysis.
The results are shown in Figs. \ref{fig:diabetes}, \ref{fig:heart_failure}, and \ref{fig:Seoul_Bike}.
The $p$-values for {\tt Bonferroni} are nearly equal to one in almost all cases, indicating that this method is \emph{conservative}.
In all cases, the $p$-value of the proposed {\tt SI-SeqFS-DA} tends to be smaller than those of the competing methods. This demonstrates that {\tt SI-SeqFS-DA} exhibits the highest statistical power.

% Second row of subfigures 
\begin{figure}[!t]
    \centering

    \includegraphics[width=\linewidth]{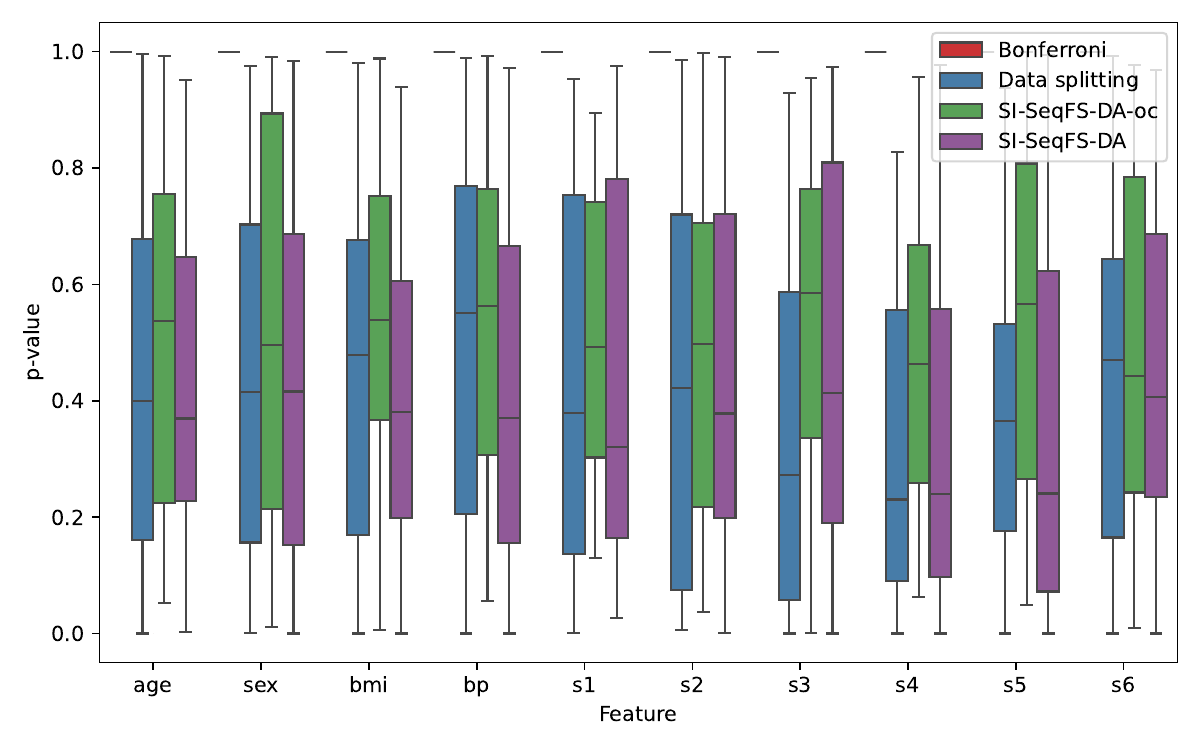}
    \vspace{-17pt}
    \caption{Diabetes dataset. The source domain consists of ``people over 50 years old'', while the target domain consists of ``people under 50 years old''.} 
    \label{fig:diabetes}
    \vspace{-8pt}
\end{figure}

\begin{figure}[!t]
    \centering
\includegraphics[width=\linewidth]{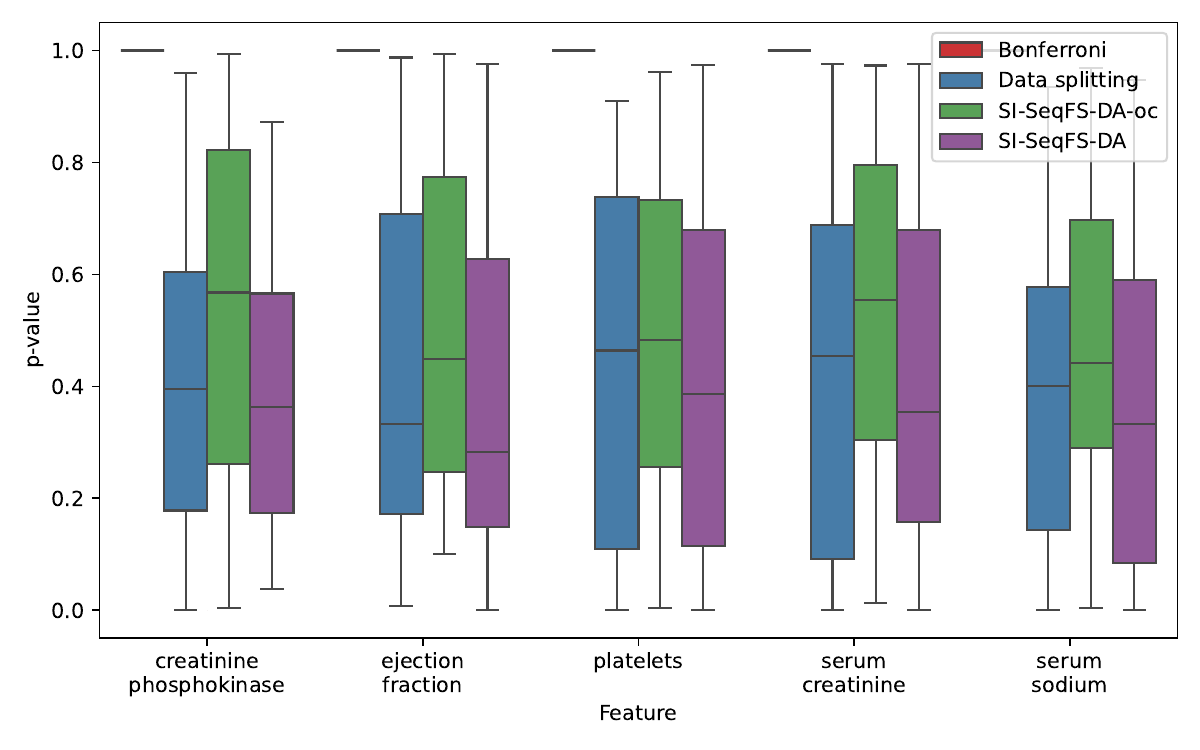}
\vspace{-17pt}
    \caption{Heart Failure dataset. The settings for the source and target domains are similar to those in Diabetes dataset.}
    \label{fig:heart_failure}
    \vspace{-8pt}
\end{figure}

\begin{figure}[!t]
    \centering
    \includegraphics[width=\linewidth]{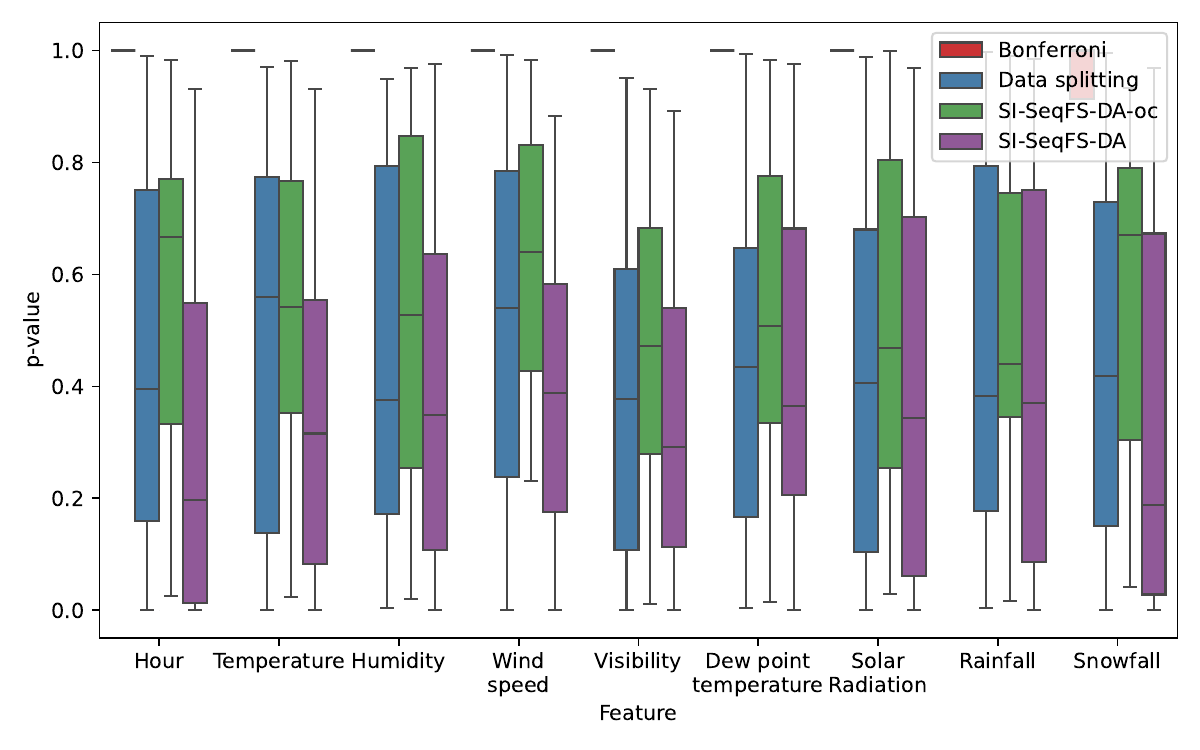}
    \vspace{-17pt}
    \caption{Seoul Bike dataset. The source domain is ``people who rent bikes on regular days'', while the target domain is ``people who rent bikes on holidays''.}
    \label{fig:Seoul_Bike}
    \vspace{-15pt}
\end{figure}

%
%
%
%
%
%
%
% \begin{figure}[ht]
%     \centering
%     % \captionsetup[subfigure]{labelformat=empty}
%     \begin{subfigure}[b]{0.9\linewidth}
%         \includegraphics[width=\linewidth]{Img/realdata/DB_backward_fixedK.pdf}
%         \caption{Diabetes.}
%         \label{fig:DB_BS_fixed}
%     \end{subfigure}
    
%     % \hspace{0.02\linewidth}
%     \vspace{4pt}
    
%     \begin{subfigure}[b]{0.9\linewidth}
%         \includegraphics[width=\linewidth]{Img/realdata/HF_backward_fixedK.pdf}
%         \caption{Heart Failure}
%         \label{fig:HF_BS_fixed}
%     \end{subfigure}
%     \vspace{4pt}
    
%     \begin{subfigure}[b]{\linewidth}
%         \includegraphics[width=0.9\linewidth]{Img/realdata/SB_backward_fixedK.pdf}
%         \caption{Seoul Bike}
%         \label{fig:SB_BS_fixed}
%     \end{subfigure}
%     \vspace{4pt}
%     \caption{K-Step Backward SeqFS after DA in real-world datasets}
%     \label{fig:fixed_BS_realdata}
%     \vspace{-8pt}
% \end{figure}

%% file: sec6.tex
\section{Conclusion}
\label{sec:conclusion}

We propose a novel method for computing a valid $p$-value to perform statistical tests on the features selected by SeqFS after DA. This method leverages the SI framework and uses a divide-and-conquer approach to efficiently compute the $p$-value. Additionally, we extend the proposed method to various settings, including backward SeqFS and optimal model selection criteria such as AIC, BIC, and adjusted $R^2$. We believe this study represents a significant step toward reliable ML in the context of DA.
At present, our method does not support other types of domain adaptations, such as MMD-based DA \citep{baktashmotlagh2013unsupervised}, metric learning-based DA \citep{saenko2010adapting}, or Deep Learning (DL)-based DA. This is due to the more complex nature of the selection events involved in these methods compared to OT-based DA. One potential solution could involve using a sampling-based approach to approximate the truncation region for $p$-value computation. Extending the proposed method to handle more complex DA methods would provide a valuable contribution.

%% file: appendix.tex
\section{Appendix}
\label{sec:appendix}

\subsection{Proof of Lemma \ref{lemma:valid_selective_p}} \label{app:proof_valid_p}

We have 
\begin{align*}
    \bm \eta_j^\top {\bm Y^s \choose \bm Y^t } \Big| \Big \{ 
	\cM_{\bm Y^s, \bm Y^t}
	=
	\cM_{\rm obs}, ~
	\cQ_{\bm Y^s, \bm Y^t}
	=
	\cQ_{\rm obs} \Big \} 
    \sim TN\left( 
            \bm \eta_j^\top {\bm \mu^s \choose \bm \mu^t
            }, \bm \eta_j^\top \Sigma \bm \eta_j, \cZ
            \right),
\end{align*}
which is a truncated normal distribution with mean ${\bm \eta}_j^\top {\bm \mu^s \choose \bm \mu^t}$, variance $\eta_j^\top \Sigma \eta_j$, in which $\Sigma = \begin{pmatrix}
	\Sigma^s & 0 \\ 
	0 & \Sigma^t
\end{pmatrix}$, and the truncation region $\cZ$ described in \S\ref{subsec:identification_cZ}. Therefore, under null hypothesis,

\begin{align*}
    p^{\rm selective}_j \ \Big| \ \Big \{ 
	\cM_{\bm Y^s, \bm Y^t}
	=
	\cM_{\rm obs}, ~
	\cQ_{\bm Y^s, \bm Y^t}
	=
	\cQ_{\rm obs} \Big \}   
    \sim \text{Unif}(0,1)
\end{align*}
Thus, $\mathbb{P}_{\rm H_{0, j}} \left( p^{\rm selective}_j \ \Big| \ 
	\cM_{\bm Y^s, \bm Y^t}
	=
	\cM_{\rm obs}, ~
	\cQ_{\bm Y^s, \bm Y^t}
	=
	\cQ_{\rm obs} 
\right) = \alpha, \forall \alpha \in [0,1]$.
Next, we have
\begin{align*}
    & \mathbb{P}_{\rm H_{0, j}} \left( p^{\rm selective}_j \ \Big| \ 
	\cM_{\bm Y^s, \bm Y^t}
	=
	\cM_{\rm obs}
\right)\\
    &= \int \mathbb{P}_{H_{0,j}} \left( p_j^{\rm selective} \leq \alpha \, \middle| \, 
    \begin{array}{l}
    \cM_{\bm Y^s, \bm Y^t} = \cM_{\rm obs}, \\ \cQ_{\bm Y^s, \bm Y^t} = \cQ_{\rm obs}     
    \end{array}
    \right)  \mathbb{P}_{H_{0,j}} \left( \cQ_{\bm Y^s, \bm Y^t} = \cQ_{\rm obs} \, \middle| \, 
    \cM_{\bm Y^s, \bm Y^t} = \cM_{\rm obs} 
    \right) \, d \cQ_{\rm obs} \\ 
    &= \int \alpha \, \mathbb{P}_{H_{0,j}} \left( \cQ_{\bm Y^s, \bm Y^t} = \cQ_{\rm obs} \, \middle| \, \cM_{\bm Y^s, \bm Y^t} = \cM_{\rm obs} \right) \, d\cQ_{\rm obs} \\ 
    &= \alpha \int \mathbb{P}_{H_{0,j}} \left( \cQ_{\bm Y^s, \bm Y^t} = \cQ_{\rm obs} \, \middle| \, \cM_{\bm Y^s, \bm Y^t} = \cM_{\rm obs} \right) \, d\cQ_{\rm obs} \\ 
    &= \alpha. 
\end{align*}
Finally, we obtain the result in Lemma \ref{lemma:valid_selective_p} as follows:
\begin{align*}
    \mathbb{P}_{\rm H_{0, j}} &\left( p^{\rm selective}_j \leq \alpha
    \right) \\
    &= \sum_{\cM_{\rm obs}} \mathbb{P}_{H_{0,j}} \left( p_j^{\rm selective} \leq \alpha \, \middle| \, \cM_{\bm Y^s, \bm Y^t} = \cM_{\rm obs} \right) \mathbb{P}_{H_{0,j}} \left( \cM_{\bm Y^s, \bm Y^t} = \cM_{\rm obs} \right)\\
    &= \sum_{\cM_{\rm obs}} \alpha \, \mathbb{P}_{H_{0,j}} \left( \cM_{\bm Y^s, \bm Y^t} = \cM_{\rm obs} \right) \\
    &= \alpha \sum_{\cM_{\rm obs}} \, \mathbb{P}_{H_{0,j}} \left( \cM_{\bm Y^s, \bm Y^t} = \cM_{\rm obs} \right) \\
    &= \alpha.
\end{align*}

\subsection{Proof of Lemma \ref{lemma:data_line}} \label{app:proof_line}
Base on the second condition in (\ref{eq:conditional_data_space}), we have
\begin{align*}
    \cQ_{\bm Y^s, \bm Y^t} &= \cQ_{\rm obs} \\
    \Leftrightarrow \left( I_{n_s + n_t} - \bm b \bm \eta_j^\top \right) {\bm Y^s \choose \bm Y^t } &= \cQ_{\rm obs} \\
    \Leftrightarrow {\bm Y^s \choose \bm Y^t } &= \cQ_{\rm obs} + \bm b \bm \eta_j^\top  {\bm Y^s \choose \bm Y^t }.
\end{align*}
By defining $\bm a= \cQ_{\rm obs}, z = \bm \eta_j^\top{\bm Y^s \choose \bm Y^t }$, and incorporating the second condition of (\ref{eq:conditional_data_space}), we obtain Lemma \ref{lemma:data_line}.

\subsection{Proof of Lemma \ref{lemma:cZ_u}} \label{app:proof_cZ_u}
Building on the results from \cite{le2024cad}, where the authors introduced a method for characterizing the event of OT using the concept of \textit{parametric linear programming}, we reformulate the OT problem between the source and target domains, as presented in (\ref{eq:ot_problem}), as follows:
\[
    \hat{\boldsymbol{t}} = \underset{\boldsymbol{t} \in \mathbb{R}^{n_s  n_t}} {\arg\min}  \ \boldsymbol{t}^\top \boldsymbol{c} \left(D^s, D^t\right) 
\]
\[
    \qquad \text{s.t.} \quad H \boldsymbol{t} = \boldsymbol{h}, \boldsymbol{t} \geq 0,
\]

where $\boldsymbol{t} =  \text{vec}(T), \boldsymbol{c} \left(D^s, D^t\right) = \text{vec} \left(C \left(D^s, D^t\right) \right) = {\bm c'} +
\left[ \Theta \begin{pmatrix}
    \bm Y^s \\
    \bm Y^t
\end{pmatrix} \right] \circ \left[ \Theta \begin{pmatrix}
    \bm Y^s \\
    \bm Y^t
\end{pmatrix} \right]$,
\[
    {\bm c'} = \text{vec} \left(\Big[
	\big \| X_i^s - X_j^t \big \|^2_2 
	\Big]_{ij}\right) \in \RR^{n_sn_t},
\]
\[
    \Theta = \text{hstack}\left(I_{n_s} \otimes \mathbf{1}_{n_t}, - \mathbf{1}_{n_s} \otimes I_{n_t}\right) \in \mathbb{R}^{n_sn_t \times (n_s + n_t)},
\]
the cost vector ${\bm c'}$ once computed from $X^s$ and $X^t$ remains fixed, vec($\cdot$) is an operator that transforms a matrix into a vector with concatenated rows, the operator $\circ$ is element-wise product, hstack($\cdot,\cdot$) is horizontal stack operation, the operator $\otimes$ is Kronecker product, $I_{n} \in \mathbb{R}^{n \times n}$ is the identity matrix, and $\mathbf{1}_{m} \in \mathbb{R}^{m}$ is a vector of ones. The matrix $H$ is defined  as $H = \begin{pmatrix}
    H_r & H_c
\end{pmatrix}^\top \in \mathbb{R}^{(n_s+n_t) \times n_sn_t}$ in which
\[
H_r = 
\begin{bmatrix}
1 & \dots & 1 & 0 & \dots & 0 & \dots & 0 & \dots & 0 \\
0 & \dots & 0 & 1 & \dots & 1 & \dots & 0 & \dots & 0 \\
\vdots & \dots & \vdots & \vdots & \dots & \vdots & \dots & \vdots & \dots & \vdots \\
0 & \dots & 0 & 0 & \dots & 0 & \dots & 1 & \dots & 1 \\
\end{bmatrix}
\in \mathbb{R}^{n_s \times n_s n_t}
\]
that performs the sum over the rows of $T$ and
\[
    H_c = \begin{bmatrix}
I_{n_t} & I_{n_t} & \dots & I_{n_t}
\end{bmatrix} \in \mathbb{R}^{n_t \times n_s n_t}
\]
that performs the sum over the columns of $T$, and $h = \left(\frac{\mathbf{1}_{n_s}}{n_s}, \frac{\mathbf{1}_{n_t}}{n_t}\right)^\top \in \mathbb{R}^{n_s + n_t}.$
Next, we analyze the OT problem formulated with the parameterized data $\boldsymbol{a} + \boldsymbol{b}z$:
\begin{align*}
    &\min_{{\bm t} \in \mathbb{R}^{n_sn_t}} {\bm t}^T \left[ ({\bm c'} + \Theta(\boldsymbol{a} + \boldsymbol{b}z) ) \circ ({\bm c'} + \Theta(\boldsymbol{a} + \boldsymbol{b}z) ) \right] \quad \text{s.t.} \quad H{\bm t} = h, \quad {\bm t} \geq 0, \\
    \Leftrightarrow &\min_{{\bm t} \in \mathbb{R}^{n_sn_t}} (\tilde{\bm p} + \tilde{\bm q}z + \tilde{\bm f}z^2)^\top {\bm t} \quad \text{s.t.} \quad H{\bm t} = h, \quad {\bm t} \geq 0,
\end{align*}
where
\begin{align*}
    \tilde{\boldsymbol{p}} &= ({\bm c'} + \Theta \boldsymbol{a}) \circ ({\bm c'} + \Theta \boldsymbol{a}), \quad 
    \tilde{\boldsymbol{q}} = (\Theta \boldsymbol{a}) \circ (\Theta \boldsymbol{b}) + (\Theta \boldsymbol{b}) \circ (\Theta \boldsymbol{a}), \quad \\ 
    \text{and} \quad \tilde{\boldsymbol{f}} &= (\Theta \boldsymbol{b}) \circ (\Theta \boldsymbol{b}).
\end{align*}
By fixing $\mathcal{B}_u$ as the optimal basic index set of the linear program, the \textit{relative cost vector} w.r.t to the set of non-basis variables $\mathcal{B}_u^c$ is defined as
\[
    \boldsymbol{r}_{\mathcal{B}_u^c} = \boldsymbol{p} + \boldsymbol{q} z + \boldsymbol{f}z^2,
\]
where
\begin{align*}
    \label{define p q f}
    \boldsymbol{p} &= (\tilde{\boldsymbol{p}}_{\mathcal{B}_u^c}^\top - \tilde{\boldsymbol{p}}_{\mathcal{B}_u}^\top H_{:,\mathcal{B}_u}^{-1} H_{:,\mathcal{B}_u^c})^\top, 
    \boldsymbol{q} = (\tilde{\boldsymbol{q}}_{\mathcal{B}_u^c}^\top - \tilde{\boldsymbol{q}}_{\mathcal{B}_u}^\top H_{:,\mathcal{B}_u}^{-1} H_{:,\mathcal{B}_u^c}) ^\top, \\
    \boldsymbol{f} &= (\tilde{\boldsymbol{f}}_{\mathcal{B}_u^c}^\top - \tilde{\boldsymbol{f}}_{\mathcal{B}_u}^\top H_{:,\mathcal{B}_u}^{-1} H_{:,\mathcal{B}_u^c})^\top,    
\end{align*}
$H_{:,\mathcal{B}_u}^{-1}$ is a sub-matrix of $H$ made up of all rows and columns in the set $\mathcal{B}_u$. The requirement for $\mathcal{B}_u$ to be the optimal basis index set is $\boldsymbol{f}_{\mathcal{B}_u^c} \geq \boldsymbol{0}$ (i.e., the cost in minimization problem will never decrease when the non-basic variables become positive and enter the basis). We note that the optimal basis index set $\mathcal{B}_u$ corresponds to the transportation $\mathcal{T}_u$. Therefore, the set $\mathcal{Z}_u$ is defined as 
\[
\begin{aligned}
    \mathcal{Z}_u &= \{ z \in \mathbb{R} \mid \mathcal{T}_{{\bm a} + {\bm b} z} = \mathcal{T}_u \}, \\
    &= \{ z \in \mathbb{R} \mid \mathcal{B}_{{\bm a} + {\bm b} z} = \mathcal{B}_u \}, \\
    &= \{ z \in \mathbb{R} \mid \boldsymbol{r}_{\mathcal{B}_u^c} = \boldsymbol{p} + \boldsymbol{q} z + \boldsymbol{f}z^2 \geq 0 \}.
\end{aligned}
\]
Thus, we obtain the result in Lemma \ref{lemma:cZ_u}. We note that this result is also discussed in \cite{loi2024statistical}.

\subsection{Proof of Lemma \ref{lemma:cZFS_v}}
\label{app:proof_Zv_fw}

At each step $k \in [K]$, the selected feature $j_k$ adding to the model satisfies:
\[
\begin{aligned}
    {\rm RSS}(\tilde{\bm Y}_u(z), \tilde{X}_{u_{\mathcal{M}_{k-1} \cup \{j_k\}}}) &\leq 
    {\rm RSS}(\tilde{\bm Y}_u(z), \tilde{X}_{u_{\mathcal{M}_{k-1} \cup \{j\}}})
    \quad  \forall j \in [p] \setminus (\cM_{k-1} \cup \{ j_k \})
\\
\Leftrightarrow
    \left\Vert P_{\tilde{X}_{u_{\cM_{k-1} \cup \{j_k\}}}}^\perp \tilde{\bm Y}_u(z) \right\Vert_2^2
    &\leq
    \left\Vert P_{\tilde{X}_{u_{\cM_{k-1} \cup \{j\}}}}^\perp \tilde{\bm Y}_u(z) \right\Vert_2^2,
\end{aligned}
\]
where $P_{X_\cM}^\perp = \left( I_{n_s+ n_t} - X_\cM(X^\top_\cM X_\cM)^{-1} X^\top_\cM \right)$.
Expressing the squared norms in their corresponding matrix forms, we denote 
\begin{align*}    
    L_{j_k} = P_{\tilde{X}_{u_{\cM_{k-1} \cup \{j_k\}}}}^{\perp^\top} P_{\tilde{X}_{u_{\cM_{k-1} \cup \{j_k\}}}}^\perp \, \text{ and } \,
    L_{j} = P_{\tilde{X}_{u_{\cM_{k-1} \cup \{j\}}}}^{\perp^\top} P_{\tilde{X}_{u_{\cM_{k-1} \cup \{j\}}}}^\perp.
\end{align*}
Then, the inequality is rewritten as:
\[
\begin{aligned}
% \Leftrightarrow
%     \tilde{\bm Y}_u(z)^\top P_{u,{\cM_{k-1} \cup \{j_k\}}}^{\perp^\top} P_{u,{\cM_{k-1} \cup \{j_k\}}}^\perp \tilde{\bm Y}_u(z)
%     &\leq
%     \tilde{\bm Y}_u(z)^\top P_{u,{\cM_{k-1} \cup \{j\}}}^{\perp^\top} P_{u,{\cM_{k-1} \cup \{j\}}}^\perp \tilde{\bm Y}_u(z)
%     \\
    % \Leftrightarrow
    \tilde{\bm Y}_u(z)^\top L_{j_k} \tilde{\bm Y}_u(z)
    &\leq
    \tilde{\bm Y}_u(z)^\top L_{j}\tilde{\bm Y}_u(z)
    \\
    \Leftrightarrow
    \bm Y (z)^\top \Omega_u^\top L_{j_k} \Omega_u \bm Y (z)
    &\leq
    \bm Y (z)^\top \Omega_u^\top L_{j} \Omega_u \bm Y (z)
    \\
    \Leftrightarrow
    (\bm a+\bm bz)^\top\Omega_u^\top L_{j_k} \Omega_u (\bm a+\bm bz)
    &\leq
    (\bm a+\bm bz)^\top\Omega_u^\top L_{j} \Omega_u (\bm a+\bm bz).
\end{aligned}
\]
By expanding and simplifying the quadratic forms on both sides, the inequality is reduced to a quadratic expression in $z$:
\[
\begin{aligned}
    w_{j_k,j} + r_{j_k,j}z + o_{j_k,j} z^2 &\leq 0 
\end{aligned}
\]
where
\[\begin{aligned}
    w_{j_k,j} &= \bm a^\top\Omega_u^\top L_{j_k} \Omega_u \bm a - \bm a^\top\Omega_u^\top L_{j} \Omega_u \bm a, \\
    r_{j_k,j} &= \bm  a^\top\Omega_u^\top L_{j_k} \Omega_u \bm b + \bm b^\top\Omega_u^\top L_{j_k} \Omega_u \bm a - \bm a^\top\Omega_u^\top L_{j} \Omega_u \bm b - \bm b^\top\Omega_u^\top L_{j} \Omega_u \bm a,
    \\
    o_{j_k,j} &= \bm b^\top\Omega_u^\top L_{j_k} \Omega_u \bm b -  \bm b^\top\Omega_u^\top L_{j} \Omega_u \bm b,
\end{aligned}
 \]
With all step $k \in [K]$, we have a system of quadratic inequalities in Lemma \ref{lemma:cZFS_v}:
\[
\bm w  + \bm r z + \bm o z^2 \leq 0
\]
where 
\begin{align*}
    \bm w &= {\rm vector} \left(\left\{ w_{j_k,j} \right\}_{k \in [K], j \in [p] \setminus (\cM_{k-1} \cup \{ j_k \})} \right),
    \\
    \bm r  &= {\rm vector} \left(\left\{ r_{j_k,j} \right\}_{k \in [K], j \in [p] \setminus (\cM_{k-1} \cup \{ j_k \})} \right),
    \\
    \bm o &= {\rm vector} \left(\left\{ o_{j_k,j} \right\}_{k \in [K], j \in [p] \setminus (\cM_{k-1} \cup \{ j_k \})} \right),
\end{align*}
${\rm vector}(\cdot)$ is the operation that converts a set to a vector.

\subsection{Proof of Lemma \ref{lemma:cZBS_v}}
\label{app:proof_Zv_bw}

 The identification of $\cZ^{bw}_v$ is constructed almost like that in Lemma \ref{lemma:cZFS_v}. At each step $k \in [p, \dots, K+1]$, the removed feature $j_k$ out of the model satisfies:
\[
\begin{aligned}
    {\rm RSS}(\tilde{\bm Y}_u(z), \tilde{X}_{u_{\mathcal{M}_{k-1} \setminus \{j_k\}}}) &\leq 
    {\rm RSS}(\tilde{\bm Y}_u(z), \tilde{X}_{u_{\mathcal{M}_{k-1} \setminus \{j\}}})
    \quad  \forall j \in \cM_{k} \setminus \{ j_k \}
\\
\Leftrightarrow
    \left\Vert P_{\tilde{X}_{u_{\cM_{k} \setminus \{j_k\}}}}^\perp \tilde{\bm Y}_u(z) \right\Vert_2^2
    &\leq
    \left\Vert P_{\tilde{X}_{u_{\cM_{k} \setminus \{j\}}}}^\perp \tilde{\bm Y}_u(z) \right\Vert_2^2.
\end{aligned}
\]
Expressing the squared norms in their corresponding matrix forms, we denote 
\begin{align*}    
    L^{\text{bw}}_{j_k} = P_{\tilde{X}_{u_{\cM_{k} \setminus \{j_k\}}}}^{\perp^\top} P_{\tilde{X}_{u_{\cM_{k} \setminus \{j_k\}}}}^\perp \, \text{ and } \,
    L^{\text{bw}}_{j} = P_{\tilde{X}_{u_{\cM_{k} \setminus \{j\}}}}^{\perp^\top} P_{\tilde{X}_{u_{\cM_{k} \setminus \{j\}}}}^\perp.
\end{align*}
The inequality is rewritten as:
\[
\begin{aligned}
    % \Leftrightarrow \tilde{\bm Y}_u(z)^\top P_{u,{\cM_{k} \setminus \{j_k\}}}^{\perp^\top} P_{u,{\cM_{k} \setminus \{j_k\}}}^\perp \tilde{\bm Y}_u(z)
    % &\leq
    % \tilde{\bm Y}_u(z)^\top P_{u,{\cM_{k} \setminus \{j\}}}^{\perp^\top} P_{u,{\cM_{k} \setminus \{j\}}}^\perp \tilde{\bm Y}_u(z)
    % \\
    % \Leftrightarrow
    \tilde{\bm Y}_u(z)^\top L^{\text{bw}}_{j_k} \tilde{\bm Y}_u(z)
    &\leq
    \tilde{\bm Y}_u(z)^\top L^{\text{bw}}_{j}\tilde{\bm Y}_u(z)
    \\
    \Leftrightarrow
    \bm Y (z)^\top \Omega_u^\top L^{\text{bw}}_{j_k} \Omega_u \bm Y (z)
    &\leq
    \bm Y (z)^\top \Omega_u^\top L^{\text{bw}}_{j} \Omega_u \bm Y (z)
    \\
    \Leftrightarrow
    (\bm a+ \bm bz)^\top\Omega_u^\top L^{\text{bw}}_{j_k} \Omega_u (\bm a+\bm bz)
    &\leq
    (\bm a+\bm bz)^\top\Omega_u^\top L^{\text{bw}}_{j} \Omega_u (\bm a+\bm bz)
\end{aligned}
\]
By expanding and simplifying the quadratic forms on both sides, the inequality is reduced to a quadratic expression in $z$:
\[
\begin{aligned}
    w^\text{bw}_{j_k,j} + r^\text{bw}_{j_k,j}z + o^\text{bw}_{j_k,j} z^2 &\leq 0 
\end{aligned}
\]
where
\[\begin{aligned}
    w^\text{bw}_{j_k,j} &= \bm a^\top\Omega_u^\top L^{\text{bw}}_{j_k} \Omega_u \bm a - \bm a^\top\Omega_u^\top L^{\text{bw}}_{j} \Omega_u \bm a, \\
    r^\text{bw}_{j_k,j} &= \bm  a^\top\Omega_u^\top L^{\text{bw}}_{j_k} \Omega_u \bm b + \bm b^\top\Omega_u^\top L^{\text{bw}}_{j_k} \Omega_u \bm a - \bm a^\top\Omega_u^\top L^{\text{bw}}_{j} \Omega_u \bm b - \bm b^\top\Omega_u^\top L^{\text{bw}}_{j} \Omega_u \bm a,
    \\
    o^\text{bw}_{j_k,j} &= \bm b^\top\Omega_u^\top L^{\text{bw}}_{j_k} \Omega_u \bm b -  \bm b^\top\Omega_u^\top L^{\text{bw}}_{j} \Omega_u \bm b,
\end{aligned}
 \]
With all step $k \in [p, \dots, K+1]$, we have a system of quadratic inequalities in Lemma \ref{lemma:cZBS_v}:
\[
\bm w^\text{bw}  + \bm r^\text{bw} z +\bm  o^\text{bw} z^2 \leq 0
\]
where 
\begin{align*}
    \bm w^\text{bw} &= {\rm vector} \left(\left\{ w^\text{bw}_{j_k,j} \right\}_{k \in [p,\dots, K+1], j \in \cM_{k} \setminus \{ j_k \}} \right),
    \\
    \bm r^\text{bw}  &= {\rm vector} \left(\left\{ r^\text{bw}_{j_k,j} \right\}_{k \in [p,\dots, K+1], j \in \cM_{k} \setminus \{ j_k \}} \right),
    \\
    \bm o^\text{bw} &= {\rm vector} \left(\left\{ o^\text{bw}_{j_k,j} \right\}_{k \in [p,\dots, K+1], j \in \cM_{k} \setminus \{ j_k \}} \right).
\end{align*}

\subsection{Additional Experiments} \label{app:additional_experiments}
\textbf{Experiments for extended methods.} 
We perform additional experiments on three real-world datasets, as described in Section \ref{subsec:realdata}. The plots display the results using extended methods of SI-SeqFS-DA, including Backward SeqFS after DA, Forward SeqFS after DA with AIC, BIC, and adjusted $R^2$, as well as Backward SeqFS after DA with AIC, BIC, and adjusted $R^2$.
The results are shown in Figs. \ref{fig:fixed_BS_realdata}, \ref{fig:AIC_FS_realdata}, \ref{fig:BIC_FS_realdata}, \ref{fig:AdjR2_FS_realdata}, \ref{fig:AIC_BS_realdata}, \ref{fig:BIC_BS_realdata}, and \ref{fig:AdjR2_BS_realdata}.
%%%%%%%%%%%%%%%%%%111111111
\begin{figure}[!t] 
    \centering

    % Subfigure 1
    \begin{subfigure}{\linewidth}
        \centering
        \includegraphics[width=0.9\linewidth]{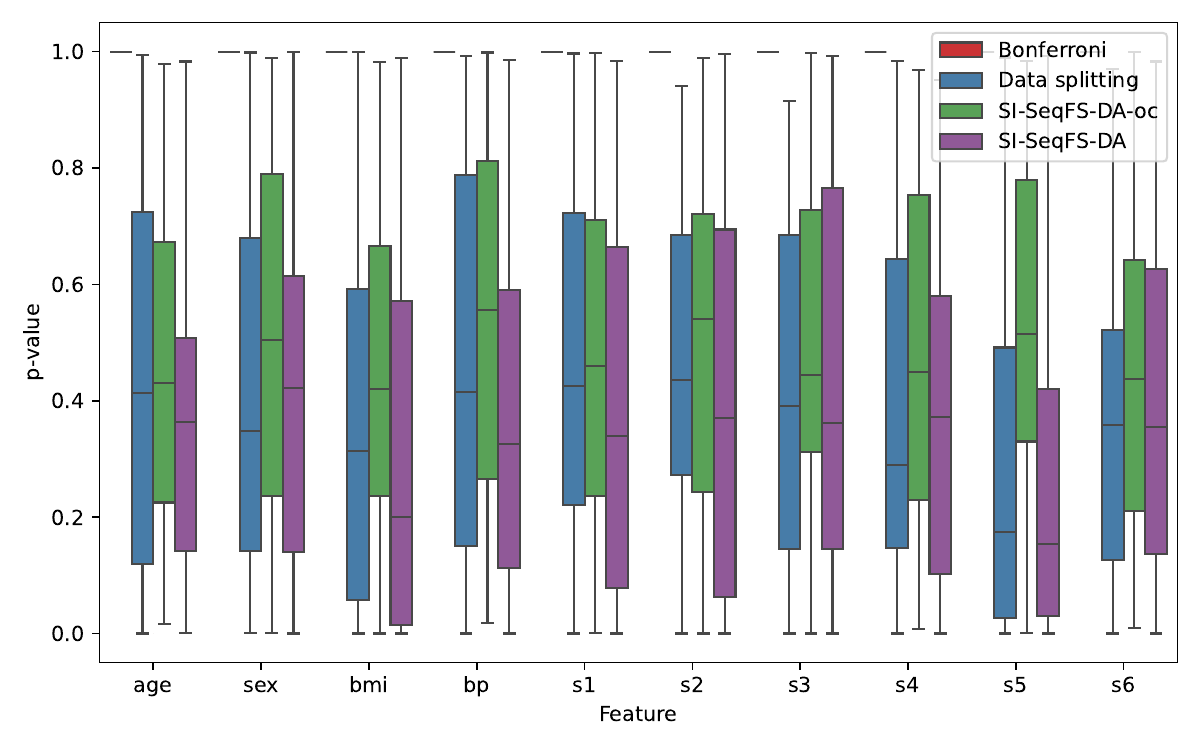}
        \caption{Diabetes.}
        % \label{fig:}
    \end{subfigure}
    \vspace{1pt}

    % Subfigure 2
    \begin{subfigure}{\linewidth}
        \centering
        \includegraphics[width=0.9\linewidth]{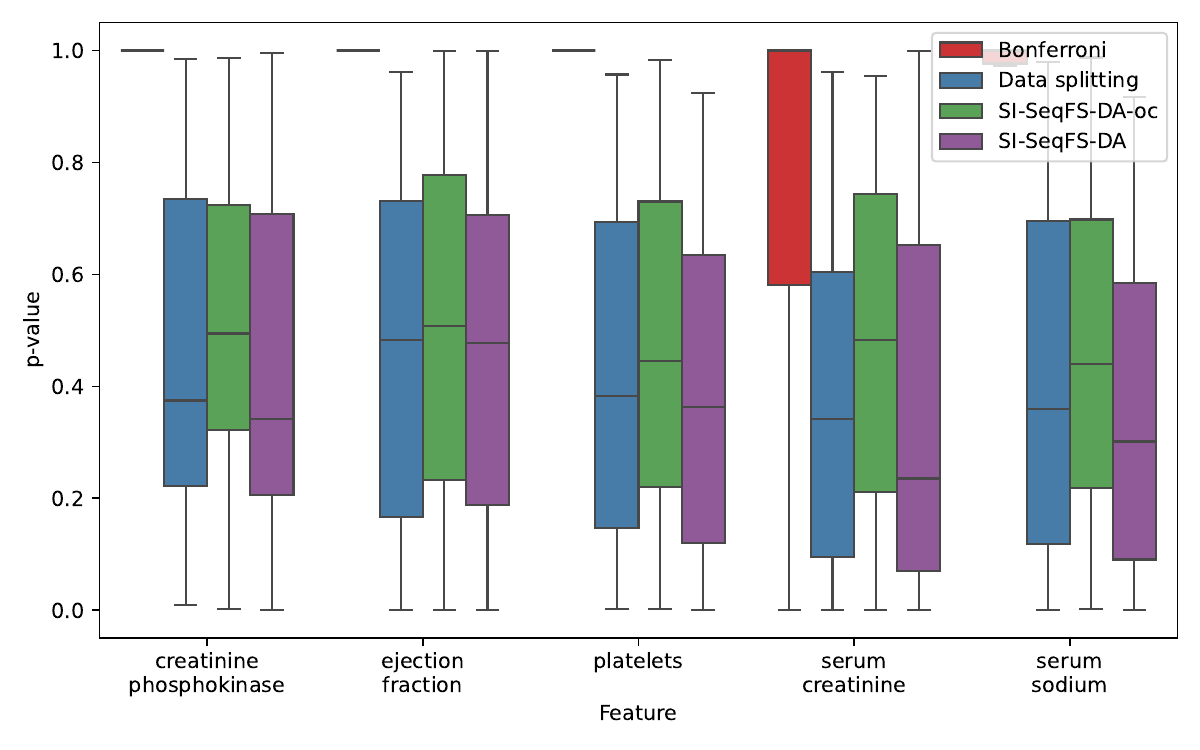}
        \caption{Heart Failure.}
        % \label{fig:}
    \end{subfigure}
    \vspace{1pt}

    % Subfigure 3
    \begin{subfigure}{\linewidth}
        \centering
        \includegraphics[width=0.9\linewidth]{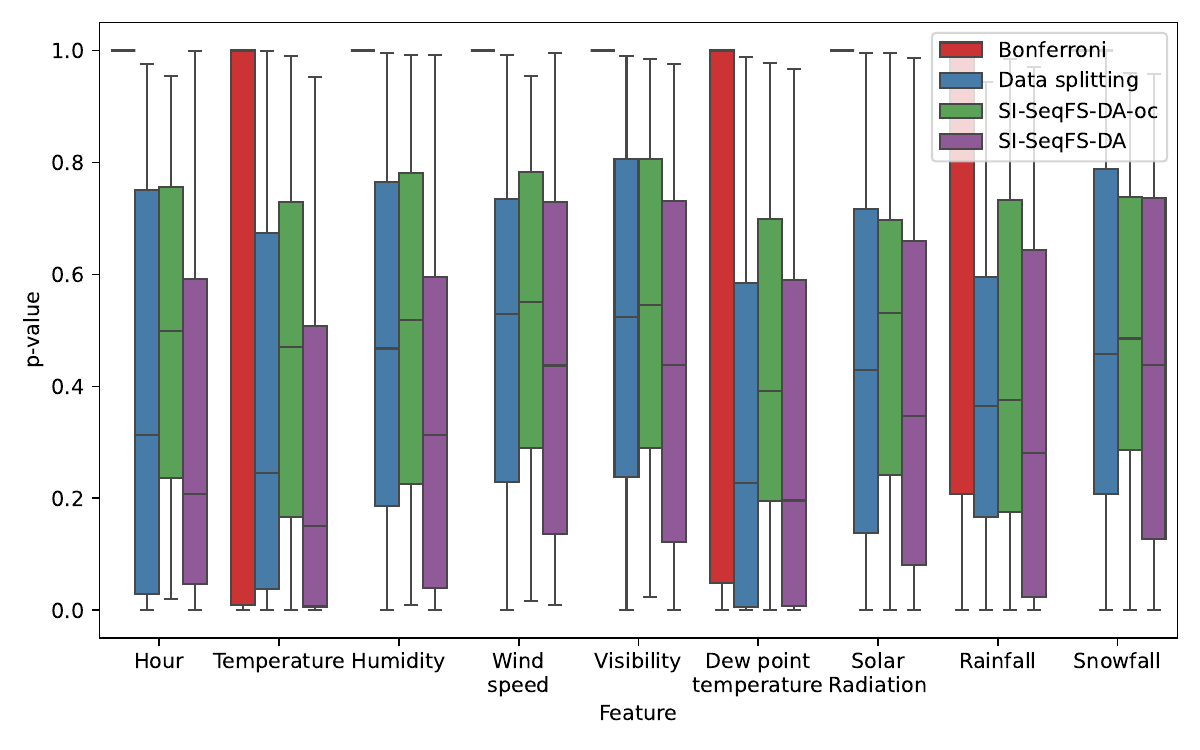}
        \caption{Seoul Bike.}
        % \label{fig:}
    \end{subfigure}
    \vspace{1pt}

    \caption{Backward SeqFS after DA on real-world datasets.}
    \label{fig:fixed_BS_realdata}
\end{figure}

%%%%%%%%%%%%%%%%%%%22222222222222

\begin{figure}[!t] % Allow placement flexibility with [htbp]
    \centering

    % Subfigure 1
    \begin{subfigure}{\linewidth}
        \centering
        \includegraphics[width=0.9\linewidth]{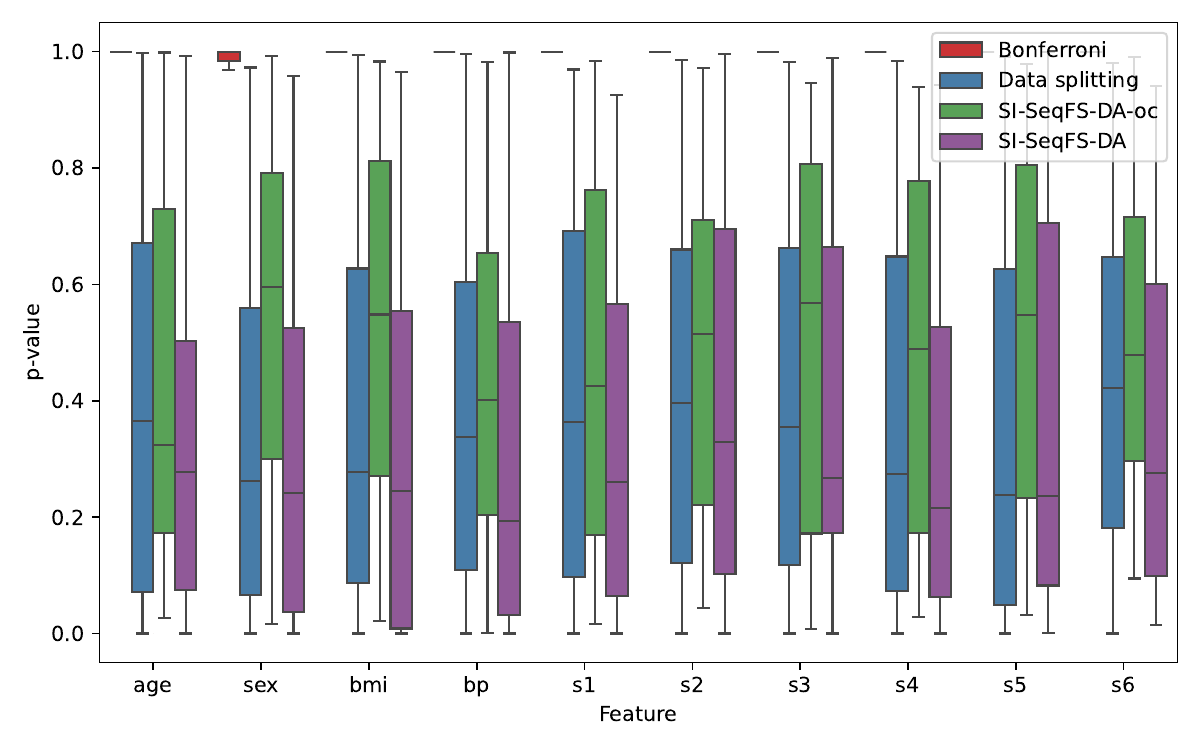}
        \caption{Diabetes.}
        % \label{fig:}
    \end{subfigure}
    \vspace{1pt}

    % Subfigure 2
    \begin{subfigure}{\linewidth}
        \centering
        \includegraphics[width=0.9\linewidth]{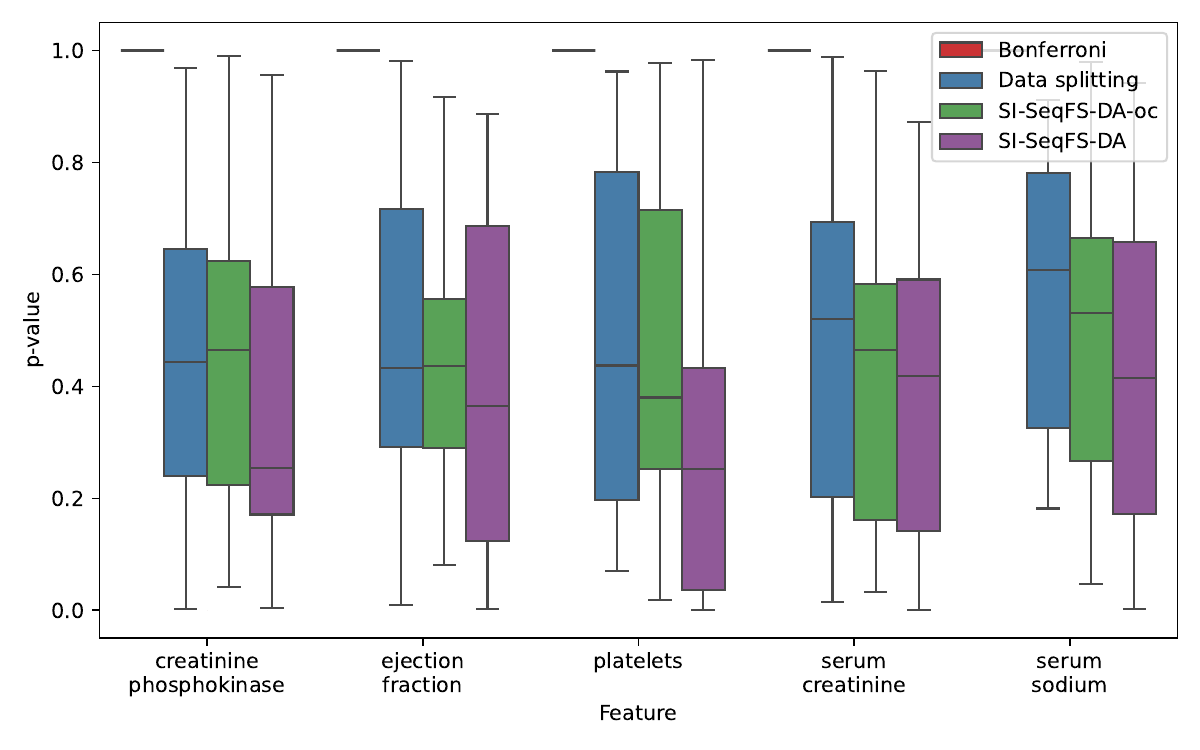}
        \caption{Heart Failure.}
        % \label{fig:}
    \end{subfigure}
    \vspace{1pt}

    % Subfigure 3
    \begin{subfigure}{\linewidth}
        \centering
        \includegraphics[width=0.9\linewidth]{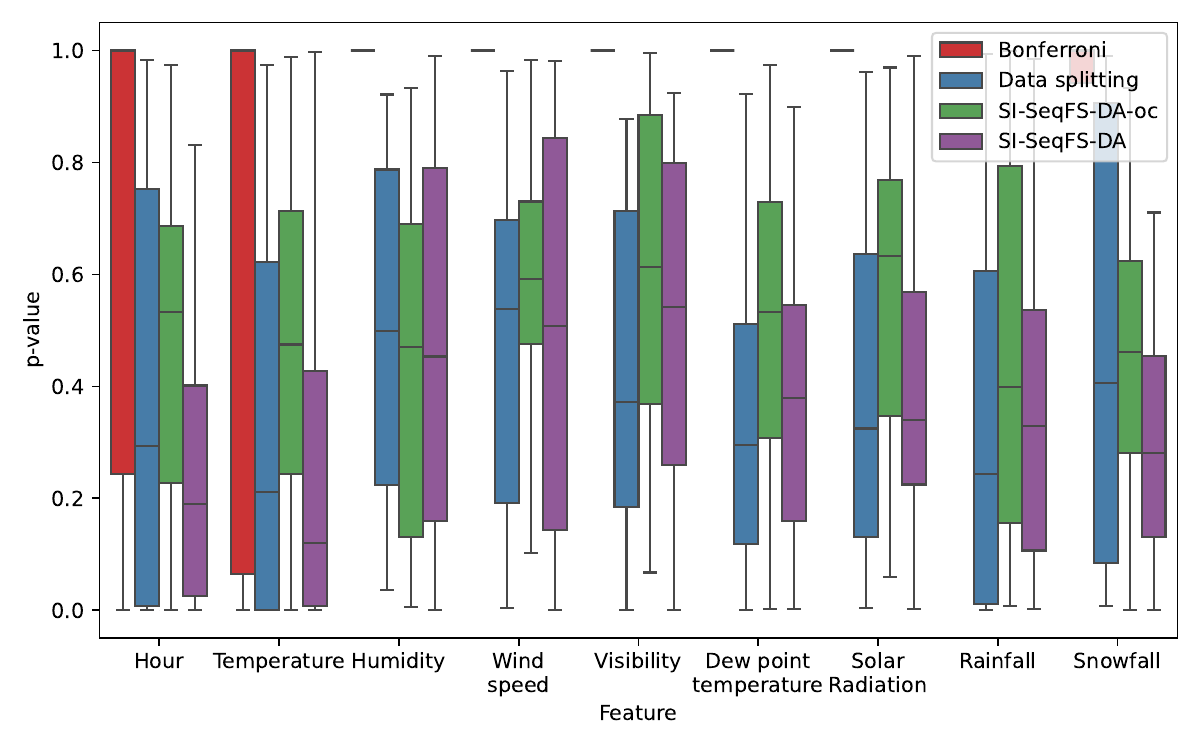}
        \caption{Seoul Bike.}
        % \label{fig:}
    \end{subfigure}
    \vspace{1pt}

    \caption{Forward SeqFS after DA with AIC on real-world datasets.}
    \label{fig:AIC_FS_realdata}
\end{figure}

%%%%%%%%%%%%%%%%%%%%%%%%3
\begin{figure}[!t] % Allow placement flexibility with [htbp]
    \centering

    % Subfigure 1
    \begin{subfigure}{\linewidth}
        \centering
        \includegraphics[width=0.9\linewidth]{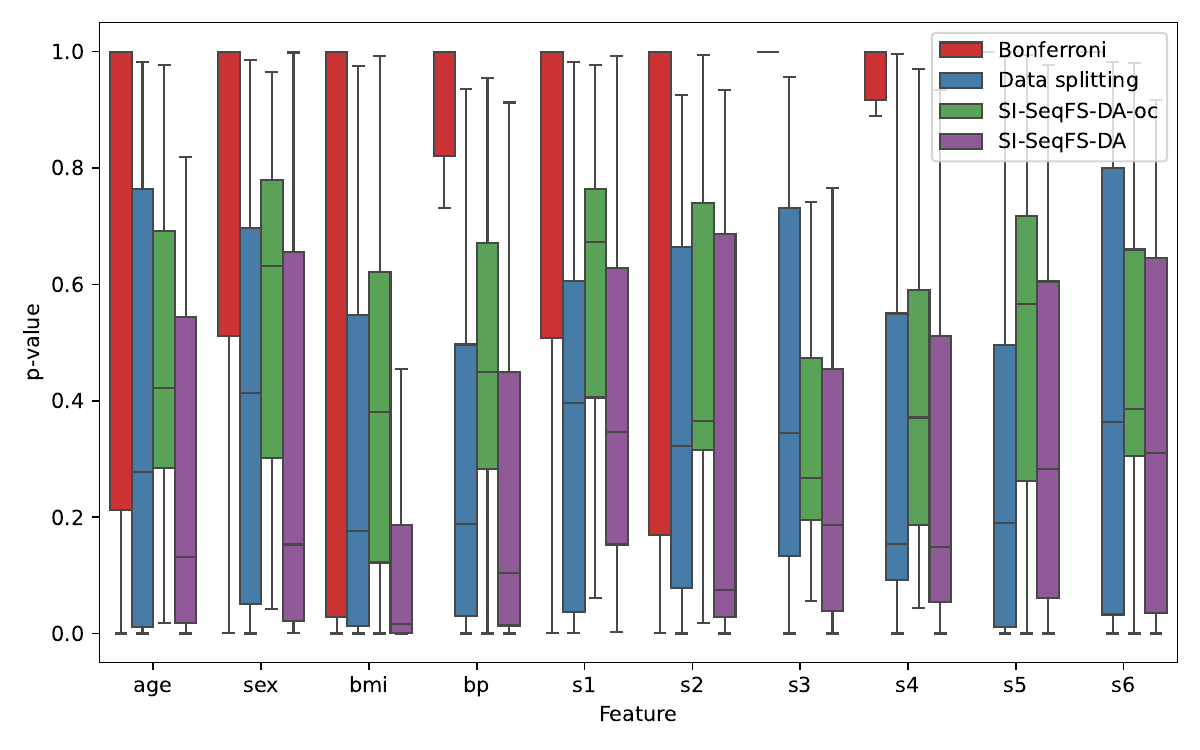}
        \caption{Diabetes.}
        % \label{fig:}
    \end{subfigure}
    \vspace{1pt}

    % Subfigure 2
    \begin{subfigure}{\linewidth}
        \centering
        \includegraphics[width=0.9\linewidth]{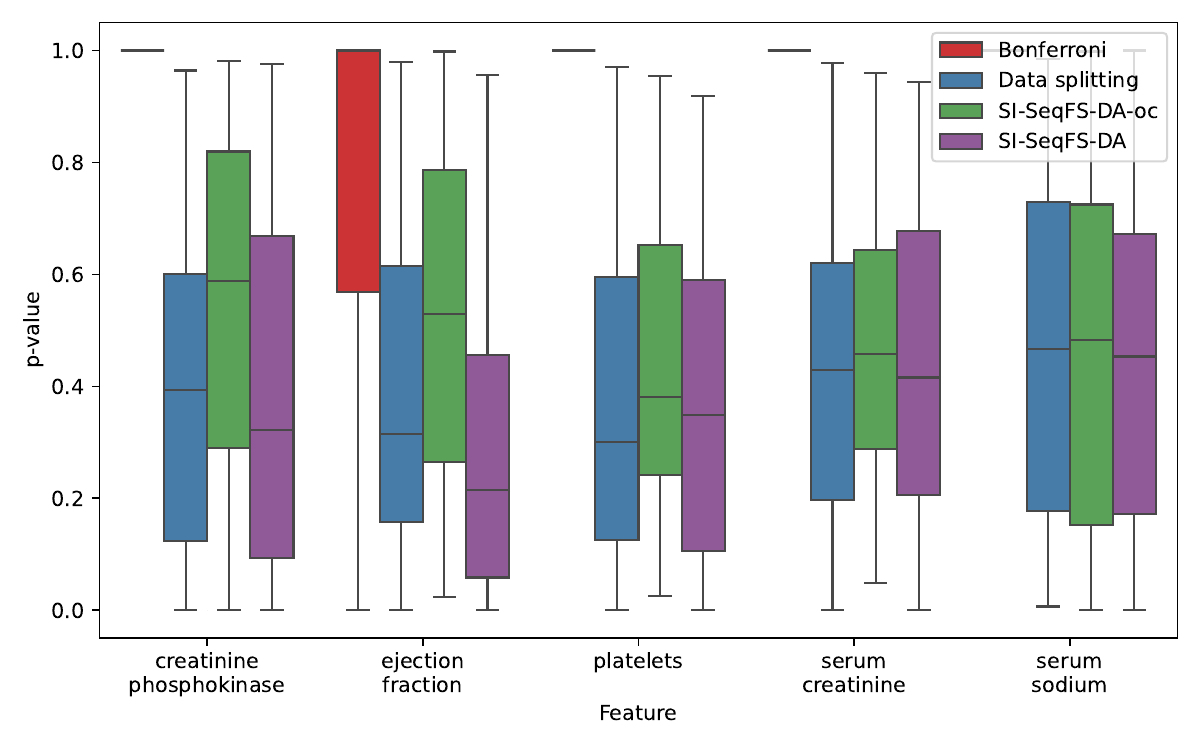}
        \caption{Heart Failure.}
        % \label{fig:}
    \end{subfigure}
    \vspace{1pt}

    % Subfigure 3
    \begin{subfigure}{\linewidth}
        \centering
        \includegraphics[width=0.9\linewidth]{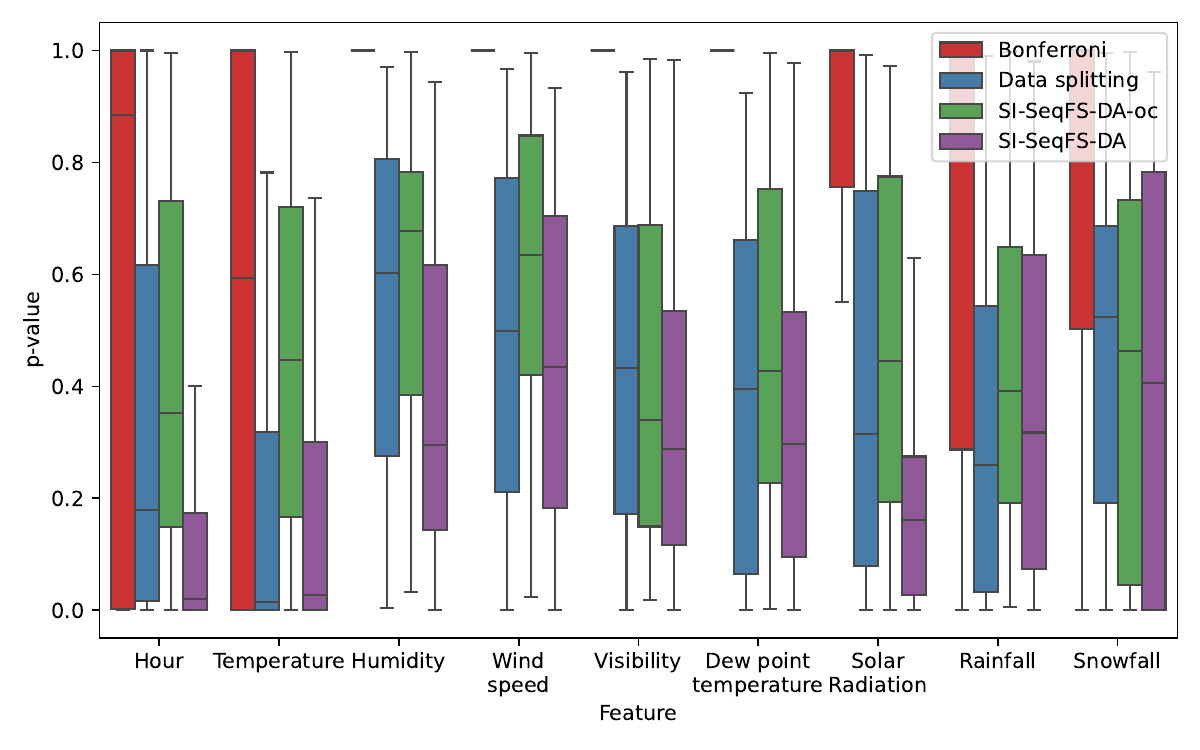}
        \caption{Seoul Bike.}
        % \label{fig:}
    \end{subfigure}
    \vspace{1pt}

    \caption{Forward SeqFS after DA with BIC on real-world datasets.}
    \label{fig:BIC_FS_realdata}
\end{figure}
%%%%%%%%%%%%%%%%%%%%%%%%4
\begin{figure}[htbp] % Allow placement flexibility with [htbp]
    \centering

    % Subfigure 1
    \begin{subfigure}{\linewidth}
        \centering
        \includegraphics[width=0.9\linewidth]{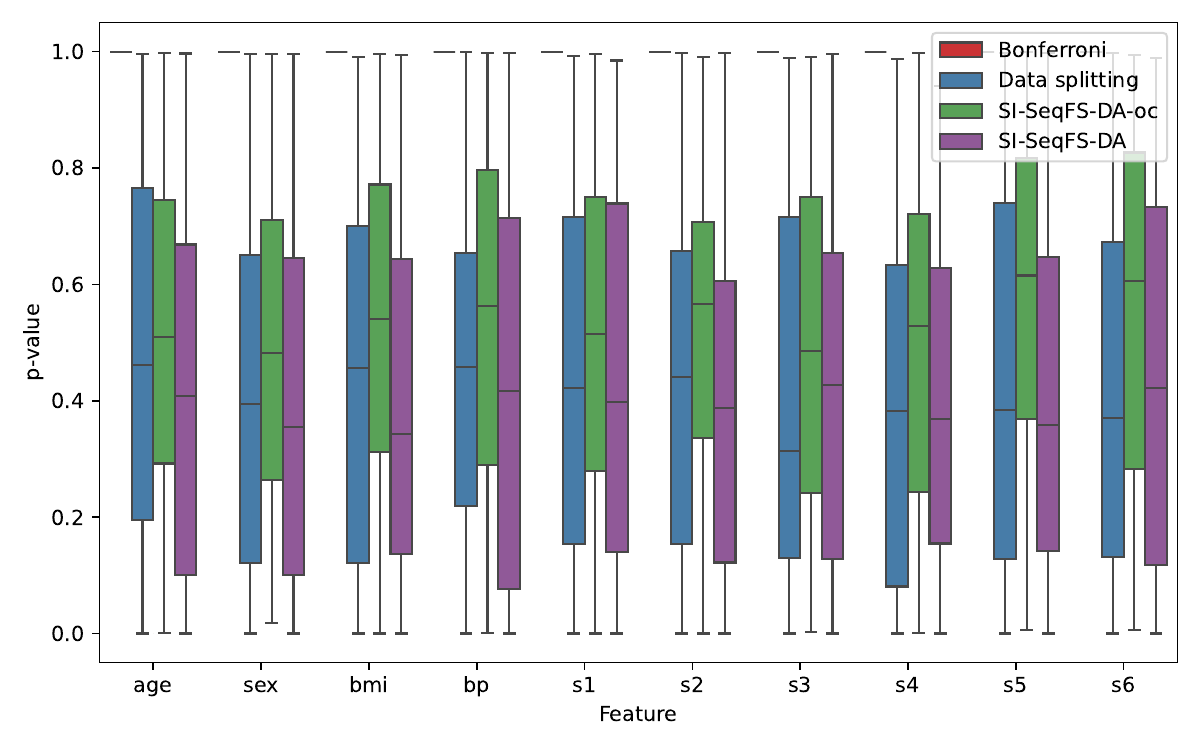}
        \caption{Diabetes.}
        % \label{fig:}
    \end{subfigure}
    \vspace{1pt}

    % Subfigure 2
    \begin{subfigure}{\linewidth}
        \centering
        \includegraphics[width=0.9\linewidth]{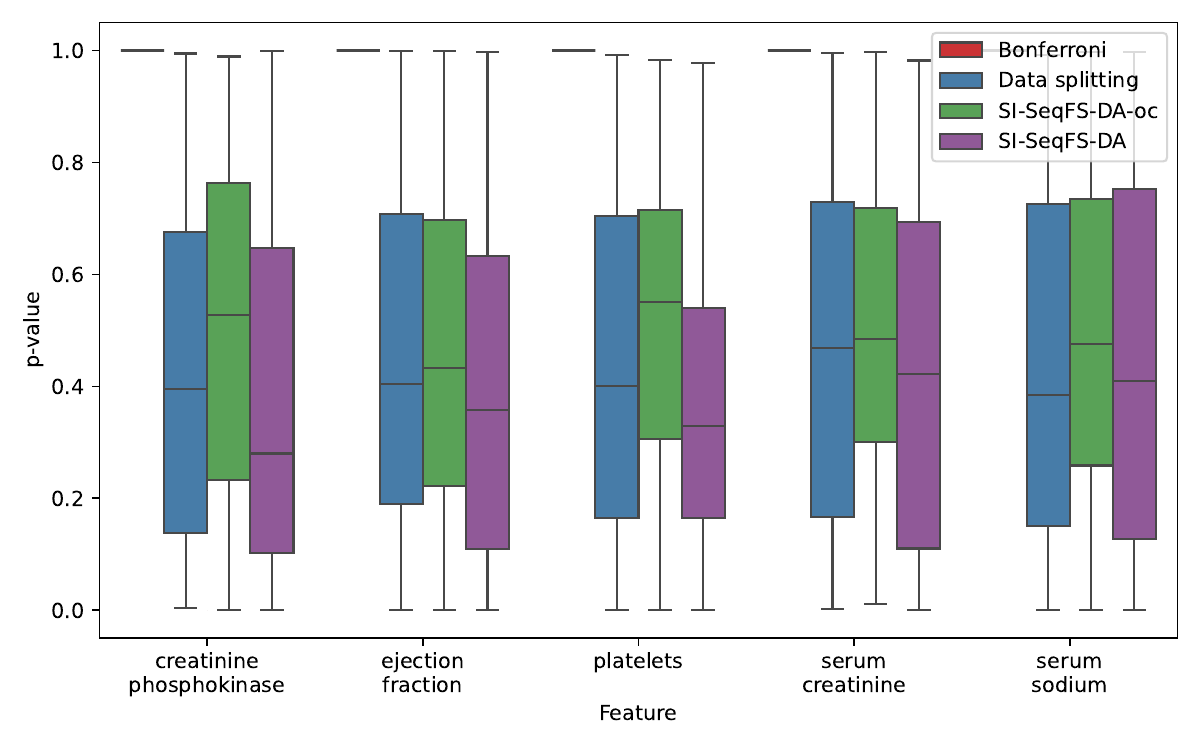}
        \caption{Heart Failure.}
        % \label{fig:}
    \end{subfigure}
    \vspace{1pt}

    % Subfigure 3
    \begin{subfigure}{\linewidth}
        \centering
        \includegraphics[width=0.9\linewidth]{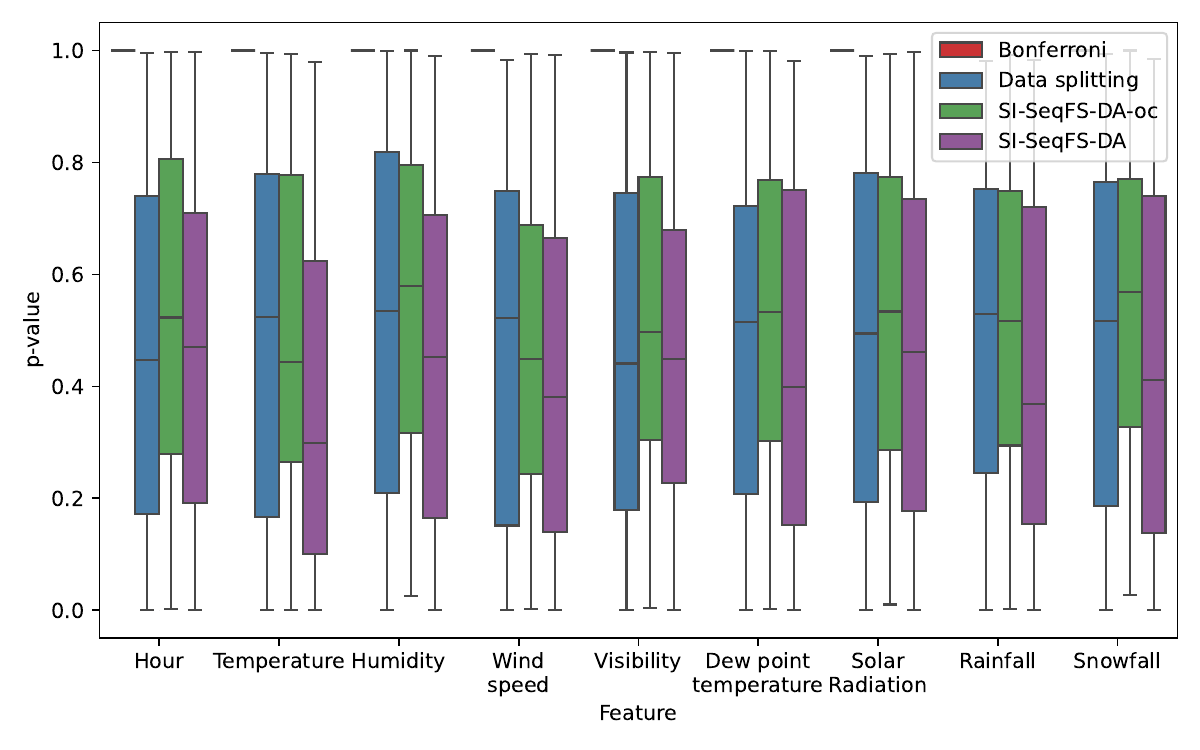}
        \caption{Seoul Bike.}
        % \label{fig:}
    \end{subfigure}
    \vspace{1pt}

    \caption{Forward SeqFS after DA with adjusted $R^2$ on real-world datasets.}
    \label{fig:AdjR2_FS_realdata}
\end{figure}
%%%%%%%%%%%%%%%%%%%%%%%%5

%%%%%%%%%%%%%%%%%%%%%%%%6

%%%%%%%%%%%%%%%%%%%%%%%%7

%%%%%%%%%%%%%%%%%%%22222222222222

\begin{figure}[!t] % Allow placement flexibility with [htbp]
    \centering

    % Subfigure 1
    \begin{subfigure}{\linewidth}
        \centering
        \includegraphics[width=0.9\linewidth]{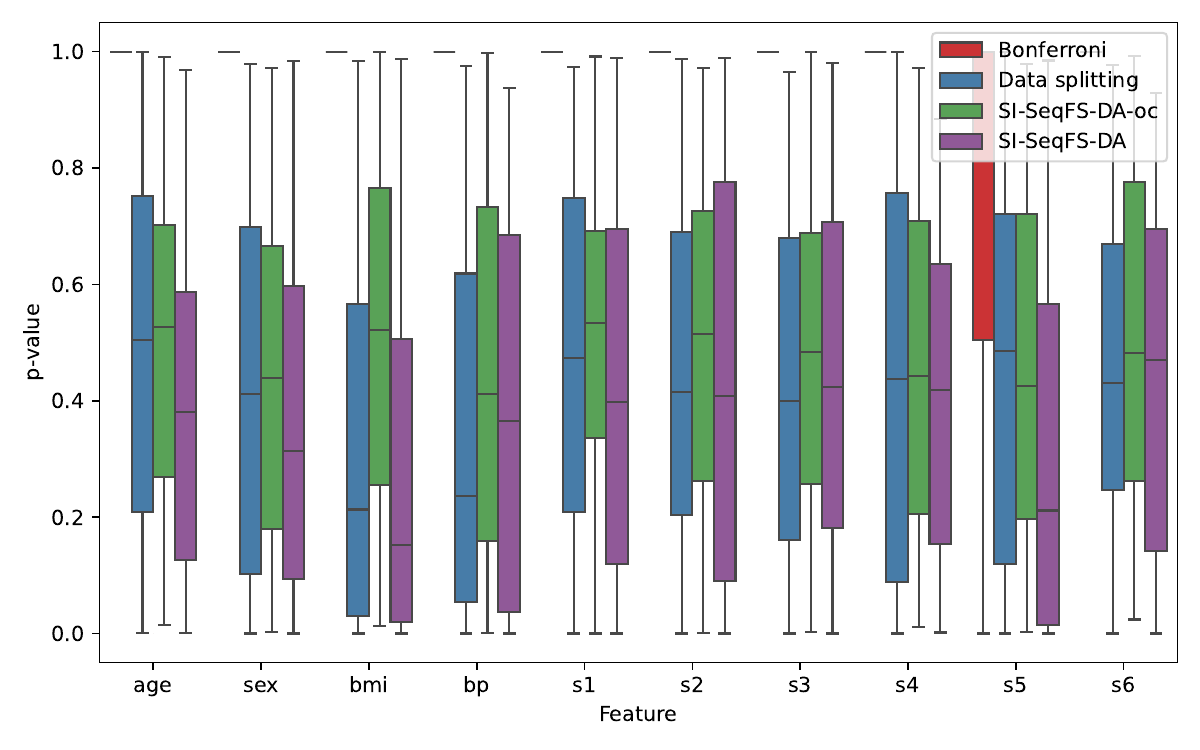}
        \caption{Diabetes.}
        % \label{fig:}
    \end{subfigure}
    \vspace{1pt}

    % Subfigure 2
    \begin{subfigure}{\linewidth}
        \centering
        \includegraphics[width=0.9\linewidth]{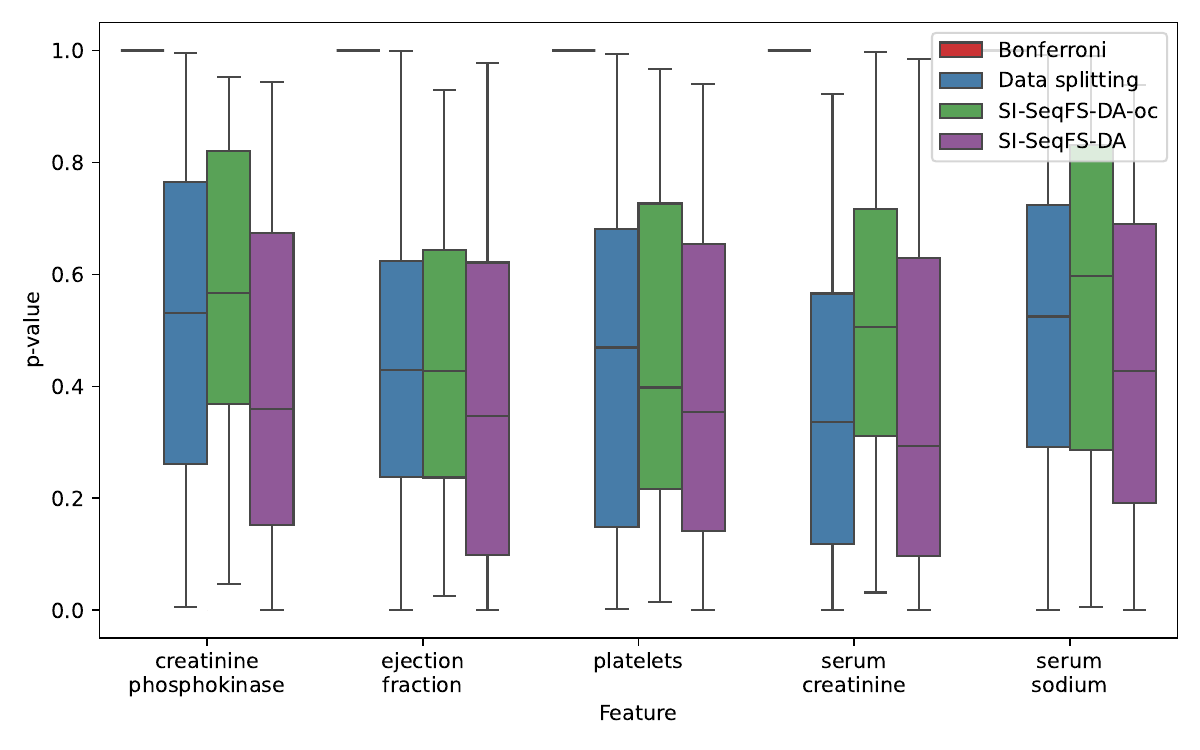}
        \caption{Heart Failure.}
        % \label{fig:}
    \end{subfigure}
    \vspace{1pt}

    % Subfigure 3
    \begin{subfigure}{\linewidth}
        \centering
        \includegraphics[width=0.9\linewidth]{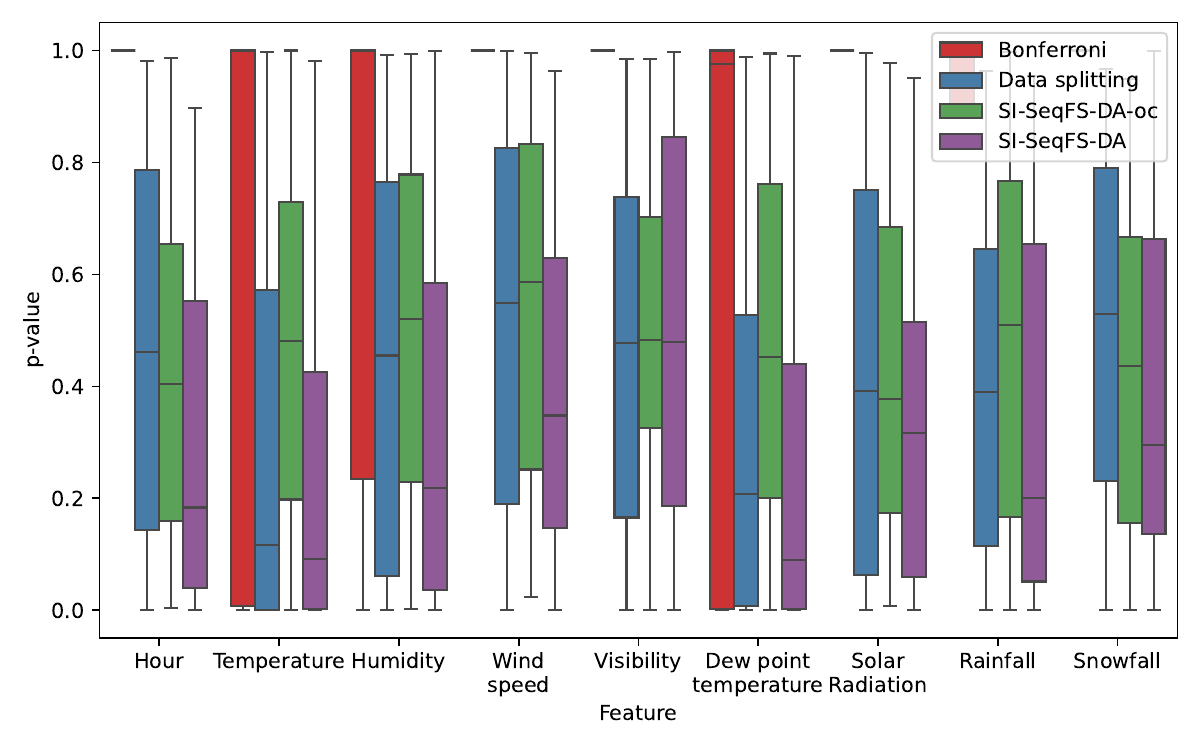}
        \caption{Seoul Bike.}
        % \label{fig:}
    \end{subfigure}
    \vspace{1pt}

    \caption{Backward SeqFS after DA with AIC on real-world datasets.}
    \label{fig:AIC_BS_realdata}
\end{figure}

%%%%%%%%%%%%%%%%%%%%%%%%3
\begin{figure}[!t] % Allow placement flexibility with [htbp]
    \centering

    % Subfigure 1
    \begin{subfigure}{\linewidth}
        \centering
        \includegraphics[width=0.9\linewidth]{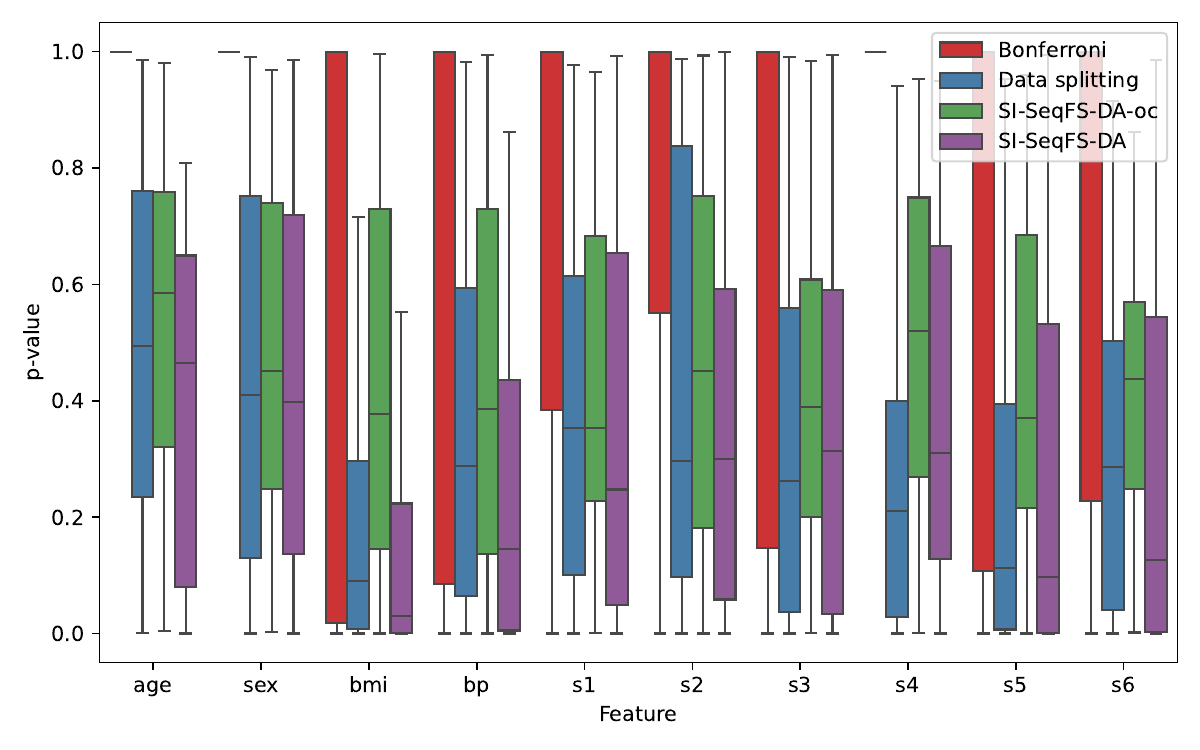}
        \caption{Diabetes.}
        % \label{fig:}
    \end{subfigure}
    \vspace{1pt}

    % Subfigure 2
    \begin{subfigure}{\linewidth}
        \centering
        \includegraphics[width=0.9\linewidth]{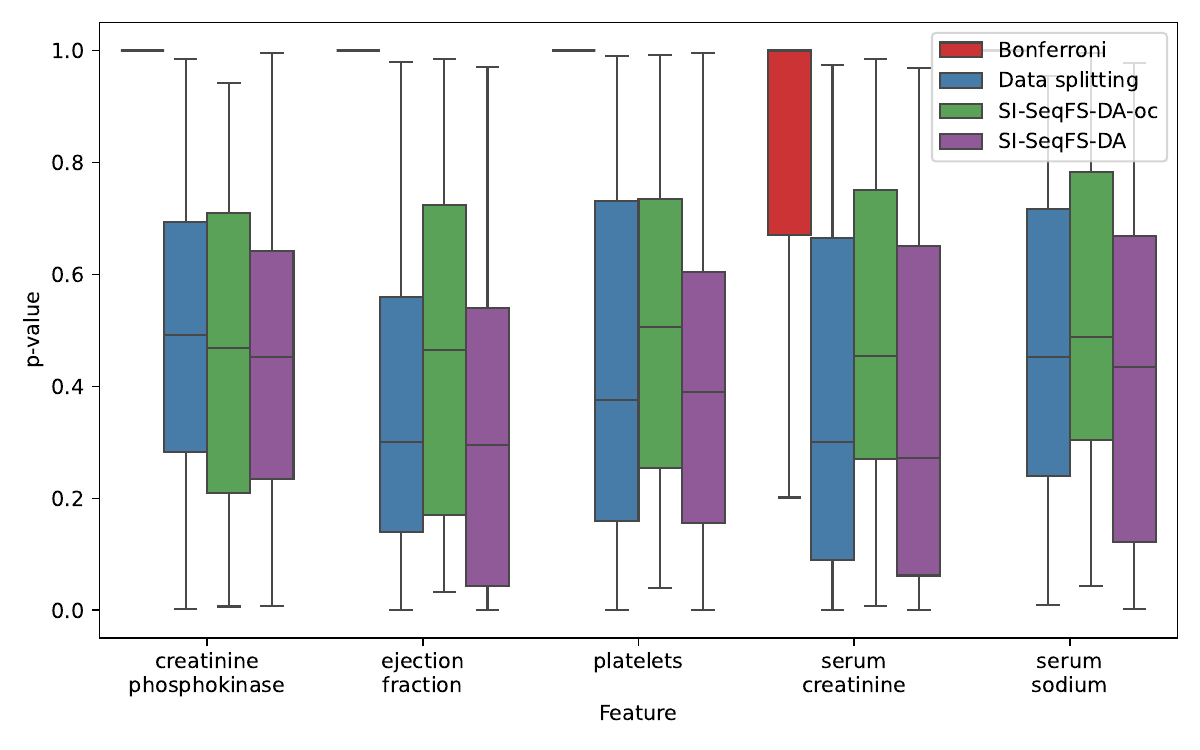}
        \caption{Heart Failure.}
        % \label{fig:}
    \end{subfigure}
    \vspace{1pt}

    % Subfigure 3
    \begin{subfigure}{\linewidth}
        \centering
        \includegraphics[width=0.9\linewidth]{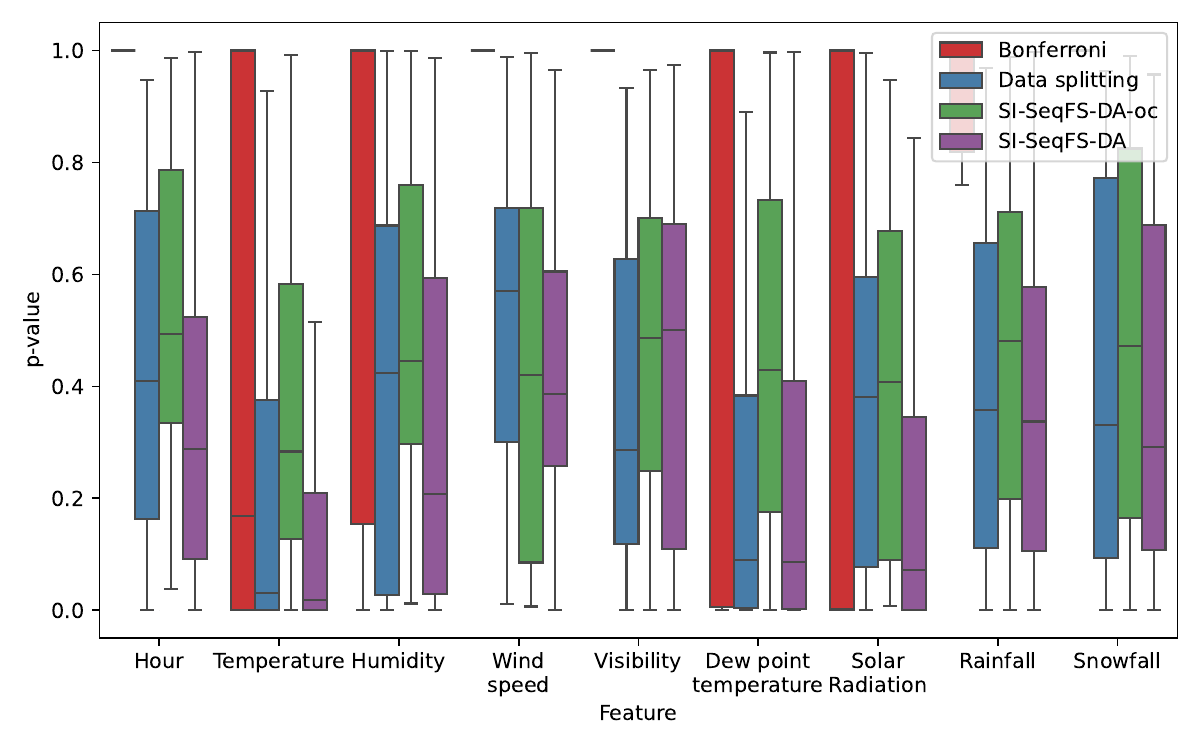}
        \caption{Seoul Bike.}
        % \label{fig:}
    \end{subfigure}
    \vspace{1pt}

    \caption{Backward SeqFS after DA with BIC on real-world datasets.}
    \label{fig:BIC_BS_realdata}
\end{figure}
%%%%%%%%%%%%%%%%%%%%%%%%4
\begin{figure}[htbp] % Allow placement flexibility with [htbp]
    \centering

    % Subfigure 1
    \begin{subfigure}{\linewidth}
        \centering
        \includegraphics[width=0.9\linewidth]{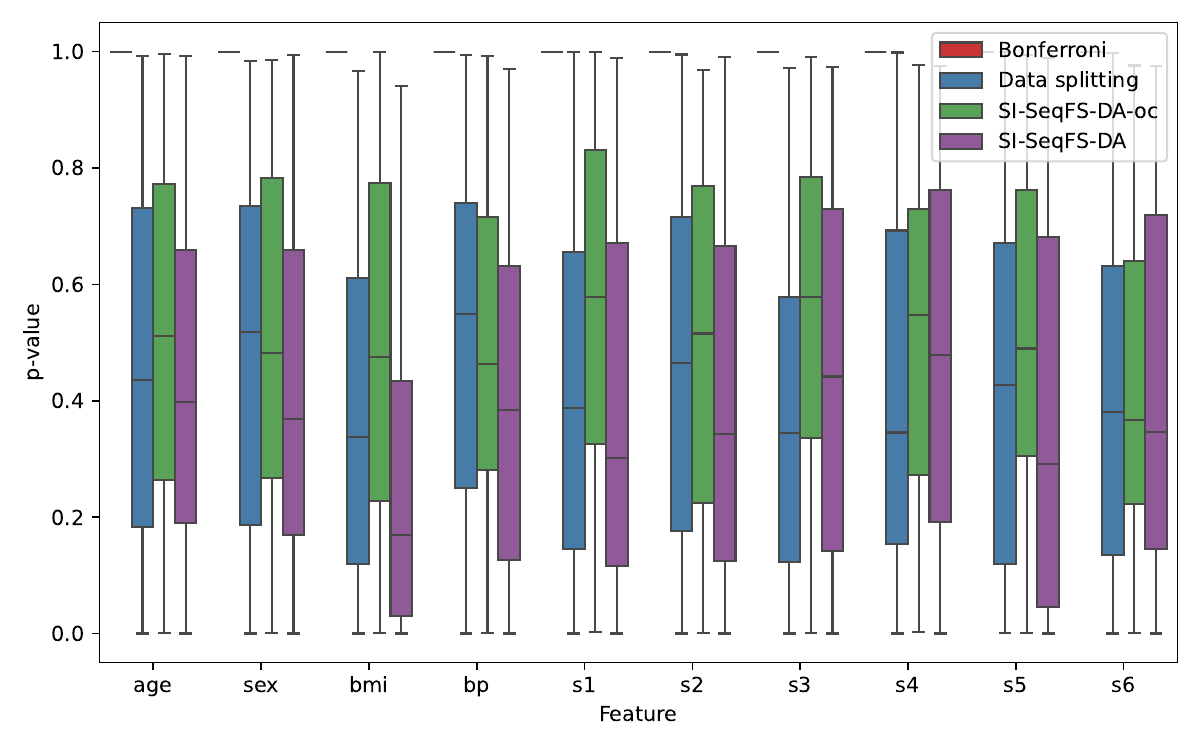}
        \caption{Diabetes.}
        % \label{fig:}
    \end{subfigure}
    \vspace{1pt}

    % Subfigure 2
    \begin{subfigure}{\linewidth}
        \centering
        \includegraphics[width=0.9\linewidth]{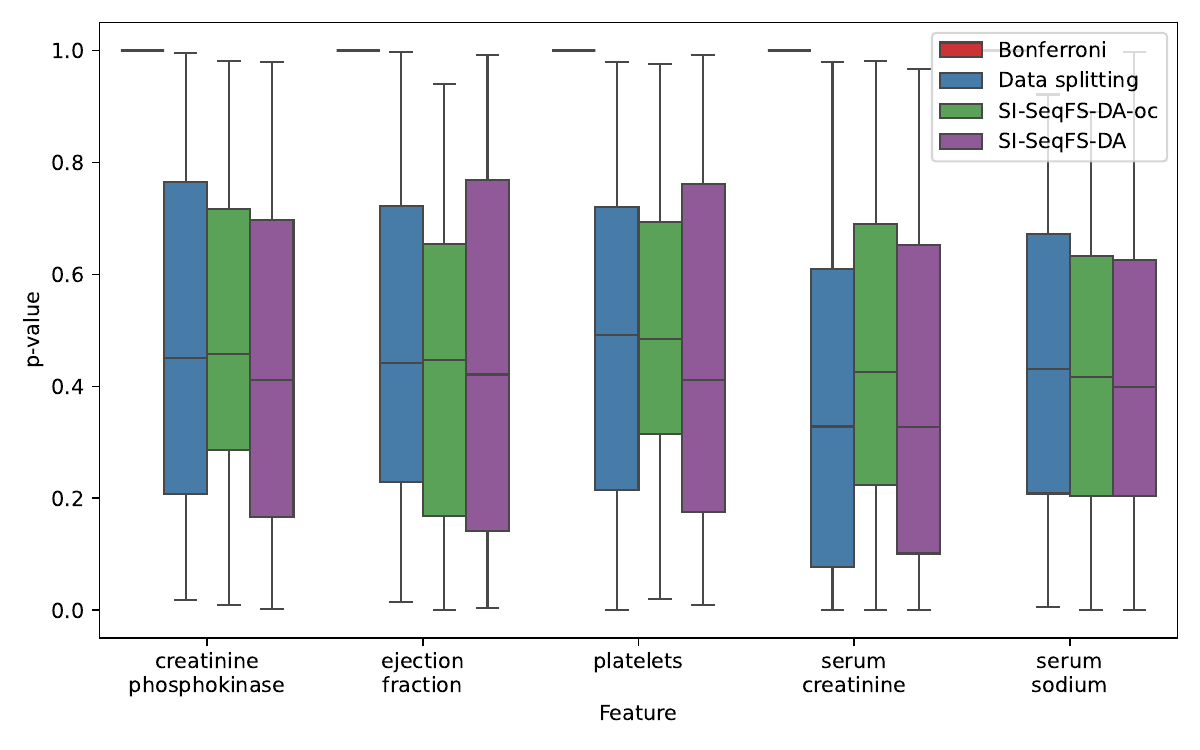}
        \caption{Heart Failure.}
        % \label{fig:}
    \end{subfigure}
    \vspace{1pt}

    % Subfigure 3
    \begin{subfigure}{\linewidth}
        \centering
        \includegraphics[width=0.9\linewidth]{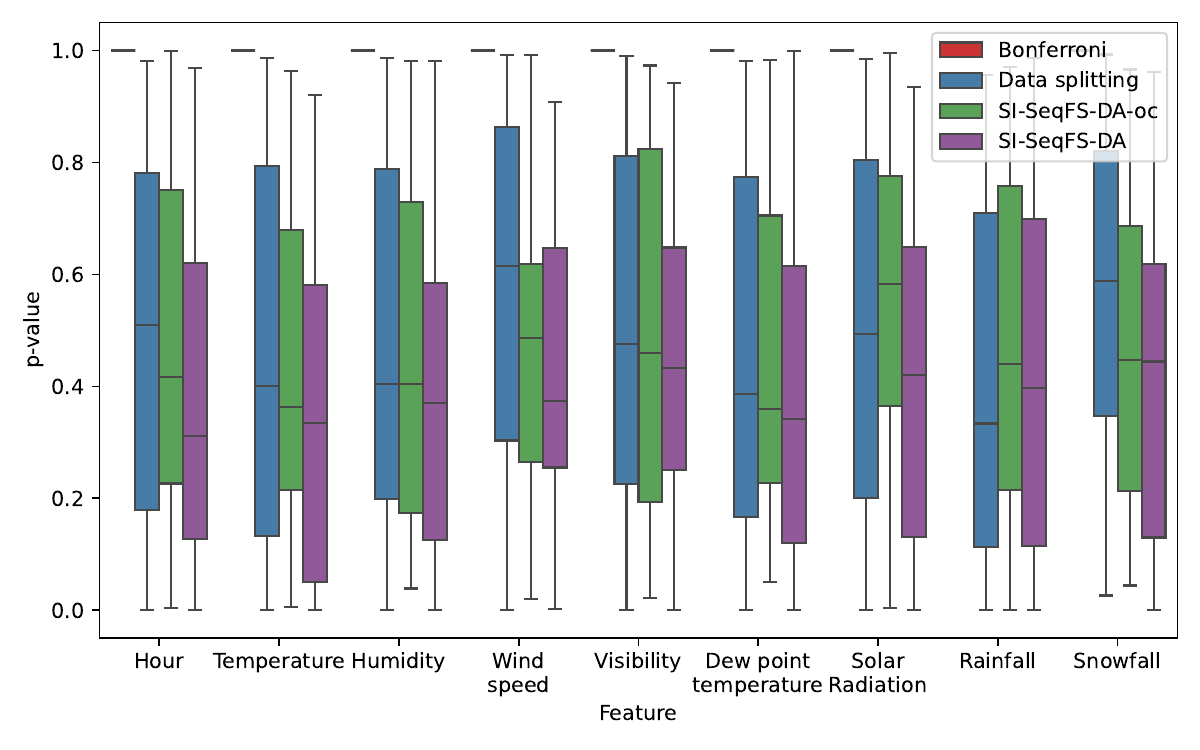}
        \caption{Seoul Bike.}
        % \label{fig:}
    \end{subfigure}
    \vspace{1pt}

    \caption{Backward SeqFS after DA with adjusted $R^2$ on real-world datasets.}
    \label{fig:AdjR2_BS_realdata}
\end{figure}